\documentclass[11pt]{article}

\usepackage{pratik} 

\newcommand{\jasa}{0}
\usepackage[left=1in,top=1in,right=1in,bottom=1in,head=0in]{geometry} 

\renewcommand{\bm}[1]{#1}

\renewcommand{\mathbf}[1]{#1}
\usepackage{pgfplots}
\pgfplotsset{compat=newest}
\usepackage{silence}
\if1\jasa
\usepackage[nodisplayskipstretch]{setspace}
\usepackage{titlesec}
\titlespacing{\section}{0pt}{\parskip}{0pt plus 2pt minus 2pt}
\titlespacing{\subsection}{0pt}{\parskip}{0pt plus 2pt minus 2pt}
\titlespacing{\subsubsection}{0pt}{\parskip}{-\parskip}
\fi

\def\spacingset#1{\renewcommand{\baselinestretch}
  {#1}\small\normalsize} 

\newcommand{\titletext}{Revisiting Optimism and Model Complexity in the Wake of
  Overparameterized Machine Learning}  

\newcommand{\authortext}{
  Pratik Patil\footremember{berkeleystats}{Department of Statistics,
    University of California, Berkeley.} \\ {\small 
    \texttt{pratikpatil@berkeley.edu}} \and
  Jin-Hong Du\footremember{cmustats}{Department of Statistics and Data Science,
    Carnegie Mellon University.}\footremember{cmumld}{Machine Learning
    Department, Carnegie Mellon University.} \\ {\small
    \texttt{jinhongd@andrew.cmu.edu}} \and 
  Ryan {J.} Tibshirani\footrecall{berkeleystats} \\ {\small
    \texttt{ryantibs@berkeley.edu}}
}

\newcommand{\abstracttext}{
Common practice in modern machine learning involves fitting a large number of 
parameters relative to the number of observations. These overparameterized
models can exhibit surprising generalization behavior, e.g., ``double descent''
in the prediction error curve when plotted against the raw number of model
parameters, or another simplistic notion of complexity. In this paper, we
revisit model complexity from first principles, by first reinterpreting and then
extending the classical statistical concept of (effective) \emph{degrees of
  freedom}. Whereas the classical definition is connected to fixed-X prediction
error (in which prediction error is defined by averaging over the same,
nonrandom covariate points as those used during training), our extension of
degrees of freedom is connected to random-X prediction error (in which
prediction error is averaged over a new, random sample from the covariate
distribution). The random-X setting more naturally embodies modern machine
learning problems, where highly complex models, even those complex enough to
interpolate the training data, can still lead to desirable generalization
performance under appropriate conditions. We demonstrate the utility of our
proposed complexity measures through a mix of conceptual arguments, theory,
and experiments, and illustrate how they can be used to interpret and compare 
arbitrary prediction models.
}

\begin{document}

\if1\jasa
\title{\titletext}
\author{}
\date{}
\maketitle
\vspace{-30pt}

\begin{abstract}
  \abstracttext
\end{abstract}

\clearpage
\spacingset{1.8} 
\fi

\if0\jasa
\title{\titletext}
\author{\authortext}
\date{\vspace{-10pt}}
\maketitle

\begin{abstract}
  \abstracttext
\end{abstract}
\fi

\section{Introduction}
\label{sec:introduction}

Model complexity is a key concept in statistics and machine learning, and is a
core consideration in prediction problems---a higher complexity allows for a
better fit to the training data, but may result in overfitting, whereas a lower
complexity may lack the ability to capture sufficiently rich behavior, and hence
lead to underfitting. There are numerous different ways to quantify the
complexity of a prediction model. One such way is called the (effective)
\emph{degrees of freedom} \citep{efron_1983, efron_1986, hastie_tibshirani_1987}
of a model, which is a classical concept in statistics, and will play a central
role in our paper. This is often interpreted as the number of ``free
parameters'' in the fitted model. 

Meanwhile, driven by the enormous practical successes of neural networks and 
deep learning, there has recently been great interest in the community in
studying \emph{overparameterized models}, where the number of parameters is
large relative to the number of observations. Overparameterized models can
exhibit surprising generalization behavior, in that they can generalize well
even if they perfectly (or nearly) interpolate noisy training data
\citep{zhang_bengio_hardt_recht_vinyals_2016, belkin_hsu_ma_mandal_2019}. 
As we will explain later (\cref{subsec:df-limitations}), classical degrees of
freedom fails to adequately explain this phenomenon. For example, it is not
able to distinguish between interpolating models: the degrees of freedom of any 
interpolator is exactly $n$, the number of training observations. 

The underlying limitation of degrees of freedom, as classically defined, is that
it is tied to a measure of prediction error which we refer to (following
\citealt{rosset_tibshirani_2020}) as \emph{fixed-X} prediction error. In this
measure, prediction error is defined by averaging over the same fixed set of
covariate points as those used during training. In certain problem
settings---that is, low-dimensional, smooth prediction problems---this measure
is a good proxy for \emph{random-X} prediction error, which is given by
averaging over a new random sample from the covariate distribution. Yet, in
high-dimensional and/or nonsmooth prediction problems, fixed-X and random-X
errors can behave quite differently. A generalizing interpolator epitomizes this
difference (\cref{subsec:fixed-random-x}): as $n \to \infty$, it has fixed-X
excess error converging to the noise level but random-X excess error converging
to zero. 

In nearly all modern machine learning prediction problems, random-X error is the
perspective of interest. Given its connection to fixed-X error, it should not be
surprising that classical degrees of freedom can break down for prediction
models such as interpolators, where random-X and fixed-X errors diverge. In this
paper, we propose a new measure of degrees of freedom that connects directly to
random-X prediction error, and allows us to reason about complexity in a
nontrivial way for \emph{any} predictive model, including interpolators. We
provide a simple illustration in \cref{fig:ridgeless-intro}.

\begin{figure}[t]
\includegraphics[width=\textwidth]{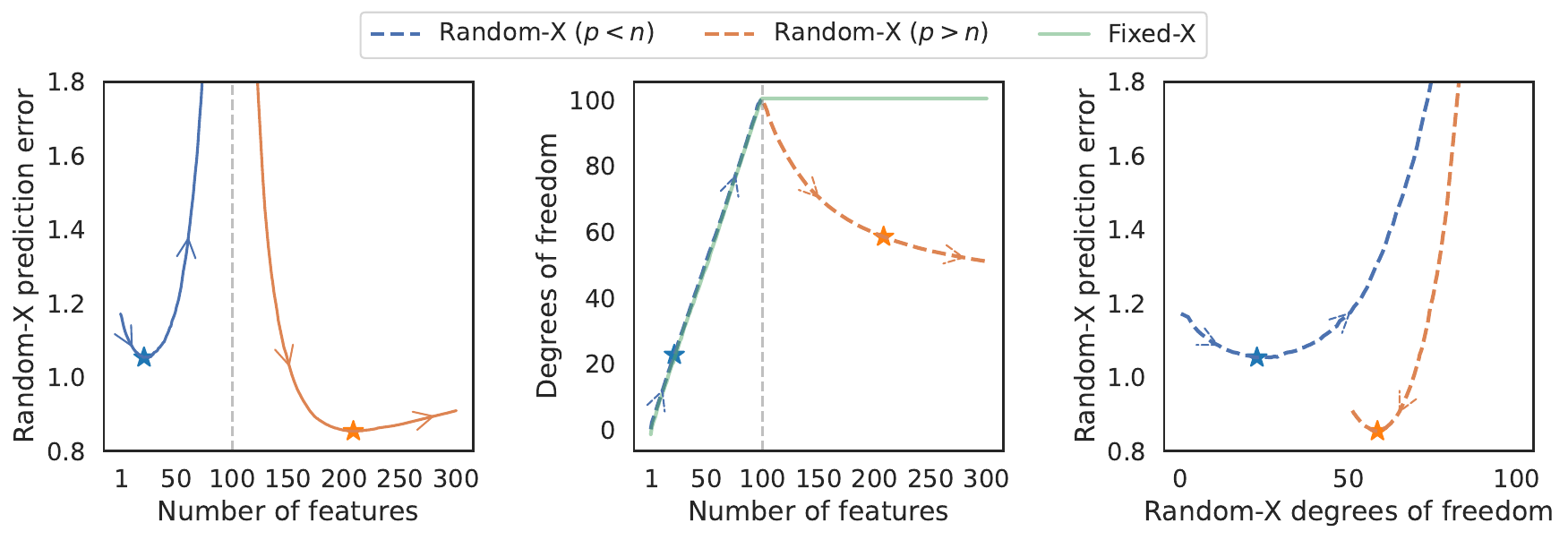}
\caption{An illustration using ridgeless least squares regression as the
  prediction model, trained on $n=100$ samples and $p$ features, where $p$
  ranges from 1 to 300. The true conditional mean is a nonlinear function in the 
  features, and hence adding more features to the working linear model helps its 
  approximation capacity (The precise details are given in
  \cref{app:data-models}). In the left panel, we can see that the random-X
  prediction error curve exhibits ``double descent'' in $p$. In the middle
  panel, the classical (fixed-X) definition of degrees of freedom increases
  linearly for $p \leq n$, but then it flattens out at the trivial answer of $n$
  degrees of freedom for all $p > n$. The ``intrinsic'' random-X degrees of
  freedom, one of two basic versions of random-X degrees of freedom to be
  defined later in \cref{sec:proposal}, is \emph{decreasing} when $p > n$,
  indicating that the ridgeless interpolator is becoming \emph{less} complex as
  the dimensionality grows. In the right panel, we plot the random-X prediction
  error as a function of random-X degrees of freedom. The interpretation: our
  proposed complexity measure maps every overparameterized model onto an
  equivalent underparameterized model, and the best-predicting model (which lies 
  in the overparameterized regime) actually has relatively low complexity.}       
\label{fig:ridgeless-intro}
\end{figure}

\subsection{Summary and outline}

We provide a summary of our contributions and outline the structure of the paper
below. 
    
\paragraph{New random-X measures of degrees of freedom.}

After we review preliminary materials in \cref{sec:preliminaries}, we present
new measures of model complexity in \cref{sec:proposal}. In particular, we
extend the classical notion of degrees of freedom to the setting of random-X.
We do so by first reinterpreting the classical construction of degrees of
freedom in a new light, then translating this to random-X prediction error. We
propose two basic versions of random-X degrees of freedom: one to capture both
bias and variance components of the error, and another based on variance alone. 

\paragraph{Basic properties and theory for random-X degrees of freedom.} 

In \cref{sec:properties}, we describe basic properties of the proposed random-X
degrees of freedom measures, and draw connections to related ideas in the
literature. \cref{sec:theory} derives theory for a few standard prediction
models, such as ridge regression and the lasso, and demonstrates that degrees of
freedom typically decreases as the regularization strength increases, and
typically increases as the number of features increases.     

\paragraph{Numerical experiments for a diverse set of prediction models.}  

In \cref{sec:experiments}, we illustrate the versatility of our complexity
measures by presenting results from numerical experiments using the lasso,
$k$-nearest neighbors regression, and random forests.

\paragraph{Decomposing degrees of freedom under distribution shift.}

In \cref{sec:decomposition}, we discuss how to decompose the random-X degrees of
freedom of a prediction model into constituent parts, so as to quantify the
contribution of various components---such as bias, variance, and covariate
shift---to the final measure of model complexity. This is based on borrowing
ideas from Shapley values.

\subsection{Related work}
\label{sec:related-work}

There is a lot of literature related to the topic of our paper, which we discuss
in two groups.

\paragraph{Model optimism and degrees of freedom.}

Optimism and (effective) degrees of freedom are classical concepts and
well-studied in statistics, with important references being
\citet{efron_1983,efron_1986, efron_2004}. Degrees of freedom for linear 
regression and linear smoothers have a particular simple form, as the trace of
the smoother matrix, and have a long history of study, for example,
\citet{mallows1973some,craven_wahba_1978, hastie_tibshirani_1987,
hastie_tibshirani_1990}. Broadly related to this is the topic of estimating
risk for model selection, which is widely studied and itself carries quite a
rich literature, for example, \citet{sclove_1969,
hocking_1976,akaike_1973,schwarz_1978,thompson_1978_1,thompson_1978_2,
golub_heath_wahba_1979, breiman_freedman_1983,breiman_spector_1992}, and many
others.

A landmark contribution in the study of degrees of freedom and unbiased risk
estimation is known as \emph{Stein's unbiased risk estimator} (SURE), due to
\citet{stein_1981}. This has enabled the development of numerous closed-form
unbiased estimators of degrees of freedom (and fixed-X prediction error) for
methods such as wavelet denoising, shape-constrained regression, quantile
regression, lasso and various generalizations, and low-rank matrix
factorization; see, for example, \citet{donoho1995adapting,
cai1999adaptive,meyer_2000, zou_hastie_tibshirani_2007,
zou2008regularized,tibshirani_taylor_2012,
candes2013unbiased,tibshirani2015degrees,mikkelsen2018degrees, chen2020degrees},
among others. For an alternative perspective based on auxiliary randomization
(which reduces to SURE in a limiting case), see \citet{oliveira2021unbiased,
oliveira2022unbiased}.

The above literature is all rooted in the fixed-X setting, which (as we will
explain precisely in the next section) measures prediction error at the same
fixed covariate points as those used in training. \citet{rosset_tibshirani_2020}
compare and contrast the bias-variance tradeoff, prediction error, and other
core concepts in statistical decision theory in the fixed-X and random-X
settings. Our work builds on theirs and introduces a notion of random-X degrees
of freedom. Though we believe that this should be of general interest, it is of
particular interest for interpolators.  

Closely related to our proposed complexity measure is the recent work of
\citet{luan2021predictive, luan2022measuring, curth2023u}. They propose a
measure of random-X degrees of freedom that is suitable for linear smoothers.
It is related to our approach in this special case, and \cref{subsec:luan}
provides details. Broadly speaking, our approach is more general (accommodates
arbitrary prediction models), and also, allows for both bias and variance
components of the random-X optimism to enter into the complexity measure,
whereas the previous proposals focus on variance alone.

\paragraph{Other complexity measures.}

There are many other criteria for measuring the complexity of a model or an
object. Broadly, this includes ideas from information theory and theoretical
computer science, such as Kolmogorov complexity \citep{kolmogorov1963tables},
minimum message length \citep{wallace1968information}, and minimum description
length \citep{rissanen1978modeling}. Closer to our study, coming from machine
learning theory, are Vapnik-Chervonenkis (VC) dimension
\citep{vapnik1971uniform} and Rademacher complexity
\citep{bartlett2002rademacher}. For a discussion of these concepts and their
role in generalization theory, see, for example,
\citet{shwartz2014understanding} or \citet{mohri2018foundations}. An important
point to clarify is that VC dimension and Rademacher complexity differ from
degrees of freedom in the following sense: the former measures apply to a
\emph{class} of prediction models, whereas the latter applies to a particular
\emph{fitted} prediction model. In other words, degrees of freedom as
complexity measure is more finely-tuned to the \emph{way} in which a given model
is trained, incorporating the action of the fitting algorithm, and the
distribution of the underlying data. As an example, a linear model trained via
least squares and ridge regression (using strong regularization) will have the
same Rademacher complexity, but different degrees of freedom.

\section{Preliminaries}
\label{sec:preliminaries}

We start with a review of fixed-X and random-X prediction error, and classical
(fixed-X) optimism and degrees of freedom. Then we discuss the limitations of 
classical degrees of freedom with respect to understanding overparameterized  
models.  

\subsection{Fixed-X and random-X prediction error}
\label{subsec:fixed-random-x}

Consider a standard regression setup, with independent and identically
distributed (i.i.d.) training samples $(x_i, y_i) \in \RR^p \times \RR$, which
follow the relationship
\begin{equation}
\label{eq:data_model}
y_i = f(x_i) + \eps_i, \quad i \in [n],
\end{equation}
for $f(x) = \EE[y_i | x_i = x]$, and i.i.d.\ mean zero stochastic errors
$\eps_i$, $i \in [n]$. We assume that each $\eps_i$ is independent of
$x_i$. Here and throughout, we abbreviate $[n] = \{1,\dots,n\}$. Also, let
$\sigma^2 = \Var[\eps_i] > 0$ denote the error variance, let $X \in \RR^{n
  \times p}$ denote the feature matrix (with $i\th$ row $x_i$), and let $y \in 
\RR^n$ denote the response vector (with $i\th$ entry $y_i$). 

Suppose that we have a model fitting procedure \smash{$\hf$} which produces the 
predictor \smash{$\hf(\cdot; X,y) : \RR^p \to \RR$} when trained on the data
$(X,y)$. Thus, \smash{$\hf(x; X,y)$} is an estimate of $f(x)$. When the training
data is clear from the context, we will simply write this as \smash{$\hf(x)$}.  

In \emph{fixed-X} prediction error, we measure the error of \smash{$\hf$} at a
set of new response values $y_i^*$, $i \in [n]$, where each $y_i^*$ and $y_i$
are i.i.d.\ conditional on $x_i$. Formally, this is  
\begin{equation}
\label{eq:err-F}
\err\f(\hf) = \EE \bigg[ \frac{1}{n} \sum_{i=1}^n \big( y^*_i - \hf(x_i) \big)^2
\, \Big| \, X \bigg],
\end{equation}
In \emph{random-X} prediction error, we measure the error of \smash{$\hf$} at a
new sample $(x_0,y_0) \in \RR^p \times \RR$, which is i.i.d.\ to the training
samples $(x_i,y_i)$, $i \in [n]$. Formally, this is  
\begin{equation}
\label{eq:err-R}
\err\r(\hf) = \EE \big[\big( y_0 - \hf(x_0) \big)^2\big].
\end{equation}
To be clear, the expectation in \eqref{eq:err-F} is taken with respect to
$y,y^*$, and is conditional on $X$, whereas that in \eqref{eq:err-R} is taken
with respect to $X, y, x_0, y_0$. 

While random-X prediction error is the central object of interest in machine
learning theory and in many modern statistics problems, fixed-X prediction error
has a long history of study in statistics; we refer to
\citet{rosset_tibshirani_2020} (and references therein) for an in-depth
discussion. For our purposes, to motivate our study, it suffices to make only
high-level comments to compare them. For smooth functions \smash{$f,\hf$} in
low dimensions (i.e., $n$ large compared to $p$), one can generally expect
\smash{$\err\f(\hf)$} and \smash{$\err\r(\hf)$} to behave similarly. For
example, empirical process theory offers uniform control on the deviation
between the $L^2$ norms based on taking a sample average over i.i.d.\ draws
$x_i$, $i \in [n]$, and taking an expectation with respect to $x_0 \sim P_x$.
Such results can be used to derive an asymptotic equivalence (and nonasymptotic
bounds) between \smash{$\err\f(\hf)$} and \smash{$\err\r(\hf)$} in certain
settings.

However, for nonsmooth functions and/or high-dimensional problem settings, the
two metrics can behave quite differently. Consider, as an example, a
generalizing interpolator: here, we would have random-X excess error
\smash{$\err\r(\hf) - \err\r(f) \to 0$} as $n \to \infty$, but fixed-X excess
error
\[
\err\f(\hf) - \err\f(f) = \EE \bigg[ \frac{1}{n} \sum_{i=1}^n (y^*_i -y_i)^2 \,
\Big| \, X \bigg] - \sigma^2 = \sigma^2,
\]
where recall $\sigma^2 = \Var[\eps_i]$ in the data model \eqref{eq:data_model}.
This represents a huge difference between the two metrics: one vanishing, and
the other pinned at the noise level.

\subsection{Fixed-X optimism and degrees of freedom} 
\label{subsec:fixed-x-df}

The (effective) \emph{degrees of freedom} of \smash{$\hf$} is defined as  
\begin{equation}
\label{eq:df-F}
\df\f(\hf) = \frac{1}{\sigma^2} \sum_{i=1}^n \Cov[y_i, \hf(x_i) \,|\, X]. 
\end{equation}
This is often motivated intuitively as follows: the more complex the fitting
procedure \smash{$\hf$}, the more ``self-influence'' each response $y_i$ will
have on the corresponding fitted value \smash{$\hf(x_i)$} (and hence the higher
the degrees of freedom in total). An important property of degrees of freedom
is its intimate connection to \emph{fixed-X} optimism, which is defined as
\begin{equation}
\label{eq:opt-F}
\opt\f(\hf) = \err\f(\hf) - \EE \bigg[ \frac{1}{n} \sum_{i=1}^n \big( y_i -
  \hf(x_i) \big)^2 \, \Big| \, X \bigg].
\end{equation}
The second quantity on the right-hand side above is simply the training error
(conditional on $X$). The precise connection between \eqref{eq:df-F} and
\eqref{eq:opt-F} is given by what is sometimes called \emph{Efron's optimism
theorem}, attributed to \citet{efron_1986, efron_2004}:
\begin{equation}
\label{eq:opt-df-F}
\opt\f(\hf) = \frac{2 \sigma^2}{n} \df\f(\hf). 
\end{equation}
This holds without any assumptions on \smash{$\hf$}, and can be checked via 
simple algebra (add and subtract $y^*_i$ within the square in each summand in
\smash{$\err\f(\hf)$} in \eqref{eq:err-F}, then expand and simplify).

The rest of this subsection can be skipped without interrupting the flow of main
ideas. We use it as an opportunity to provide general context about classical
interest in degrees of freedom, as alluded to in the related work subsection.
\emph{Stein's lemma} \citep{stein_1981} says if \smash{$\hf$} is weakly
differentiable as a function of $y$, and we assume Gaussian errors $\eps_i$, $i
\in [n]$ in \eqref{eq:data_model}, then
\begin{equation}
\label{eq:df-F-stein}
\df\f(\hf) = \EE \bigg[ \sum_{i=1}^n \frac{\partial \hf(x_i)}{\partial y_i} \,
\Big| \, X \bigg].     
\end{equation}
Based on \eqref{eq:df-F-stein}, we are able to form an unbiased estimate of
\smash{$\df\f(\hf$)}, namely, \smash{$\hdf\f = \sum_{i=1}^n \partial \hf(x_i)
  /\partial y_i$} (if we are able to compute it). From \eqref{eq:opt-F} and
\eqref{eq:opt-df-F}, we see that this in turn provides an unbiased estimate of
fixed-X prediction error, namely, \smash{$\frac{1}{n} \sum_{i=1}^n (y_i -
  \hf(x_i))^2 + 2 \sigma^2 \hdf\f$}.

Thus we can see that there is a clear interest in estimating degrees of freedom,
and utilizing Stein's formula, in order to estimate fixed-X prediction error.
However, this is not really aligned with the general focus of our paper
henceforth, and our paper actually proceeds in the opposite direction: we will
presume an estimate of prediction error in order to estimate degrees of freedom.
As we will see in \cref{sec:proposal}, this is a fruitful way to extend degrees
of freedom past the fixed-X setting.

\subsection{Limitations of classical degrees of freedom}
\label{subsec:df-limitations}

A critical limitation of classical (fixed-X) degrees of freedom, as defined in
\eqref{eq:df-F}, is straightforward to state. For any interpolator, satisfying
\smash{$\hf(x_i) = y_i$}, $i \in [n]$, we have the trivial answer: 
\begin{equation}
\label{eq:df-F-interpolator}
\df\f(\hf) = \frac{1}{\sigma^2} \sum_{i=1}^n \Cov[y_i, y_i \,|\, X] = n. 
\end{equation}
If characterizing fixed-X optimism is truly the end goal of degrees of freedom,
then we should not be bothered by this (seemingly) obvious fact since any
interpolator has zero training error and the same fixed-X prediction error.
Yet, if we are to think of degrees of freedom as a general measure of model
complexity, then \eqref{eq:df-F-interpolator} leaves a lot to be desired. As we
know from the recent wave of work in machine learning and statistics (for
example, see the review articles \citet{belkin2021fit,bartlett2021deep} and
references therein), some interpolators---in particular, implicitly regularized
ones---are actually quite well-behaved and can generalize well to unseen data.
In classical degrees of freedom, thus, we are lacking a complexity measure that
can distinguish between well-behaved interpolators, which are smooth in between
the covariate points, and wild ones, which are arbitrarily nonsmooth.

The next section develops an extension of the classical notion of degrees of
freedom which connects to random-X (rather than fixed-X) prediction error. As
we will see, the extension will overcome the limitation just described---the new
notion will assign a meaningful complexity measure to every prediction model,
including interpolators.

\section{Random-X degrees of freedom}
\label{sec:proposal}

In this section, we first present a fresh reinterpretation of fixed-X degrees of
freedom. Then we show how this leads to a generalization of degrees of freedom
in the random-X setting.

\subsection{Reinterpreting fixed-X degrees of freedom}
\label{subsec:df-reinterpretation}

We first recall a standard fact about fixed-X degrees of freedom: if the feature
matrix $X \in \RR^{n \times p}$ has linearly independent columns, then least
squares regression of $y$ on $X$, given by \smash{$\hf\ls(x) = x^\top\hbeta\ls$}
where \smash{$\hbeta\ls = (X^\top X)^{-1} X^\top y$}, has degrees of freedom
exactly $p$. This is simply the number of parameters in \smash{$\hbeta\ls$}.
This fact is easily verified from \eqref{eq:df-F}, abbreviating $P_X = X (X^\top
X)^{-1} X^\top$:
\begin{align}
\nonumber
\df\f(\hf\ls) 
&= \frac{1}{\sigma^2} \tr( \Cov[X \hbeta\ls, y \,|\, X] ) \\   
\nonumber
&= \frac{1}{\sigma^2} \tr( \Cov[P_X \hspace{1pt} y, y \,|\, X] ) \\ 
\label{eq:df-F-ls}
&= \tr(P_X) \\
&= p,
\end{align}
where we used $\Cov[P_X \hspace{1pt} y, y \,| X] = P_X \Cov[y | X] =
\sigma^2 P_X$ in the second-to-last line, and we used the cyclic property 
$\tr(P_X) = \tr(X^\top X (X^\top X)^{-1}) = p$ in the last line. 

Now we show that the fact about least squares in \eqref{eq:df-F-ls}, which is
well-known in the literature, can be used to reinterpret fixed-X degrees of
freedom in a new light. Recalling Efron's optimism formula \eqref{eq:opt-df-F},
the least squares regression predictor \smash{$\hf\ls$} has fixed-X optimism
\[
\opt\f(\hf\ls) = \frac{2 \sigma^2}{n} p.
\]
Given an arbitrary predictor \smash{$\hf$}, we know that it still satisfies
(copying \eqref{eq:opt-df-F} here for convenience)
\[
\opt\f(\hf) = \frac{2 \sigma^2}{n} \df\f(\hf).
\]
Comparing the last two displays, we see that we may hence interpret the degrees
of freedom of \smash{$\hf$} as the value of $d \in [0, \infty]$ for which least
squares predictor on $d$ linearly independent features has the same fixed-X
optimism as \smash{$\opt\f(\hf)$}. This is simply a reformulation of the
original definition \eqref{eq:df-F}, and the next proposition records this idea
precisely.

\begin{proposition}
\label{prop:df-reinterpretation}
For each fixed $d \leq n$, let \smash{$\tX_d \in \RR^{n \times d}$} be an
arbitrary feature matrix having linearly independent columns, and consider
\smash{$\hf\ls(\cdot; \tX_d, y)$}, the predictor from least squares regression 
of $y$ on \smash{$\tX_d$}, which we call our ``reference'' model, and abbreviate
as \smash{$\hf\rf_d$}.
This satisfies       
\begin{equation}
\label{eq:opt-F-ls}
\opt\f(\hf\rf_d) = \frac{2 \sigma^2}{n} d, \quad d = 1,\dots,n.
\end{equation}
Let us extend these reference values so that we may write for all nonnegative $d$,  
\begin{equation}
\label{eq:opt-F-ls-extended}
\bopt\f(\hf\rf_d) = \frac{2 \sigma^2}{n} d, \quad d \in [0, \infty].
\end{equation}
Given an arbitrary predictor \smash{$\hf = \hf(\cdot; X, y)$}, define $d$ to be
the unique nonnegative number for which 
\begin{equation}
\label{eq:opt-F-match}
\opt\f(\hf) = \bopt\f(\hf\rf_d).
\end{equation}
Then \smash{$\df\f(\hf) = d$}.
\end{proposition}

\begin{proof}
The proof is immediate. The left-hand side in \eqref{eq:opt-F-match} equals
\smash{$(2 \sigma^2 / n) \df\f(\hf)$} and the right-hand side equals \smash{$(2 
  \sigma^2 / n) d$}. Cancelling the common factor of $2 \sigma^2 / n$ gives the 
result.
\end{proof}

Next we show how to lift this idea to the random-X setting.

\subsection{Defining random-X degrees of freedom}

The idea behind \cref{prop:df-reinterpretation} is both fairly natural and
fairly general. To cast the core idea at a high level, in order to define the
complexity of a given prediction model \smash{$\hf$}, we require two things:   

\begin{enumerate}[label=\roman*.]
\item a \emph{metric} $\met$, which we assume (without loss of generality) is
  negatively-oriented: the lower the value of \smash{$\met(\hf)$}, the less
  complex we deem \smash{$\hf$}; 

\item a \emph{reference class} \smash{$\{ \hf\rf_d : d \in D \}$}, which is a
  class of models indexed by a number of parameters $d$, assumed to be
  ``canonical'' in some sense to the prediction task at hand.  
\end{enumerate}

We then assign to \smash{$\hf$} a complexity of $d$ where $d$ is smallest value
in $D$ for which \smash{$\met(\hf) \leq \met(\hf\rf_d)$}. In other words, it is
defined to be the number of parameters in the smallest reference model whose
metric value is at least that of \smash{$\hf$}.

Fixed-X degrees of freedom is a special case of this general recipe, in which
the metric is implicitly taken to be fixed-X optimism---but suitably extended so
that this metric ranges over the full set of nonnegative reals, and we can
always achieve equality: \smash{$\met(\hf) = \met(\hf\rf_d)$} for some $d \geq
0$. The reference class is taken to be least squares regression on an arbitrary 
full rank feature matrix.

Towards a random-X extension, a natural inclination would be to maintain least
squares regression as the reference class, and simply replace fixed-X optimism
\eqref{eq:opt-F} with random-X optimism, defined as
\begin{equation}
\label{eq:opt-R}
\opt\r(\hf) = \err\r(\hf) - \EE \bigg[ \frac{1}{n} \sum_{i=1}^n \big( y_i -
  \hf(x_i) \big)^2 \bigg].
\end{equation}
This is now the random-X prediction error (rather than the fixed-X error) minus
the training error. Before we pursue a random-X extension, it is important to
note that the classical definition, which uses least squares and fixed-X
optimism in the equivalent characterization given in
\cref{prop:df-reinterpretation}, is special for two reasons. The metric
assigned to the reference model here, i.e., the fixed-X optimism
\eqref{eq:opt-F-ls} of least squares, depends neither on $X$ nor on the law of
$y | X$, beyond assuming isotropic errors (as we have done throughout, i.e.,
$\Cov[y | X] = \sigma^2 I$, with $I$ being the $n \times n$ identity matrix).

In comparison, the random-X optimism \eqref{eq:opt-R} of least squares
regression of $y$ on $X$ depends on both the distribution of $X$ and of $y |
X$. This means that we will have to be more precise in defining the distribution
of the data on which we measure the random-X optimism of least squares, so that 
this quantity becomes well-defined. The next definition provides details.

\begin{definition}
\label{def:df-R}
Assume that $n \geq 2$. For each fixed $d \leq n-1$, let \smash{$\tX_d \in
  \RR^{n \times d}$} have i.i.d.\ rows from $\cN(0,\Sigma)$, with $\Sigma \in
\RR^{d \times d}$ an arbitrary deterministic positive definite covariance
matrix. Let     
\begin{equation}
\label{eq:df-R-response}
\ty | \tX_d \sim \cN(\tX_d \hspace{1pt} \beta, \sigma^2 I),  
\end{equation}
with $\beta \in \RR^d$ an arbitrary deterministic coefficient vector. Consider
\smash{$\hf\ls(\cdot; \tX_d, \ty)$}, the predictor from least squares regression
of \smash{$\ty$} on \smash{$\tX_d$}, as our reference model, which we abbreviate
as \smash{$\hf\rf_d$}. We have 
\begin{equation}
\label{eq:opt-R-ls}
\opt\r(\hf\rf_d) = \sigma^2 \bigg( \frac{d}{n} + \frac{d}{n-d-1} \bigg), \quad d =
1,\dots,n-1. 
\end{equation}
Let us extend these reference values so that we may write 
\begin{equation}
\label{eq:opt-R-ls-extended}
\bopt\r(\hf\rf_d) = \sigma^2 \bigg( \frac{d}{n} + \frac{d}{n-d-1} \bigg), \quad 
d \in [0,n-1]. 
\end{equation}
Then, given an arbitrary predictor \smash{$\hf = \hf(\cdot; X, y)$}, we define
\smash{$\df\r(\hf) = d$} as the unique $d \in [0,n-1]$ for which    
\begin{equation}
\label{eq:opt-R-match}
\opt\r(\hf) = \bopt\r(\hf\rf_d).
\end{equation}
\end{definition}

The result in \eqref{eq:opt-R-ls} is driven by the random-X prediction error of
least squares regression for jointly Gaussian data, which is well-known, and can
be found in, e.g., \citet{stein1960multiple, tukey1967discussion,hocking_1976,
thompson_1978_1, thompson_1978_2,dicker_2013,rosset_tibshirani_2020}, among
others. We give a derivation in \cref{app:opt-R-ls} for completeness.

Several remarks are in order, to discuss random-X degrees of freedom as defined
in \cref{def:df-R} and compare it to the classical notion of fixed-X degrees of
freedom.

\begin{itemize}
\item Fixed-X degrees of freedom ranges from 0 to $\infty$.\footnote{In fact,
    negative values are also allowed, but we implicitly rule this out in
    \cref{prop:df-reinterpretation}.}  That is, we cannot rule out arbitrarily
  large values of fixed-X degrees of freedom, a property that has been
  criticized by some authors (e.g., \citet{janson_fithian_hastie_2015}). In
  contrast, random-X degrees of freedom ranges from 0 to $n-1$. The reason for
  this is that the random-X optimism of least squares diverges at $d = n-1$,
  whereas the fixed-X optimism does not (and only diverges as $d \to\infty$). In
  other words, the random-X optimism of least squares sweeps the entire range of
  possible optimism values as we vary the number of features from 0 to $n-1$,
  and this places a finite upper limit on random-X degrees of freedom of $n-1$,
  achieved when the given predictor has infinite random-X optimism.    

\item The two metrics used in defining fixed-X and random-X degrees of freedom,
  namely, fixed-X and random-X optimism, scale differently with the number of
  parameters $d$ in the underlying reference model, least squares regression. 
  As we can see, \eqref{eq:opt-F-ls-extended} scales linearly with $d$, whereas 
  \eqref{eq:opt-R-ls-extended} scales nonlinearly. For large $d$ (close to $n$),
  the latter demonstrates ``diminishing returns'': large increases in random-X
  optimism only contribute small increases in random-X degrees of freedom. 
  \cref{fig:ruler} gives an illustration.  

\begin{figure}[htb]
\centering
\includegraphics[width=0.85\textwidth]{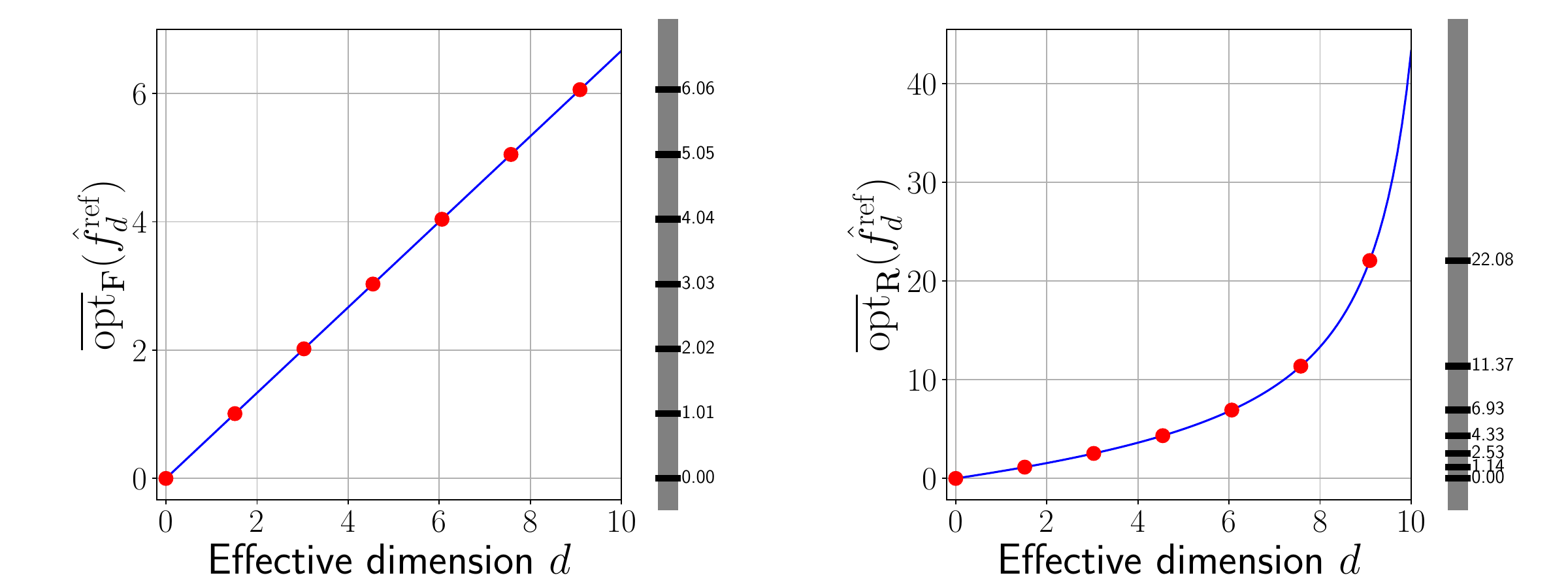}
\caption{An illustration of the metrics that underlie fixed-X and random-X
  degrees of freedom: fixed-X and random-X optimism of least squares regression
  on $d$ features.}  
\label{fig:ruler}
\end{figure}

\item The choice of Gaussian features \smash{$\tX_d$} in \cref{def:df-R}
  facilitates the calculation of the random-X optimism of least squares
  regression \eqref{eq:opt-R-ls}, since we can leverage well-known properties of
  the (inverse) Wishart distribution. Interestingly, we can see that the result
  \eqref{eq:opt-R-ls} does not depend on the feature covariance $\Sigma$. By
  standard arguments in random matrix theory, as explained in
  \cref{subsec:opt-universality}, the formula \eqref{eq:opt-R-ls} remains
  asymptotically valid (as $d/n \to \xi < 1$) for a broad class of feature
  models. 

\item The linear mean \smash{$\EE[\ty | \tX_d] = \tX_d \hspace{1pt}\beta$} in
  \cref{def:df-R} is important, but the assumption of Gaussian errors in
  \eqref{eq:df-R-response} is not. The calculations in \cref{app:opt-R-ls}
  actually only assume isotropic errors (i.e., \smash{$\ty = \tX_d \hspace{1pt}
    \beta + v$}, where \smash{$v | \tX_d$} has mean zero and covariance
  $\sigma^2 I$). Moreover, the random-X optimism \eqref{eq:opt-R-ls} does not
  depend on the underlying signal vector $\beta$ (due to the unbiasedness of
  underparameterized least squares regression), and only depends on the noise
  level $\sigma^2$. 
\end{itemize}

\subsection{An intrinsic version of model complexity}

The reference model we use in \cref{def:df-R} is least squares regression on 
\emph{well-specified} data, where the mean is linear in the covariates, as
can be seen in \eqref{eq:df-R-response}. As previously commented (and verified
in \cref{app:opt-R-ls}), the least squares predictor is unbiased in this case,
and its random-X prediction error and thus random-X optimism is comprised of
pure variance. 

Therefore, when we match the observed optimism to the reference one in
\eqref{eq:opt-R-match}, we are comparing \smash{$\opt\r(\hf)$}---which is
generically comprised of both bias and variance, to
\smash{$\bopt\r(\hf\rf_d)$}---which is made up of variance alone. This is
intentional. The notion of random-X degrees of freedom from \cref{def:df-R}
determines the complexity of the given predictor \smash{$\hf$} by incorporating
the ``full effect'' of the data at hand, allowing for potential model
misspecification to enter into the calculation of optimism. To emphasize, we
will sometimes refer to this as the \emph{emergent} random-X degrees of freedom.

Alternatively, we might want to match variance to variance in determining
degrees of freedom, i.e., we might want to exclude bias effects in calculating
the random-X optimism of the given model \smash{$\hf$}. This gives rise to a
different notion of model complexity, which we define next.

\begin{definition}
\label{def:df-R-intrinsic}
Under the exact same setup as in \cref{def:df-R}, draw $v \sim \cN(0, \sigma^2  
I)$, independent of everything else. We define \smash{$\df\r\i(\hf) = d$} to 
be the unique $d \in [0,n-1]$ for which 
\begin{equation}
\label{eq:opt-R-intrinsic-match}
\opt\r(\hf(\cdot; X, v)) = \bopt\r(\hf\rf_d). 
\end{equation}
\end{definition}

The difference between \eqref{eq:opt-R-match}, \eqref{eq:opt-R-intrinsic-match}
is that the latter measures the random-X optimism of \smash{$\hf$} when it is
being trained and tested on ``pure noise'' $v \sim \cN(0, \sigma^2 I)$. Because
the random-X optimism of least squares does not depend on $\beta$ in
\eqref{eq:df-R-response}, note that we may set $\beta = 0$ and write
\eqref{eq:opt-R-intrinsic-match} equivalently as     
\[
\opt\r(\hf(\cdot; X, v)) = \bopt\r(\hf\ls(\cdot; \tX_d, v)).
\]
We call the quantity \smash{$\df\r\i(\hf)$} in \cref{def:df-R-intrinsic} the
\emph{intrinsic} random-X degrees of freedom of \smash{$\hf$}. It can be
interpreted as the model complexity that is intrinsic or inherent to the model
\smash{$\hf$}, a reflection of its ability to overfit to pure noise (calibrated
to that of least squares). 

In what follows, we will further examine the relationship between emergent and
intrinsic random-X degrees of freedom, and learn through theory and experiments
that the emergent notion is generally larger than the intrinsic one. In short,
the presence of bias generally ``adds complexity''.

\subsection{Universality of random-X optimism for least squares}
\label{subsec:opt-universality}

As is well-known to those versed in random matrix theory, the random-X
prediction error of least squares regression, for well-specified,
underparameterized data models, displays a remarkable degree of universality. 
This is studied in, e.g., \citet{girko1990theory, girko1995statistical,
  verdu1997multiuser, verdu1998multiuser, tse1999linear, tse2000linear,
  serdobolskii2001solution, serdobolskii2002unimprovable}, among others. Thus,
the random-X optimism also has a universal limit under proportional asymptotics,
as noted in \citet{rosset_tibshirani_2020}. For completeness, we relay this
precisely below. 

\begin{theorem}
\label{thm:opt-R-ls-limit}
Assume \smash{$\tX_d = Z \Sigma^{1/2}$} where $Z \in \RR^{n \times d}$ has
i.i.d.\ entries with zero mean, unit variance, and bounded moments up to order
$4 + \delta$ for some $\delta > 0$, and $\Sigma \in \RR^{d \times d}$ is an
arbitrary deterministic positive definite covariance matrix. Also assume for an
arbitrary deterministic signal vector $\beta \in \RR^d$, 
\[
\ty = \tX_d \hspace{1pt} \beta + v,
\quad \text{where $\EE[v | \tX_d] = 0$ and $\,\Cov[v | \tX_d] = \sigma^2 I$}. 
\]
Then as $n, d \to \infty$ such that $d / n \to \xi \in (0, 1)$, we have, almost
surely with respect to \smash{$\tX_d$},   
\[
\opt\r(\hf\ls(\cdot; \tX_d, \ty) \,|\, \tX_d) \to \sigma^2 \bigg( \xi +  
\frac{\xi}{1-\xi} \bigg),
\]
where \smash{$\opt\r(\hf\ls(\cdot; \tX_d, \ty) \,|\, \tX_d) = \EE[ (\ty_0 - 
  \tx_0^\top \hbeta\ls)^2 - \| \ty - \tX_d \hspace{1pt} \hbeta\ls \|_2^2 / n
  \,|\, \tX_d ]$} denotes the random-X optimism conditional on \smash{$\tX_d$}
(and \smash{$(\tx_0, \ty_0)$} is a test point that is i.i.d.\ to the training
data \smash{$(\tX_d, \ty)$}). 
\end{theorem}

\begin{proof}
Following the calculations in \cref{app:opt-R-ls} leads to
\begin{align*}
\opt\r(\hf\ls(\cdot; \tX_d, \ty) \,|\, \tX_d) 
&= \sigma^2 \big( d/n + \tr[ \Sigma (\tX_d^\top \tX_d)^{-1} ] \big) \\ 
&= \sigma^2 \big( d/n + \tr[ (Z^\top Z)^{-1} ] \big).
\end{align*}
Under the assumptions in the theorem, the quantity
\[
\tr[ (Z^\top Z)^{-1} ] = \frac{d}{n} \cdot \frac{1}{d} \tr\bigg[ \bigg(
\frac{Z^\top Z}{n} \bigg)^{-1} \bigg]
\]
has a universal limit, almost surely with respect to $Z$; see, e.g., Theorem
3.10 of \citet{bai_silverstein_2010}. Again from the calculations in
\cref{app:opt-R-ls}, if the entries of $Z$ are i.i.d.\ standard Gaussian, then 
\[
\EE\big[ \tr[ (Z^\top Z)^{-1} ] \big] = \frac{d}{n-d-1}. 
\]
This converges to $\xi / (1-\xi)$ as $d /n \to \xi$, which must thus also be the
universal almost sure limit in the general case, regardless of the distribution
of entries of $Z$. This yields the almost sure limit of the conditional optimism    
\[
\opt\r(\hf\ls(\cdot; \tX_d, \ty) \,|\, \tX_d) \to \sigma^2 \bigg( \xi + 
\frac{\xi}{1-\xi} \bigg),
\]
as claimed.
\end{proof}

\cref{thm:opt-R-ls-limit} reveals that the choice of Gaussian features in the 
reference optimism calculation, for either \cref{def:df-R} or
\cref{def:df-R-intrinsic}, is in a certain sense unimportant, because all
feature models of the form described in the theorem lead to the same asymptotic 
answer anyway.   

\subsection{Practical calculation of random-X degrees of freedom}
\label{subsec:practical}

The concept of random-X degrees of freedom, from \cref{def:df-R}, is a
population-level quantity---it depends on the random-X optimism
\smash{$\opt\r(\hf)$}, which of course itself depends on the (unknown) joint
distribution of the features and response. To estimate \smash{$\df\r(\hf)$} in
practice, we need to first estimate \smash{$\opt\r(\hf)$}, which we can do by
estimating random-X prediction error using (say) cross-validation and then
subtracting off the observed training error. We also need to estimate the noise
level $\sigma^2$, which is an equally (if not more) difficult task, but as a
proxy we can use the random-X prediction error of the best-predicting model we
have for the task at hand. Given such estimates \smash{$\hopt\r(\hf)$} and 
\smash{$\hsigma^2$}, we set up the sample analog of the matching equation
\eqref{eq:opt-R-match},      
\begin{equation}
\label{eq:opt-R-estimated-match}
\hopt\r(\hf) = \hsigma^2 \bigg( \frac{d}{n} + \frac{d}{n-d-1} \bigg),
\end{equation}
solve for $d$, and set \smash{$\hdf\r(\hf) = d$}. 

To estimate intrinsic random-X degrees of freedom, from
\cref{def:df-R-intrinsic}, we can follow the analogous steps. The only
difference is that we train the predictor \smash{$\hf$} on pure noise \smash{$v
  \sim \cN(0, \hsigma^2 I)$} (instead of the original response $y$) which alters 
our estimates of both random-X prediction error and training error. We set up
the sample analog of the matching equation \eqref{eq:opt-R-intrinsic-match},  
\begin{equation}
\label{eq:opt-R-intrinsic-estimated-match}
\hopt\r(\hf(\cdot; X, v)) = \hsigma^2 \bigg( \frac{d}{n} + \frac{d}{n-d-1}
\bigg), 
\end{equation}
solve for $d$, and set \smash{$\hdf\r\i(\hf) = d$}. 

Lastly, just to emphasize, we do not require the (estimated) random-X degrees of
freedom to be an integer in any of \eqref{eq:opt-R-match},
\eqref{eq:opt-R-intrinsic-match}, \eqref{eq:opt-R-estimated-match},
\eqref{eq:opt-R-intrinsic-estimated-match}. If desired, then one could of course
achieve this taking the integer ceiling \smash{$\lceil d \rceil$} of the
solution $d$ to the given matching equation. We find this unnecessary; note
that fixed-X degrees of freedom as originally defined in \eqref{eq:df-F} is also
not restricted to be an integer.

\section{Properties and connections}
\label{sec:properties}

We develop some basic properties of the random-X degrees of freedom proposals
from the previous section, and make connections to related ideas in the
literature.   

\subsection{Mapping optimism to degrees of freedom}
\label{subsec:mapping}

Reflecting on the matching equations \eqref{eq:opt-R-match},
\eqref{eq:opt-R-intrinsic-match}, \eqref{eq:opt-R-estimated-match},
\eqref{eq:opt-R-intrinsic-estimated-match}, each one is an equation of the form   
\[
x = \frac{d}{n} + \frac{d}{n-d-1}.
\]
The above is a quadratic equation in $d$. It is straightforward to check that it
has a unique solution in $[0, n-1]$ which we can write as $d = \omega_n(x)$,
where 
\begin{equation}
\label{eq:omega_n}
\omega_n(x) = \frac{2n-1 + nx - \sqrt{(2n-1 + nx)^2 - 4(n-1)nx}}{2}. 
\end{equation}
The function $\omega_n(x)$ is a map from normalized optimism $x$ to degrees of
freedom $d$. It is increasing, concave, and ranges from $0$ (at $x=0$) to $n-1$
(as $x \to \infty$). Each of the definitions of (estimated) random-X degrees of
freedom from the last section, given by solving \eqref{eq:opt-R-match},
\eqref{eq:opt-R-intrinsic-match}, \eqref{eq:opt-R-estimated-match}, or
\eqref{eq:opt-R-intrinsic-estimated-match}, can be written concisely in terms of
$\omega_n$, and differ only in the form of normalized optimism that they use:
\begin{alignat*}{2}
\df\r(\hf) &= \omega_n\big( \opt\r(\hf) / \sigma^2 \big), \qquad
\df\r\i(\hf) &&= \omega_n\big( \opt\r\i(\hf) / \sigma^2 \big), \\
\hdf\r(\hf) &= \omega_n\big( \hopt\r(\hf) / \hsigma^2 \big), \qquad
\hdf\r\i(\hf) &&= \omega_n\big( \hopt\r\i(\hf) / \hsigma^2 \big). 
\end{alignat*}
Here and henceforth we write \smash{$\opt\r\i(\hf) = \opt\r(\hf(; \cdot,\tX_d,
v))$} for convenience, and will refer to this as intrinsic random-X optimism
(and similarly for the estimated version).

For large $n$, the function $\omega_n$ in \eqref{eq:omega} is well-approximated
by $\omega_n(x) \approx n \cdot \omega(x)$, where
\begin{equation}
\label{eq:omega}
\omega(x) = 1 + \frac{x}{2} - \sqrt{1 + \frac{x^2}{4}}.
\end{equation}
This function is increasing, concave, and ranges from $0$ (at $x=0$) to $1$ (as
$x \to \infty$). See \cref{fig:omega} for a visualization. The precise
relationship between $\omega_n$ and $\omega$ is that, for any fixed $x$, 
\begin{equation}
\label{eq:omega_approx}
|\omega_n(x) / n - \omega(x)| \to 0, \quad \text{as $n \to \infty$}, 
\end{equation}
which is verified in \cref{app:omega_approx}.

Finally, a calculation involving L'H{\^o}pital's rule can be used to show
$\omega(x) / (x/2) \to 1$ as $x \to 0^+$. In other words, for small values of
normalized optimism $x$ and large $n$ we have $d = \omega_n(x) \approx n
\omega(x) \approx n x/2$, which mirrors the relationship in the fixed-X setting
\eqref{eq:opt-df-F}.

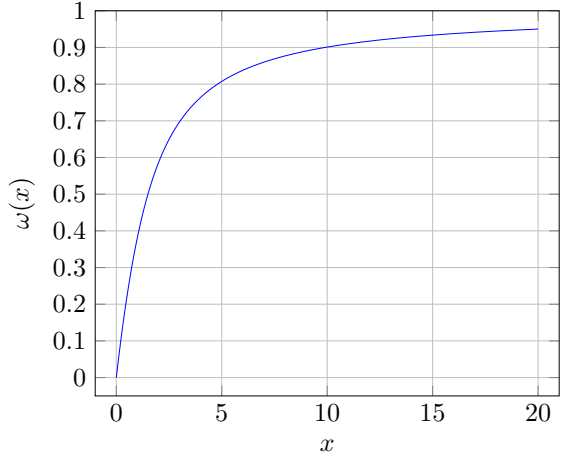
\begin{figure}[tb]
\centering
\begin{tikzpicture}[scale=0.9]
  \begin{axis}[
    xlabel = $x$,
    ylabel = $\omega(x)$,
    ymin = -0.05,
    ymax = 1,
    xmin = -1,
    xmax = 21,
    ytick = {0,0.1,...,1},
    grid=both,
    ]
    \addplot[blue, samples=500, domain=0:20] {1 + x/2 - sqrt(1 + (x^2)/4)};
  \end{axis}
\end{tikzpicture}
\caption{Plot of $\omega$ in \eqref{eq:omega}, which maps from normalized
  optimism (optimism divided by $\sigma^2$) to normalized degrees of freedom
  (degrees of freedom divided by $n-1$).} 
\label{fig:omega}
\end{figure}

\subsection{Linear smoothers}
\label{subsec:smoother}

Let \smash{$\hf$} be a linear smoother, which means that we can write
\begin{equation}
\label{eq:smoother}
\hf(x; X, y) = L_X(x)^\top y,
\end{equation}
for a weight function $L_X : \RR^p \to \RR^n$ that is allowed to depend on the
training features $X$, but not the training response $y$. For convenience, we
will write  
\[
L_X(X) = \begin{bmatrix} L_X(x_1)^\top \\ \vdots \\ L_X(x_n)^\top \end{bmatrix}
\in \RR^{n \times n}. 
\]
Similarly, for a function $g : \RR^p \to \RR$, we will write $g(X) =
(g(x_1),\dots, g(x_n)) \in \RR^n$ for the row-wise application of $g$ to $X$.
In this notation, we can rewrite the data model \eqref{eq:data_model} more
compactly as
\begin{equation}
\label{eq:data_model2}
y = f(X) + \eps,
\end{equation}
where $\EE[\eps] = 0$ and $\Cov[\eps] = \sigma^2 I$. 

The following proposition provides closed-form expressions for random-X optimism
and degrees of freedom for linear smoothers. 

\begin{proposition}
\label{prop:smoother}
For the linear smoother \eqref{eq:smoother}, its intrinsic and emergent random-X
optimism are  
\begin{align}
\label{eq:opt-R-smoother-intrinsic}
\opt\r\i(\hf) &= \sigma^2 \EE \bigg[ \frac{2}{n} \tr[L_X(X)] + \EE
[L_X(x_0)^\top L_X(x_0) \, | \, X] - \frac{1}{n} \tr[L_X(X)^\top
L_X(X)] \bigg], \\  
\label{eq:opt-R-smoother-emergent}
\opt\r(\hf) &= \opt\r\i(\hf) + \EE \bigg[ \EE \big[ (f(x_0) - L_X(x_0)^\top
f(X))^2 \, | \, X \big]  - \frac{1}{n} \| (I - L_X(X)) f(X) \|_2^2 \bigg]. 
\end{align}
Consequently, the intrinsic and emergent random-X degrees of freedom are given
by dividing by $\sigma^2$ and applying $\omega_n$ in \eqref{eq:omega_n}. 
\end{proposition}

The calculations to derive \eqref{eq:opt-R-smoother-intrinsic},
\eqref{eq:opt-R-smoother-emergent} are standard; they are based on the
bias-variance decomposition of random-X prediction error for linear smoothers,
which is found in many places in the literature. In the next subsection, we
draw a connection to \citet{rosset_tibshirani_2020}, whose work provides a
framework that allows us to easily verify the optimism results
\eqref{eq:opt-R-smoother-intrinsic}, \eqref{eq:opt-R-smoother-emergent}.

It is worth noting that the intrinsic optimism for a linear smoother
\eqref{eq:opt-R-smoother-intrinsic} is directly proportional to $\sigma^2$. As
a result, the intrinsic random-X degrees of freedom does not depend on
$\sigma^2$. 

It is also worth noting that for an interpolating linear smoother, we have
$L_X(X) = I$. In this case, intrinsic and emergent optimism simplify to 
\begin{align*}
\opt\r\i(\hf) &= \sigma^2 \big( 1 + \EE [L_X(x_0)^\top L_X(x_0)] \big), \\
\opt\r(\hf) &= \opt\r\i(\hf) +\EE \big[ (f(x_0) - L_X(x_0)^\top f(X))^2 \big].
\end{align*}
As a result we can see that intrinsic and emergent random-X degrees of freedom
(given by dividing by $\sigma^2$ and applying $\omega_n$) are each able to
distinguish between interpolating linear smoothers, unlike fixed-X degrees of
freedom, which always equals $n$ for an interpolator, recalling
\eqref{eq:df-F-interpolator}.

\subsection[Connection to Rosset and Tibshirani (2020)]{Connection to
  \citet{rosset_tibshirani_2020}}   
\label{subsec:rosset-tibshirani-2020}

\citet{rosset_tibshirani_2020} proposed the following decomposition of random-X
optimism, for an arbitrary predictor \smash{$\hf$}: 
\begin{equation}
\label{eq:opt-F-opt-R}
\opt\r(\hf) = \EE[\opt\f(\hf)] + B^+(\hf) + V^+(\hf).
\end{equation}
The first expectation on the right-hand side above is with respect to the
training covariates $X$, and the next two terms \smash{$B^+(\hf), V^+(\hf)$}
are called the \emph{excess bias} and \emph{excess variance} of \smash{$\hf$}, 
respectively, defined as:
\begin{align}
\label{eq:B+}
B^+(\hf)  &= \EE\big[ (f(x_0) - \of(x_0))^2 \big] - \EE\bigg[ \frac{1}{n} \|
f(X) - \of(X) \|_2^2 \bigg], \\   
\label{eq:V+}
V^+(\hf) &= \EE\big[ \Var[\hf(x_0) | X,x_0] \big] - \EE\bigg[ \frac{1}{n} 
\tr(\Cov[\hf(X) | X ]) \bigg],
\end{align}
where we abbreviate \smash{$\of(X) = \EE[\hf(X) | X]$} and \smash{$\of(x_0)
=\EE[\hf(x_0) | X,x_0]$}. The relationship \eqref{eq:opt-F-opt-R} follows from
expressing the random-X and fixed-X prediction errors of \smash{$\hf$} into bias
and variance terms, and then comparing the two decompositions:
\smash{$B^+(\hf)$} represents the difference in random-X and fixed-X squared
bias, and \smash{$V^+(\hf)$} the difference in random-X and fixed-X variance.

Though the decomposition \eqref{eq:opt-F-opt-R} is general, we now describe its
implications for linear smoothers in particular. For \smash{$\hf$} as in
\eqref{eq:smoother}, fixed-X degrees of freedom is simple to compute:
\[
\df\f(\hf) = \frac{1}{\sigma^2} \tr( \Cov[L_X(X) \hspace{1pt} y, y \,|\, X] = 
\tr[L_X(X)]. 
\]
Based on \eqref{eq:opt-df-F}, this gives a simple formula for fixed-X optimism:
\smash{$\opt\f(\hf) = (2 \sigma^2/ n) \tr[L_X(X)]$}. We can plug this into
\eqref{eq:opt-F-opt-R} (after integrating over $X$), along with excess bias and
variance calculations, to verify the random-X optimism claims in
\eqref{eq:opt-R-smoother-intrinsic}, \eqref{eq:opt-R-smoother-emergent}:
beginning with the intrinsic case, where we set $f = 0$, it is not hard to see
the excess bias is zero and we only need to compute \smash{$V^+(\hf)$}, which is
given by the latter two terms in \eqref{eq:opt-R-smoother-intrinsic}; as for the
emergent case, we add in \smash{$B^+(\hf)$}, which is given by the latter two
terms in \eqref{eq:opt-R-smoother-emergent}. This completes the proof of
\cref{prop:smoother}.

It is worth emphasizing a result that appears in passing in the arguments from
the last paragraph: for a linear smoother, 
\begin{equation}
\label{eq:opt-R-smoother-emergent-intrinsic}
\opt\r(\hf) = \opt\r\i(\hf) + B^+(\hf).
\end{equation}
This is not true for a general predictor \smash{$\hf$}. For linear smoothers, it
holds for \emph{any} distribution of the error vector $\eps$ in the original
data model \eqref{eq:data_model2} (provided we maintain $\EE[\eps] = 0$ and
$\Cov[\eps] = \sigma^2 I$), even though the pure noise model used for intrinsic
optimism in \cref{def:df-R-intrinsic} specifies $v \sim \cN(0, \sigma^2I)$.
This is because the random-X optimism for a linear smoother depends only on
$\sigma^2$, the noise level, and not the distribution of the error $\eps$
itself.

The fact in \eqref{eq:opt-R-smoother-emergent-intrinsic} is important because,
together with monotonicity of the map $\omega_n$ in \eqref{eq:omega_n}, it tells
us when we should expect emergent degrees of freedom to be larger than intrinsic
degrees of freedom:
\begin{align*}
\df\r(\hf) \geq \df\r\i(\hf) &\iff \opt\r(\hf) \geq \opt\r\i(\hf) \\
&\iff B^+(\hf) \geq 0. 
\end{align*}
\citet{rosset_tibshirani_2020} established nonnegativity of \smash{$B^+(\hf)$}
for various predictors \smash{$\hf$}; the next proposition summarizes these
results and their implications for random-X degrees of freedom. 

\begin{proposition}
\label{prop:B+}
For any linear smoother defined by minimizing a penalized least squares
criterion, excess bias is always nonnegative, and hence emergent random-X
degrees of freedom always larger than intrinsic random-X degrees of
freedom. This includes:   
\begin{itemize}
\item least squares regression (underparameterized case); 
\item ridgeless least squares regression (overparameterized case);
\item ridge regression, for any regularization strength $\lambda \ge 0$; 
\item kernel ridge, smoothing splines, and thin-plate splines, for any
  regularization strength $\lambda \geq 0$.  
\end{itemize}
\end{proposition}

For nonlinear smoothers, such as the lasso, direct analysis of
\smash{$\df\r(\hf) - \df\r\i(\hf)$} (or its sign) does not appear to be as
generally tractable. However, as we will see later in
\cref{subsec:lasso-theory,subsec:lassoless-theory,subsec:convex-theory}, it is
possible to prove the excess bias is nonnegative asymptotically, under certain
assumptions on the feature matrix and response model.

\subsection[Connection to Luan et al.\ (2021)]{Connection to
  \citet{luan2021predictive}}  
\label{subsec:luan}

\citet{luan2021predictive} proposed an extension of classical fixed-X degrees of
freedom to the random-X setting, which they called ``predictive model'' degrees
of freedom. Their proposal is limited to linear smoothers. In the notation of
the \cref{subsec:smoother} above, it can be expressed as:
\begin{equation}
\label{eq:luan}
\df^{\mathrm{pm}}_X(\hf) = \tr[L_X(X)] + \frac{n}{2} \bigg( \EE[L_X(x_0)^\top
L_X(x_0) \, | \, X] - \frac{1}{n} \tr[L_X(X)^\top L_X(X)] \bigg). 
\end{equation}
Comparing this to \eqref{eq:opt-R-smoother-intrinsic}, we note that 
\[
\opt\r\i(\hf) = \frac{2 \sigma^2}{n} \EE[\df^{\mathrm{pm}}_X(\hf)],
\]
where the expectation on the right-hand side is with respect to $X$. Thus we
can see that, for linear smoothers, \citet{luan2021predictive} define a notion
of model complexity in terms of intrinsic random-X optimism by reusing the same
functional form that connects fixed-X degrees of freedom to fixed-X optimism
\eqref{eq:opt-df-F}. (Their follow-up work \citet{luan2022measuring} considers a
weighted version of \eqref{eq:luan} which allows for heteroscedastic noise.)

There are three differences worth pointing out, to the ideas in the current
paper. First, restricting our attention to intrinsic optimism for linear
smoothers, \citet{luan2021predictive} transform normalized intrinsic optimism
\smash{$x = \opt\r\i(\hf) / \sigma^2$} to degrees of freedom via the linear map
$x \mapsto nx / 2$, whereas we use the nonlinear map $x \mapsto \omega_n(x)$,
with $\omega_n$ as defined in \eqref{eq:omega_n}, for what we call intrinsic
random-X degrees of freedom. Recalling the discussion in \cref{subsec:mapping},
we have $\omega_n(x) \approx nx / 2$ for small values of $x$, but for large
values of $\omega_n$ behaves quite differently, and it saturates at $n-1$.

Second, still restricting our attention to linear smoothers, we also consider
another (usually larger) notion of model complexity that stems from
incorporating bias into random-X optimism, which we call emergent random-X
degrees of freedom.

Third, the concepts of emergent and intrinsic random-X degrees of freedom in
\cref{def:df-R,def:df-R-intrinsic} do not require \smash{$\hf$} to be a linear
smoother and allow it to be arbitrary. This is possible because the core
motivation for these proposals is to match random-X optimism between the given
model and a reference model, which we take to be least squares. Being able to
carry out this matching does not require special knowledge of any sort about the
given predictor \smash{$\hf$} (beyond being able to estimate its random-X
optimism, in practice).

\section{Case studies: theory}
\label{sec:theory}

In this section, we pass through various standard prediction models, and develop
some theory on random-X degrees of freedom in each case. 

\subsection{Ridge regression}
\label{subsec:ridge-theory}

Recall the ridge regression predictor, given a response vector $y$ and feature
matrix $X$, is defined as \smash{$\hf^\ridge_\lambda(x) = x^\top
  \hbeta^\ridge_\lambda$}, where \smash{$\hbeta^\ridge_\lambda = (X^\top X / n
  + \lambda I)^{-1} X^\top y / n$} and $\lambda > 0$ is a tuning parameter.  
The coefficient vector \smash{$\hbeta^\ridge_\lambda$} equivalently solves the
following $\ell_2$-regularized least squares problem: 
\begin{equation}
\label{eq:ridge-opt}
\hbeta^\ridge_\lambda = \argmin_{b \in \RR^p} \, \frac{1}{n} \| y - X b \|_2^2 +
\lambda \|b\|_2^2. 
\end{equation}
The ridge predictor is a linear smoother, with $L_X(x) = X (X^\top X / n +
\lambda I)^{-1} x / n$. Hence, the results in \cref{prop:smoother} and
\cref{prop:B+} apply. Recall, these results explicitly characterize its
random-X degrees of freedom, and assert the nonnegativity of the amount of
degrees of freedom ``due to bias'' \smash{$\df\r(\hf^\ridge_\lambda) -
\df\r\i(\hf^\ridge_\lambda)$}, respectively. 

In this subsection, we derive two further characterizations, one finite-sample
and one asymptotic. The first, finite-sample property concerns the behavior of
intrinsic random-X degrees of freedom as a function of the regularization
parameter $\lambda$. 

\begin{proposition}
\label{prop:ridge-monotonicity}
For the ridge predictor \smash{$\hf^\ridge_\lambda$} with tuning parameter
$\lambda > 0$, its intrinsic random-X degrees of freedom
\smash{$\df\r\i(\hf^\ridge_\lambda)$} is monotonically decreasing in $\lambda$,
and \smash{$\df\r\i(\hf^\ridge_\lambda) \to 0$} as $\lambda \to \infty$. 
\end{proposition}

The proof of \cref{prop:ridge-monotonicity} is elementary, and deferred to
\cref{app:ridge-monotonicity}. Numerical illustrations of the results can be
found in \cref{app:ridge-illustration}.

The second property gives asymptotic equivalents for emergent and intrinsic
random-X degrees of freedom. In preparation for this, we first state our
assumptions on the distribution of the features and response variable. These
assumptions are similar to those used in \cref{thm:opt-R-ls-limit}, and to those
used in the literature on analyzing ridge regression under proportional
asymptotics. %

\begin{assumption}
\label{asm:ridge}~ 
\begin{enumerate}
\item \label{asm:ridge-features} 
  The features satisfy $X = Z \Sigma^{1/2}$, where $Z \in \RR^{n  \times p}$ is
  a random matrix with i.i.d.\ entries having zero mean, unit variance, and
  bounded moments up to order $4 + \delta$ for some $\delta > 0$, and where
  $\Sigma \in \RR^{p \times p}$ is a deterministic positive definite covariance
  matrix whose eigenvalues are bounded above and below by $r_{\max} < \infty$
  and $r_{\min} > 0$, respectively. 

\item \label{asm:ridge-response}
  The response vector satisfies $y = f(X) + \eps$, where $f$ is centered
  (which means $\EE[f(x)] = 0$ for a draw $x$ from the feature distribution) 
  with bounded $L^{4+\delta}$ norm (which means $\EE[|f(x)|^q]^{1/q}$ is
  bounded for $q = 4 + \delta$) for some $\delta > 0$, and the noise vector
  $\eps \in \RR^n$ has i.i.d.\ entries with zero mean, variance $\sigma^2$, and
  bounded moments up to order $4 + \eta$ for some $\eta > 0$.    
\end{enumerate}
\end{assumption}

Note that we can always decompose the regression function as
\begin{equation}
\label{eq:f-decomposition}
f(x) = x^\top \beta + f\nl(x).
\end{equation}
Here $x^\top \beta$ is the projection of $f$ onto the space of functions linear
in $x$, i.e., it minimizes $\EE[(f(x) - x^\top b)^2]$ over $b \in \RR^p$, where 
recall we use $x$ for a draw from the feature distribution. By construction, the
components $x^\top \beta$ and $f\nl(x)$ are uncorrelated, though in general they 
are dependent. We denote the variance of the nonlinear component by
\smash{$\sigma^2\nl = \EE[|f\nl(x)|^2]$}.        

To introduce some additional notation, let $\gamma_n = p / n$, and for given
$\lambda, \gamma_n > 0$, let $\mu_n = \mu(\lambda; \gamma_n)$ be the unique 
solution to the fixed point equation: 
\begin{equation}
\label{eq:ridge-fixed-point-mu}
\mu_n = \lambda + \gamma_n \mu_n \otr[\Sigma (\Sigma + \mu_n I)^{-1}], 
\end{equation}
where here and in what follows, we abbreviate \smash{$\otr(A) = \tr(A) / p$} for  
$A \in \RR^{p \times p}$. We are now ready to state our asymptotic results. 

\begin{theorem}
\label{thm:ridge-asymptotics}
Consider the ridge predictor \smash{$\hf^\ridge_\lambda$} with tuning parameter
$\lambda > 0$, and assume
\[
0 < \liminf_{n \to \infty} \gamma_n \leq \limsup_{n \to \infty} \gamma_n <
\infty,
\]
where recall $\gamma_n = p/n$. Under \cref{asm:ridge}\ref{asm:ridge-features}
for fixed-X degrees of freedom and intrinsic random-X degrees of freedom, and
additionally \cref{asm:ridge}\ref{asm:ridge-response} for emergent random-X
degrees of freedom, we have the following asymptotic equivalences, where recall
$\omega$ is the function in \eqref{eq:omega}:
\begin{align}
\label{eq:ridge-df-F-asympequi}
\df\f(\hf^\ridge_\lambda) / n &\asympequi 1 - \lambda / \mu_n, \\  
\label{eq:ridge-df-R-i-asympequi} 
\df\r\i(\hf^\ridge_\lambda) / n &\asympequi \omega\big( (1 - \lambda^2 /
\mu_n^2) (V_n / D_n + 1) \big), \\  
\label{eq:ridge-df-R-asympequi}
\df\r(\hf^\ridge_\lambda) / n &\asympequi \omega\big( (1 - \lambda^2 / \mu_n^2)
(B_n /  D_n + (V_n / D_n + 1) (1 + \sigma^2\nl/\sigma^2)) \big).  
\end{align}
Here we use $a_n \asympequi b_n$ to mean $|a_n - b_n| \to 0$ as $n \to \infty$
(almost surely, if $a_n,b_n$ are random). Also,   
\begin{align}
\label{eq:ridge-Vn} 
V_n &= \gamma_n \otr[\Sigma^2 (\Sigma + \mu_n I)^{-2}], \\   
\label{eq:ridge-Bn} 
B_n &= \mu_n^2 \beta^{\top} (\Sigma + \mu_n I)^{-1}  \Sigma (\Sigma + \mu_n
  I)^{-1} \beta / \sigma^2, \\   
\label{eq:ridge-Dn} 
D_n &= 1 - \gamma_n \otr[\Sigma^2 (\Sigma + \mu_n I)^{-2}]. 
\end{align}
\end{theorem}

The proof of \cref{thm:ridge-asymptotics} is given in
\cref{app:ridge-asymptotics}. It is based on the exact asymptotic analysis the
of training and prediction errors of ridge regression in various settings
(fixed-X, intrinsic random-X, emergent random-X), which can be done following 
techniques developed and employed previously in \citet{dobriban_wager_2018,  
  hastie2022surprises, patil2023generalized, bach2024high,
  lejeune2024asymptotics, patil2024optimal}, among others. Numerical  
examination of the results in \cref{thm:ridge-asymptotics} can be found in 
\cref{app:ridge-illustration}.   

We now reflect on the interpretation of the asymptotic equivalences for ridge
degrees of freedom in \cref{thm:ridge-asymptotics}. Inspecting the result for
fixed-X degrees of freedom in \eqref{eq:ridge-df-F-asympequi}, observe that by  
\eqref{eq:ridge-fixed-point-mu} we can write its asymptotic (and deterministic) 
equivalent as      
\[
1 - \lambda / \mu_n = \gamma_n \otr[\Sigma (\Sigma + \mu_n I)^{-1}] = \tr[\Sigma
(\Sigma + \mu_n I)^{-1}] / n.
\]
We can see this as a (normalized) ``population-level'' degrees of freedom
for ridge regression, where we replace \smash{$\hSigma = X^\top X/n$} by
$\Sigma$ in the usual ``sample-level'' formula, \smash{$\tr[L_X(X)] / n =
  \tr[\hSigma (\hSigma + \lambda I)] / n$}. Furthermore, in the population-level
formula in the last display, we can see that the regularization level has been
changed from $\lambda$ to $\mu_n$. In other words, each regularization level
$\lambda$ at the sample-level induces a corresponding regularization level
$\mu_n$ at the population-level, determined by solving a fixed point equation 
\eqref{eq:ridge-fixed-point-mu}. If $\gamma_n = p/n \to 0$, then one can check
that $\mu_n \to \lambda$, as would be expected in the low-dimensional regime. In
general, we have that $\mu_n \geq \lambda$, with strict inequality in the 
proportional asymptotic regime $\gamma_n \to \gamma > 0$. Further properties of
$\mu_n$ can be found in \citet{patil2024optimal}.

It is interesting to note that the inflation ratio in the regularization
level, $(\mu_n - \lambda) / \mu_n = 1 - \lambda / \mu_n$, is precisely the  
asymptotic equivalent for fixed-X degrees of freedom in
\eqref{eq:ridge-df-F-asympequi}. This relationship is not limited to ridge 
regression and in fact it holds more generally for regularized estimators with  
convex penalties, as we will see in \cref{subsec:convex-theory}. 

The asymptotic equivalents for random-X degrees of freedom in
\eqref{eq:ridge-df-R-i-asympequi}, \eqref{eq:ridge-df-R-asympequi} also have 
nice interpretations. Note from \eqref{eq:ridge-df-F-asympequi} that $\lambda /
\mu_n$ is asymptotically equivalent to \smash{$1 - \df\f(\hf^\ridge_\lambda) /
  n$}. Thus the factor of $1 - \lambda^2 / \mu_n^2$ in both
\eqref{eq:ridge-df-R-i-asympequi}, \eqref{eq:ridge-df-R-asympequi} is \smash{$1
  - (1 - \df\f(\hf^\ridge_\lambda) / n)^2$}. The other terms in these
expressions $V_n$ and $B_n$ in \eqref{eq:ridge-Vn}, \eqref{eq:ridge-Bn} are
asymptotic (and deterministic) equivalents for prediction variance and squared
bias (scaled by the noise level $\sigma^2$) for population ridge regression, at
a regularization level $\mu_n$. The final factor that makes this work in $D_n$,
which we interpret next.   

The quantity \smash{$1 - D_n = \gamma_n \otr[\Sigma^2 (\Sigma + \mu_n I)^{-2}] 
  = \tr[\Sigma^2 (\Sigma + \mu_n I)^{-2}] / n$}, from \eqref{eq:ridge-Dn}, is a 
related notion of a (normalized) ``population-level'' degrees of freedom of a
linear smoother, where we square the smoothing matrix before taking the
trace. This has appeared in classic literature on additive models
\citep{buja1989linear, hastie_tibshirani_1990}, and in later analyses of linear
and ridge regression generalization \citep{zhang2005learning,
  caponnetto2007optimal, hsu14ridge}. The link between the prediction error of
ridge regression at regularization level $\lambda$, and a population ridge
estimator at an induced level $\mu_n$ through the factor $D_n$, was first
derived (using a heuristic argument) by \citet{sollich2001gaussian} in the
context of Gaussian processes. It has been recently rederived using the replica
method (again heuristic) in \citet{bordelon2020spectrum}, and using random
matrix theory in \citet{hastie2022surprises, cheng2022dimension, bach2024high},
among others. 

In our discussion above, we restricted $\lambda > 0$ for simplicity. But, as we  
can see from \eqref{eq:ridge-fixed-point-mu}, if $\mu_n > \lambda$ and we want
to keep $\mu_n$ small (yet still positive), then we can actually set $\lambda < 
0$. The greater the degree of overparameterization (higher $\gamma_n$), the more
the flexibility we have. This is at the heart of why small (i.e., zero or even
negative) values of $\lambda$ can lead to favorable prediction accuracy in the
overparameterized regime. Let us define \smash{$\mu_{\min}$}, as in
\citet{lejeune2024asymptotics}, to be the unique solution that satisfies 
\smash{$\mu_{\min} > -r_{\min}$} to the fixed point equation:   
\begin{equation}
\label{eq:mu-min}
1 = \gamma_n \otr[\Sigma^2 (\Sigma + \mu_{\min} I)^{-2}].
\end{equation}
Note from \eqref{eq:ridge-Dn}, \eqref{eq:mu-min} that \smash{$\mu_{\min}$} is
the value at which $D_n = 0$, i.e., both \smash{$\opt\r\i(\hf^\ridge_\lambda)$}
and \smash{$\opt\r(\hf^\ridge_\lambda)$} diverge to $\infty$ (equivalently,
both \smash{$\df\r\i(\hf^\ridge_\lambda) / n$} and
\smash{$\df\r(\hf^\ridge_\lambda) / n$} converge to 1). We will revisit the 
relation between $D_n$ and overfitting soon, in the context of ridgeless
regression.  

\subsection{Ridgeless regression}
\label{subsec:ridgeless-theory}

Next we study a special case of ridge regression when $\lambda \to 0^+$,
also known as ``ridgeless'' regression. This is defined by
\smash{$\hf^\ridge_0(x) = x^\top \hbeta^\ridge_0$}, where
\[
\hbeta^\ridge_0 = \lim_{\lambda \to 0^+} \hbeta^\ridge_\lambda = (X^\top
X)^\pinv X^\top y,
\]
and $A^\pinv$ is the usual (Moore-Penrose) pseudoinverse of a matrix
$A$. In the underparameterized case where $p \leq n$ (and $\rank(X) = p$), this
reduces to the ordinary least squares estimator. However, in the
overparameterized case where $p > n$ (and $\rank(X) = n$), there are infinitely
many solutions in the least squares problem, each achieving perfect training
error, and the ridgeless solution \smash{$\hbeta^\ridge_0$} can be interpreted
as the interpolator with minimum $\ell_2$ norm: 
\[
\hbeta^\ridge_0 = \argmin_{b \in \RR^p} \, \{ \| b \|_2 : y = X b \}.
\]
Ridgeless regression has been thrust into the spotlight, due to recent interest
in overparameterized machine learning and the study of double descent. See
\citet{bartlett2020benign, belkin2020two, hastie2022surprises}, among many
others.  

Continuing in the vein ridge analysis from the last subsection, we will study
the degrees of freedom of the ridgeless predictor in an asymptotic regime where
we let the sample size $n$ and the feature size $p$ diverge, while keeping their
ratio bounded. In preparation for this, for a given $\gamma_n > 1$, define
$\mu_n = \mu_n(0; \gamma_n)$ be the unique solution to the fixed point equation:    
\begin{equation}
\label{eq:ridgeless-fixed-point-mu}
1 = \gamma_n \otr[\Sigma (\Sigma + \mu_n I)^{-1}]. 
\end{equation}
Observe that \eqref{eq:ridgeless-fixed-point-mu} is the limiting case of
\eqref{eq:ridge-fixed-point-mu} as $\lambda \to 0^+$. We are now ready to state
our asymptotic results on ridgeless degrees of freedom.  

\begin{theorem}
\label{thm:ridgeless-asymptotics}
For the ridgeless predictor \smash{$\hf^\ridge_0$}, under the same assumptions
as \cref{thm:ridge-asymptotics}, we have the following asymptotic equivalences:   
\begin{align}
\label{eq:ridgeless-df-F-asympequi}
\df\f(\hf^\ridge_0) / n &\asympequi 
\begin{cases}
\gamma_n & \text{for $\gamma_n \leq 1$} \\
1 & \text{for $\gamma_n > 1$},
\end{cases} \\
\label{eq:ridgeless-df-R-i-asympequi} 
\df\r\i(\hf^\ridge_0) / n &\asympequi 
\begin{cases}
\gamma_n & \text{for $\gamma_n \leq 1$} \\
\omega(V_n / D_n + 1) & \text{for $\gamma_n > 1$},
\end{cases} \\
\label{eq:ridgeless-df-R-asympequi}
\df\r(\hf^\ridge_0) / n &\asympequi 
\begin{cases}
\omega\big( (\gamma_n + \gamma_n / (1 - \gamma_n)) (1 + \sigma^2\nl/\sigma^2)   
\big) & \text{for $\gamma_n \leq 1$} \\ 
\omega\big( B_n / D_n + (V_n / D_n + 1) (1 + \sigma^2\nl/\sigma^2) \big) &
\text{for $\gamma_n > 1$},
\end{cases} 
\end{align}
where $\mu_n$ is as defined in \eqref{eq:ridge-fixed-point-mu}, and all other 
quantities are as defined in \cref{thm:ridge-asymptotics}.
\end{theorem}

Note: if the response model is well-specified, or in other words,
\smash{$f\nl(x) = 0$} in \eqref{eq:f-decomposition}, then \smash{$\sigma^2\nl = 
0$}, so the emergent random-X degrees of freedom in
\eqref{eq:ridgeless-df-R-asympequi} reduces to (as expected):
\[
\omega(\gamma_n + \gamma_n / (1 - \gamma_n)) = \gamma_n.
\]
The check the equality above, recall $\omega(x)$ in \eqref{eq:omega} is the
value of $u$ that solves $x = u + u / (1-u)$.   

The proof of \cref{thm:ridgeless-asymptotics} is given in
\cref{app:ridgeless-asymptotics}, and numerical examination of the results can
be found in \cref{app:ridgeless-illustration}. It is interesting to note that
each of the intrinsic and emergent normalized random-X degrees of freedom curves
are continuous at $\gamma_n = 1$, even though the prediction error of ridgeless
regression blows up at $\gamma_n = 1$ (it has an essential discontinuity at this
point).

Our next result develops monotonicity properties of the asymptotic equivalents
for intrinsic and emergent random-X degrees of freedom for ridgeless regression.  

\begin{proposition}
\label{prop:ridgeless-monotonicity}
The following properties hold for the asymptotic equivalents from
\cref{thm:ridgeless-asymptotics}. 

\begin{enumerate}
\item The asymptotic equivalent for intrinsic random-X degrees of freedom
  \smash{$\df\r\i(\hf^\ridge_0) / n$} in \eqref{eq:ridgeless-df-R-i-asympequi} 
  is increasing in $\gamma_n$ on $(0,1)$, maximized at $\gamma_n = 1$, and 
  decreasing in $\gamma_n$ on $(1,\infty)$.

\item The asymptotic equivalent for emergent random-X degrees of freedom  
  \smash{$\df\r(\hf^\ridge_0) / n$} in \eqref{eq:ridgeless-df-R-asympequi} is 
  increasing in $\gamma_n$ on $(0,1)$, and maximized at $\gamma_n = 1$.
\end{enumerate}
\end{proposition}

The proof of \cref{prop:ridgeless-monotonicity} is in
\cref{app:ridgeless-monotonicity}, and numerical illustrations are in 
\cref{app:ridgeless-illustration}. 

Beyond what we discussed in the last subsection, we provide one more connection
between $D_n$ and overfitting in ridgeless regression.
\citet{mallinar2022benign} defined three categories: \emph{benign},
\emph{catastrophic}, and \emph{tempered} overfitting, based on whether the
excess random-X prediction error goes to $0$, $\infty$, or is bounded away from
$0$ and $\infty$, respectively. \citet{zhou2023agnostic} then showed that these
regimes can be characterized in terms of the spectrum of $\Sigma$, which
recovers the results of \citet{bartlett2020benign}, by connecting this to the
notion of effective rank. In terms of $D_n$, these regimes correspond to whether
$1/D_n$ goes to $1$, $\infty$, or is bounded away from $1$ and $\infty$,
respectively.

\subsection{Lasso regression}
\label{subsec:lasso-theory}

Many of the qualitative properties and relationships we observed for ridge and 
ridgeless regression degrees of freedom carry over to nonlinear smoothers too.
To see this, we first study lasso regression \citep{tibshirani1996regression},
which recall, is defined by \smash{$\hf^\lasso_\lambda(x) = x^\top
  \hbeta^\lasso_\lambda$}, where \smash{$\hbeta^\lasso_\lambda$} solves the  
following $\ell_1$-regularized least squares optimization problem, for a tuning
parameter $\lambda > 0$:         
\begin{equation}
\label{eq:lasso-opt}
\hbeta^\lasso_\lambda \in \argmin_{b \in \RR^p} \, \frac{1}{2} \| y - X b \|_2^2 + 
\lambda \|b\|_1. 
\end{equation}
The element notation above is used to emphasize the fact that the minimizer in
\eqref{eq:lasso-opt} is not unique in general. However, it is unique under
weak conditions, for example, if the columns of the feature matrix $X$ are in
general position \citep{tibshirani2013lasso}.

As indicated in \cref{sec:related-work}, the fixed-X degrees of freedom of the
lasso and various generalizations have been studied extensively. When the lasso
solution is unique, the fixed-X degrees of freedom of the lasso predictor is
the expected number of nonzero coefficients in the lasso solution
\citep{zou_hastie_tibshirani_2007, tibshirani_taylor_2012}. Here, we will derive
exact formulae for the limiting random-X degrees of freedom of the lasso
predictor under proportional asymptotics, where $n,p$ both diverge, and their 
ratio converges to a constant, $p/n \to \gamma \in (0, \infty)$. 

We begin by stating our assumptions on the training data; these will be more
restrictive than those in \cref{asm:ridge}, used for ridge, but are standard
when using approximate message passing (AMP) or the convex Gaussian minimax
theorem (CGMT) to analyze regularized M-estimators. 

\begin{assumption}
\label{asm:cgmt}~
\begin{enumerate}
\item \label{asm:cgmt-features} 
  The feature matrix $X$ has i.i.d.\ entries from $\cN(0,1/n)$.

\item \label{asm:cgmt-response}
  The response vector follows $y = X \beta + \eps$, where the signal vector
  $\beta \in \RR^p$ has i.i.d.\ entries from a distribution $F$ with bounded
  second moment, and the noise vector $\eps \in \RR^n$ has i.i.d.\ entries with
  zero mean and variance $\sigma^2$. 
\end{enumerate}
\end{assumption}

Note: under \cref{asm:cgmt}\ref{asm:cgmt-features}, the columns of $X$ will be
in general position almost surely, and hence the lasso solution will be unique 
almost surely.  

The limiting degrees of freedom of the lasso, under the assumptions stated
above, is determined by the solution of a nonlinear system. We introduce some
relevant notation. First, define 
\[
\soft(u; t) = 
\begin{cases}
u - t & \text{if $u > t$} \\
0 & \text{if $u \in [-t, t]]$} \\
u + t & \text{if $u < -t$}.
\end{cases}
\]
Next, for a fixed $\gamma \in (0, \infty)$, define $(\tau, \mu) \in \RR^2$ as
the unique solution to the nonlinear system:
\begin{align}
\label{eq:lasso-fixed-point-tau}
\tau^2 &= \sigma^2 + \gamma \EE [(\soft(B + \tau H; \mu) - B)^2], \\  
\label{eq:lasso-fixed-point-mu}
\mu &= \lambda + \gamma \mu \EE [\soft' (B + \tau H; \mu)],
\end{align}
where $B \sim F$ and $H \sim \cN(0, 1)$ are independent. This system is from 
\citet{bayati2011lasso}, who show that its solution determines the limiting
behavior of the lasso estimator. (We modify the form of the system slightly
in order to unify our presentation of ridge, lasso, and convex penalties.)
Moreover, we use $(\tau_0, \mu_0)$ to denote the solution in
\eqref{eq:lasso-fixed-point-tau}, \eqref{eq:lasso-fixed-point-mu}    
when we replace $F$ by a point mass at 0 (i.e., we set $B = 0$). We are ready to
state our asymptotic results. 

\begin{theorem}
\label{thm:lasso-asymptotics}
Consider the lasso predictor \smash{$\hf^\lasso_\lambda$} with tuning parameter 
$\lambda > 0$. Under \cref{asm:cgmt}, the following asymptotic equivalences
hold, as $n,p \to \infty$ such that $p/n \to \gamma \in (0, \infty)$, where
recall $\omega$ is the function in \eqref{eq:omega}:
\begin{align}
\label{eq:lasso-df-F-asympequi}
\df\f(\hf^\lasso_\lambda) /n &\asympequi 1 - \lambda / \mu, \\ 
\label{eq:lasso-df-R-i-asympequi}
\df\r\i(\hf^\lasso_\lambda) / n &\asympequi \omega\big( (1 - \lambda^2 / \mu^2)  
\tau_0^2 / \sigma^2 \big), \\  
\label{eq:lasso-df-R-asympequi}
\df\r(\hf^\lasso_\lambda) / n &\asympequi \omega\big( (1 - \lambda^2 / \mu^2) 
\tau^2 / \sigma^2 \big).
\end{align}    
Here we use $a_n \asympequi b_n$ to mean $|a_n - b_n| \to 0$ as $n \to \infty$
(in probability, if $a_n,b_n$ are random). 
\end{theorem}

\cref{thm:lasso-asymptotics} is a special case of a more general result on
regularized least squares estimators that we derive later in
\cref{thm:convex-asymptotics}. Numerical verification of
\cref{thm:lasso-asymptotics} is given in \cref{app:lasso-illustration}. 

We pause to interpret the lasso results above, and compare them to those on
ridge regression from \cref{thm:ridge-asymptotics}. It is instructive to rewrite 
\eqref{eq:lasso-df-F-asympequi} using \eqref{eq:lasso-fixed-point-mu}, which
gives  
\[
1 - \lambda / \mu = \gamma \EE [\soft'(B + \tau H ; \mu)] = p
\EE [\soft'(B + \tau H ; \mu)] / n. 
\]
In this reformulation, the factor $\mu$ again plays the role of an induced 
regularization amount at the ``population level'', analogous to $\mu_n$ in the 
ridge regression analysis. The right-hand side in the last display can thus be
viewed as the normalized (scaled by $n$) fixed-X degrees of freedom of the
lasso, with regularization parameter $\mu$, when it is fit on a population model
with orthogonal features and responses drawn according to the original linear
model, but with noise variance $\tau^2$.

As with ridge regression, if the aspect ratio diminishes: $p/n \to 0$ (i.e.,
$\gamma = 0$) then we have $\mu = \lambda$ and thus the induced regularization
level $\mu$ matches the original one $\lambda$ in the low-dimensional regime. In
general, however, we have $\mu > \lambda$ when $\gamma \in (0, \infty)$, which
mirrors the inflation of the effective regularization level in ridge regression
in the high-dimensional regime.

Just as with ridge regression, the fixed-X degrees of freedom of the lasso is
asymptotically \eqref{eq:lasso-df-F-asympequi} the inflation ratio in effective
regularization, $(\mu - \lambda) / \mu = 1 - \lambda / \mu$. However, the
following is a notable difference between the ridge and lasso fixed point
equations. For ridge regression, we can solve for $\mu_n$ in
\eqref{eq:ridge-fixed-point-mu} based on knowledge of $\Sigma$ only. In
particular, this means that $\mu$ does not depend on the signal, through either
its linear $\beta$ or nonlinear \smash{$f\nl$} parts. For lasso, we must solve 
for $(\tau, \mu)$ jointly in \eqref{eq:lasso-fixed-point-tau},
\eqref{eq:lasso-fixed-point-mu}, which depends on the signal distribution $F$
(via the draw $B \sim F$). A consequence of this difference is that the fixed-X
degrees of freedom for ridge \eqref{eq:ridge-df-F-asympequi} does not actually
depend on the signal, whereas for lasso \eqref{eq:lasso-df-F-asympequi} it
does. 

The expressions for random-X degrees of freedom in
\eqref{eq:lasso-df-R-i-asympequi}, \eqref{eq:lasso-df-R-asympequi} also have
interesting interpretations, which we leave to \cref{subsec:convex-theory}, when
we cover regularized least squares estimators more generally. 

Unlike ridge regression (which is a linear smoother), it is not possible to
establish monotonicity of intrinsic random-X degrees of freedom as a function of
$\lambda$ (recall \cref{prop:ridge-monotonicity}) or nonnegativity of the
random-X degrees of freedom ``due to bias'' (recall \cref{prop:B+}) for the
lasso, via elementary arguments. However, the results in
\cref{thm:lasso-asymptotics} allow us to infer such properties asymptotically.  

\begin{proposition}
\label{prop:lasso-monotonicity-nonnegativity}
The following properties hold for the asymptotic equivalents from
\cref{thm:lasso-asymptotics}. 

\begin{enumerate}
\item The asymptotic equivalent for intrinsic random-X degrees of freedom 
  \smash{$\df\r\i(\hf^\lasso_\lambda) / n$} in \eqref{eq:lasso-df-R-i-asympequi} 
  is monotonically decreasing in $\lambda$, and converges to 0 as $\lambda \to
  \infty$.    

\item The asymptotic equivalent for emergent random-X degrees of freedom  
  \smash{$\df\r(\hf^\lasso_\lambda) / n$} in \eqref{eq:ridgeless-df-R-asympequi}
  is always larger or equal to that for intrinsic random-X degrees of freedom   
  in \eqref{eq:lasso-df-R-i-asympequi}. 
\end{enumerate}
\end{proposition}

The proof of \cref{prop:lasso-monotonicity-nonnegativity} is in
\cref{app:lasso-monotonicity-nonnegativity}, and numerical illustrations are in
\cref{app:lasso-illustration}.

\subsection{Lassoless regression}
\label{subsec:lassoless-theory}

We now study a special case of lasso regression when $\lambda \to 0^+$, which we 
term ``lassoless'' regression. The simplest way to define this estimator is to
assume that the lasso solution is unique for each $\lambda$ (which recall, is
implied almost surely under \cref{asm:cgmt}\ref{asm:cgmt-features}), and define
the lassoless solution as 
\[
\hbeta^\lasso_0 = \lim_{\lambda \to 0^+} \hbeta^\lasso_\lambda,
\]
and correspondingly define the predictor \smash{$\hf^\lasso_0(x) = x^\top
  \hbeta^\lasso_0$}. When $p \leq n$ (and $\rank(X) = p$), this is no different
from the ordinary least squares estimator. Meanwhile, when $p > n$ (and
$\rank(X) = n$), the lassoless estimator as defined above is an interpolator
with minimum $\ell_1$ norm: 
\[
\hbeta^\ridge_0 \in \argmin_{b \in \RR^p} \, \{ \| b \|_1 : y = X b \}. 
\]
We note that even when the lasso solution is not unique, it is still possible to
construct a sequence of lasso solutions converging to a minimum $\ell_1$ norm
interpolator; see \citet{tibshirani2013lasso} for details. 

We will derive the asymptotics of random-X degrees of freedom for the lassoless 
predictor, in the proportional asymptotics model from the previous
subsection. In preparation for this, for $\gamma > 1$, let $(\tau, \mu) \in \RR^2$
be the unique solution to the nonlinear system:   
\begin{align}
\label{eq:lassoless-fixed-point-tau}
\tau^2 &= \sigma^2 + \gamma \EE [(\soft(B + \tau H; \mu) - B)^2], \\  
\label{eq:lassoless-fixed-point-mu}
1 &= \gamma \EE [\soft'(B + \tau H; \mu)],
\end{align}
where again $B \sim F$ and $H \sim \cN(0, 1)$ are independent. This system
is studied in \citet{li_wei_2021} (we modify its presentation to suit our
purposes), and is the limit of \eqref{eq:lasso-fixed-point-tau},
\eqref{eq:lasso-fixed-point-mu} as $\lambda \to 0^+$. Similar to our earlier 
convention, let $(\tau_0, \mu_0)$ denote the solution to the above system when
$B = 0$.  

\begin{theorem}
\label{thm:lassoless-asymptotics}
For the lassoless predictor \smash{$\hf^\lasso_0$}, under the same conditions
as \cref{thm:lasso-asymptotics}, we have the following asymptotic equivalences: 
\begin{align}
\label{eq:lassoless-df-F-asympequi}
\df\f(\hf^\lasso_0) / n &\asympequi 
\begin{cases}
\gamma & \text{for $\gamma \leq 1$} \\
1 & \text{for $\gamma > 1$},
\end{cases} \\
\label{eq:lassoless-df-R-i-asympequi}
\df\r\i (\hf^\lasso_0) &\asympequi
\begin{cases}
 \gamma & \text{for $\gamma \leq 1$} \\
\omega(\tau_0^2 / \sigma^2) & \text{for $\gamma > 1$},
\end{cases} \\
\label{eq:lassoless-df-R-asympequi}
\df\r(\hf^\lasso_0) &\asympequi
\begin{cases}
 \gamma & \text{for $\gamma \leq 1$} \\
\omega(\tau^2 / \sigma^2) & \text{for $\gamma > 1$}.
\end{cases}
\end{align}
\end{theorem}

The proof of \cref{thm:lassoless-asymptotics} is given in
\cref{app:lassoless-asymptotics}, and numerical verification of the results can 
be found in \cref{app:lassoless-illustration}. 

As before, in the ridgeless setting, we can leverage the asymptotics above to
develop monotonicity properties for intrinsic and emergent random-X degrees for
freedom in lassoless regression. 

\begin{proposition}
\label{prop:lassoless-monotonicity}
The following properties hold for the asymptotic limits from
\cref{thm:lassoless-asymptotics}. 

\begin{enumerate}
\item The asymptotic limit of intrinsic random-X degrees of freedom
  \smash{$\df\r\i(\hf^\lasso_0) / n$} in \eqref{eq:lassoless-df-R-i-asympequi} 
  is increas- ing in $\gamma$ on $(0,1)$, maximized at $\gamma = 1$, and 
  decreasing in $\gamma$ on $(1,\infty)$.

\item The asymptotic limit of emergent random-X degrees of freedom  
  \smash{$\df\r(\hf^\lasso_0) / n$} in \eqref{eq:lassoless-df-R-asympequi} is 
  increas- ing in $\gamma$ on $(0,1)$, and maximized at $\gamma = 1$.
\end{enumerate}
\end{proposition}

Note: recall that when $\gamma \leq 1$, the lassoless solution is just least 
squares, as is the ridgeless solution. Therefore the underparameterized
statements in the lassoless results above are duplicates of those in 
\cref{prop:ridgeless-monotonicity}. However, for ease of interpretation, we
leave the underparameterized cases in the presentation of 
\cref{prop:lassoless-monotonicity}. 

The proof of \cref{prop:lassoless-monotonicity} is given in
\cref{app:lassoless-monotonicity}, and numerical illustrations can be found in  
\cref{app:lassoless-illustration}. Our next and last subsection generalizes the
study of ridge and lasso estimators.

\subsection{Convex regularized least squares}
\label{subsec:convex-theory}

Given a proper closed convex function $\reg : \RR \to [0, \infty]$, consider
defining a regularized least squares estimator by  
\begin{equation}
\label{eq:convex-opt}
\hbeta^\convex_\lambda \in \argmin_{b \in \RR^p} \, \frac{1}{2} \sum_{i=1}^{n}
(y_i - x_i^\top b)^2 + \lambda \sum_{i=1}^p \reg(b_i). 
\end{equation}
for a tuning parameter $\lambda > 0$. The corresponding predictor is defined as
\smash{$\hf^\convex_\lambda = x^\top \hbeta^\convex_\lambda$}. Note that the
element notation above emphasizes the fact that the solution in
\eqref{eq:convex-opt} need not be unique; the theory below applies to any one of
its solutions. 

We will extend the degrees of freedom analysis in \cref{subsec:lasso-theory} to
the regularized estimator in \eqref{eq:convex-opt}. To introduce some relevant
notation, recall that the proximal operator $\reg$ is defined as 
\[
\prox_\reg(x; t) = \argmin_{z \in \RR} \, \frac{1}{2t} (x - z)^2 + \reg(z), 
\]
for a parameter $t > 0$. Still working under \cref{asm:cgmt} for our asymptotic
analysis, we will now describe the nonlinear system which generalizes
\eqref{eq:lasso-fixed-point-tau}, \eqref{eq:lasso-fixed-point-mu}: define
$(\tau, \mu) \in \RR^2$ to solve 
\begin{align}
\label{eq:cgmt-fixed-point-tau}
\tau^2 &= \sigma^2 + \gamma \EE [(\prox_\reg(B + \tau H; \mu) - B)^2], \\ 
\label{eq:cgmt-fixed-point-mu}
\mu &= \lambda + \gamma \mu \EE [\prox_\reg'(B + \tau H; \mu)],
\end{align}
where $B \sim F$ and $H \sim \cN(0, 1)$ are independent. This system is adapted
from  \citet{thrampoulidis_abbasi_hassibi_2018}, who show that its solution
determines the limiting behavior of the regularized estimator in
\eqref{eq:convex-opt}. (We modify the form of the system, with the full details 
given in the proof of our next result.) Moreover, we use $(\tau_0, \mu_0)$ to
denote the solution in \eqref{eq:lasso-fixed-point-tau},
\eqref{eq:lasso-fixed-point-mu} when we replace $F$ by a point mass at 0 (i.e.,
set $B = 0$). We are ready to state our asymptotic results.

\begin{theorem}
\label{thm:convex-asymptotics}
Consider the convex regularized predictor \smash{$\hf^\convex_\lambda$} with
tuning parameter $\lambda > 0$. Under \cref{asm:cgmt}, the following asymptotic 
equivalences hold, as $n,p \to \infty$ such that $p/n \to \gamma \in (0,
\infty)$:   
\begin{align}
\label{eq:convex-df-F-asympequi}
\df\f(\hf^\convex_\lambda) /n &\asympequi 1 - \lambda / \mu, \\ 
\label{eq:convex-df-R-i-asympequi}
\df\r\i(\hf^\convex_\lambda) / n &\asympequi \omega\big( (1 - \lambda^2 / \mu^2)  
\tau_0^2 / \sigma^2 \big), \\  
\label{eq:convex-df-R-asympequi}
\df\r(\hf^\convex_\lambda) / n &\asympequi \omega\big( (1 - \lambda^2 / \mu^2) 
\tau^2 / \sigma^2 \big). 
\end{align}    
Here we use $a_n \asympequi b_n$ to mean $|a_n - b_n| \to 0$ as $n \to \infty$
(in probability, if $a_n,b_n$ are random). 
\end{theorem}

Note that \cref{thm:lasso-asymptotics} is a special case of
\cref{thm:convex-asymptotics} when the regularizer is the $\ell_1$ norm, $\reg = 
\|\cdot\|_1$, and the proximal operator is soft-thresholding,
\smash{$\prox_{\|\cdot\|_1} = \soft$}. The proof of
\cref{thm:convex-asymptotics} is given in \cref{app:convex-asymptotics}.     

We now give an interpretation of the asymptotic limits
\eqref{eq:convex-df-R-i-asympequi}, \eqref{eq:convex-df-R-asympequi} for 
random-X degrees of freedom by drawing an analogy to generalized
cross-validation (GCV). Initially designed for linear smoothers, GCV scales
the training error by a factor that involves the trace of the smoothing
matrix. This can be understood more broadly, beyond linear smoothers, as a
fixed-X degrees of freedom adjustment. This leads to the following
approximation, writing \smash{$\err\t(\hf^\convex_\lambda)$} for the training
error of \smash{$\hf^\convex_\lambda$}: 
\begin{equation}
\label{eq:err-R-gcv-approx}
\err\r(\hf^\convex_\lambda) \approx \frac{\err\t(\hf^\convex_\lambda)}{(1 - 
  \df\f(\hf^\convex_\lambda) / n)^2}.
\end{equation}
For ridge regression \citep{patil_wei_rinaldo_tibshirani_2021}, and regularized
least squares with convex penalties \citep{bellec2023out} more broadly, this
approximation is exact under proportional asymptotics. To connect this to the
results in \cref{thm:convex-asymptotics}, consider    
\begin{align*}
\opt\r(\hf^\convex_\lambda) 
&= \err\r(\hf^\convex_\lambda) - \err\t(\hf^\convex_\lambda) \\
&\approx \err\r(\hf^\convex_\lambda) - (1 - \df\f(\hf^\convex_\lambda) / n)^2
\cdot \err\r(\hf^\convex_\lambda) \\ 
&\approx \err\r(\hf^\convex_\lambda) - (\lambda^2 / \mu^2)
\cdot \err\r(\hf^\convex_\lambda), 
\end{align*}
where the second line uses \eqref{eq:err-R-gcv-approx}, and the third line uses
\eqref{eq:convex-df-F-asympequi}. Since the parameters $\tau_0$ and $\tau$ 
are the limiting random-X prediction errors in the intrinsic and emergent
settings, respectively, the last line provides a way to understand
\eqref{eq:convex-df-R-i-asympequi}, \eqref{eq:convex-df-R-asympequi} (after
applying $\omega$ to map optimism to degrees of freedom).

\section{Case studies: experiments}
\label{sec:experiments}

We continue our study of degrees of freedom, now via numerical
experiments. Python code to reproduce our experiments is available at: \url{https://github.com/jaydu1/model-complexity}; additional results for 
$k$-nearest-neighbor (\knn) and random features regression are given in
\cref{app:additional-experiments}.

\subsection{Lasso regression}
\label{subsec:lasso-experiments}

Returning to the lasso, whose degrees of freedom we studied asymptotically in
the previous section, we now empirically compare its random-X degrees of freedom
to the (average) number of nonzero coefficients in the lasso solution. The
latter is known to be its fixed-X degrees of freedom
\citep{zou_hastie_tibshirani_2007, tibshirani_taylor_2012}, in general.    

We simulate data according to a sparse linear model $y_i = x_i^\T \beta +
\eps_i$, $i \in [n]$. The entries of $x_i \in \RR^p$ and $\eps_i$ are all
i.i.d.\ standard normal. The first $s$ entries of $\beta$ are equal to $\alpha$,
while the remaining are equal to zero. We choose $\alpha$ so that the
signal-to-noise ratio (SNR) is 1, and compute all quantities by averaging
over 500 repetitions (500 times drawing the simulated data set; and in each
repetition, we compute prediction errors using an independent test set of 1000 
samples). 

\begin{figure}[tb]
\includegraphics[width=\textwidth]{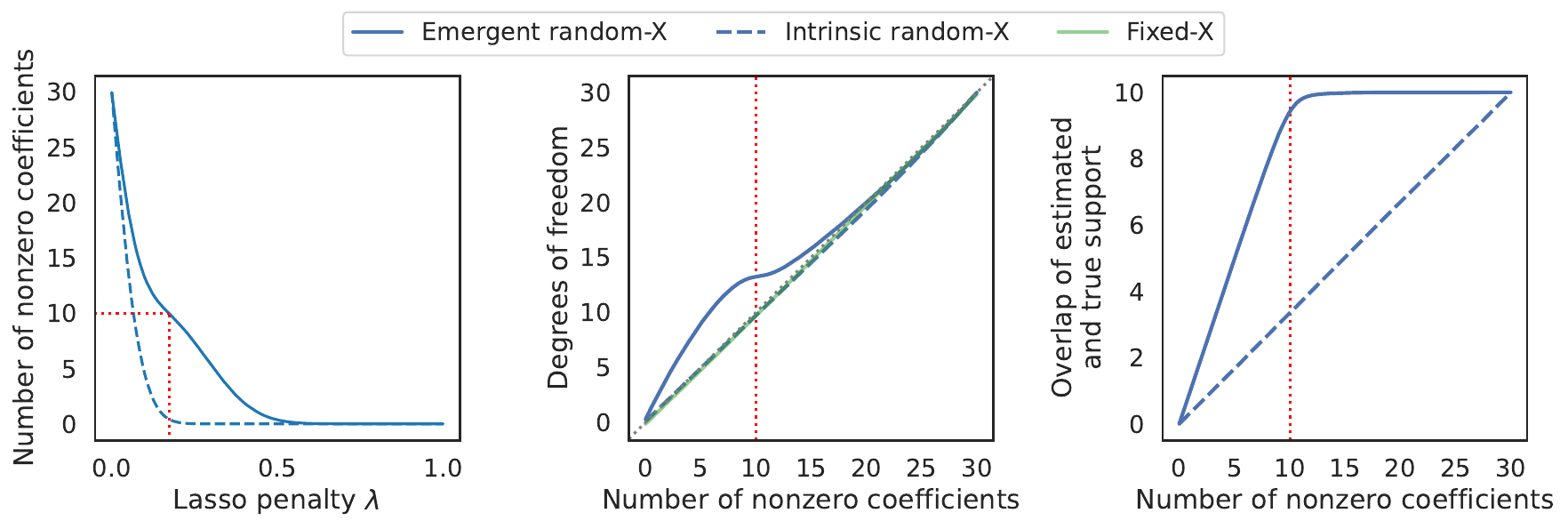}
\caption{Degrees of freedom of lasso predictors, parameterized by the average
  number of nonzero coefficients, in a problem setting with $n=200$, $p=30$, and
  sparsity level $s=10$.} 
\label{fig:lasso-underparam}

\bigskip
\centering
\includegraphics[width=\textwidth]{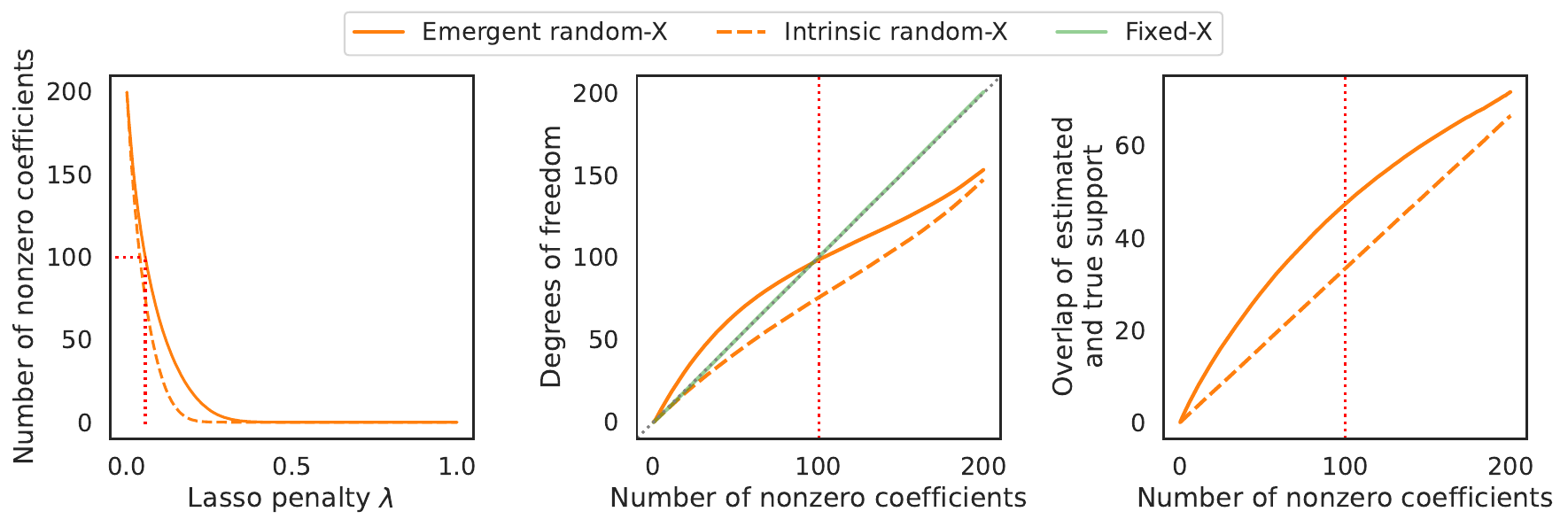} 
\caption{Degrees of freedom of lasso predictors, parameterized by the average
  number of nonzero coefficients, in a problem setting with $n=200$, $p=300$,
  and sparsity level $s=100$.}  
\label{fig:lasso-overparam}
\end{figure}

\cref{fig:lasso-underparam,fig:lasso-overparam} show the results 
for underparameterized and overparameterized cases, respectively. The
underparameterized case uses $n=200$, $p=30$, and $s=10$, while the
overparameterized case uses $n=200$, $p=300$, and $s=100$. As expected by the
theory, the fixed-X degrees of freedom of the lasso (middle panel) equals the
average number of nonzero coefficients in its solution, over the full path of
$\lambda$ values. Interestingly, in the underparameterized case, this also
appears to be true of the intrinsic random-X degrees of freedom: it also
coincides with the average number of nonzero lasso coefficients. 

By comparison, in this same setting, the emergent random-X degrees of freedom is
initially quite a bit larger, and then eventually settles back down to coincide
with the number of nonzero estimated coefficients, at higher levels of estimated
sparsity. As shown in the right panel, once it has estimated a little more than
$s = 10$ nonzero coefficients, it soon achieves near-perfect support recovery in
our simulations. Though this underparameterized problem setup with large $n$,
small $p$, and uncorrelated features is somewhat idealistic (the lasso is able
to identify a near-perfect support), it is nonetheless interesting to observe
the behavior the degrees of freedom ``due to bias'', emergent minus intrinsic
degrees of freedom, here: it is initially quite large, for lower levels of
estimated sparsity, and then it vanishes, at higher levels of estimated
sparsity. 

In the overparameterized case, with $p=300$ features, the lasso is only able to
recover about half of the true support once it has estimated $s=100$ nonzero
coefficients, as the right panel of \cref{fig:lasso-overparam} shows. From the
middle panel of the figure, we see that the intrinsic random-X degrees of
freedom grows increasingly smaller than the fixed-X degrees of freedom, as
number of nonzero coefficients increases. Meanwhile, the emergent random-X
degrees of freedom exceeds fixed-X degrees of freedom up until the point at
which the lasso exhibits roughly 100 nonzero coefficients, when it drops
below fixed-X degrees of freedom. Lastly, the degrees of freedom ``due to
bias'', the different in emergent and intrinsic degrees of freedom, is fairly
large throughout.

\subsection{Random forests}
\label{subsec:rf-experiments}

Next, we study random forests. Following the experimental setup in
\citet{belkin_hsu_ma_mandal_2019}, we use a single tree (rather than an average
of trees over subsamples of training data) up until the point of interpolation, 
increasing the maximum number of leaves \smash{$N_{\text{leaf}}^{\max}$} allowed
in the tree until we reach zero training error. After interpolation, we keep
\smash{$N_{\text{leaf}}^{\max}$} fixed and increase the number of trees
\smash{$N_{\text{tree}}$}.     

We draw data according to a linear model $y_i = x_i^\T \beta + \eps_i$, $i
\in [n]$, where each \smash{$x_i \sim \cN(0, \Sigma_{\textsc{ar1},
    \rho=0.25})$}, $\eps_i \sim \cN(0, 0.5^2)$, and $\beta$ is drawn uniformly
from the unit sphere in $\RR^d$. Here the covariance matrix
\smash{$\Sigma_{\textsc{ar1}, \rho}$} has entries $\rho^{|i-j|}$. The SNR in
this setup is 4. As before, we average all quantities over 500 repetitions (and 
in each one, compute prediction errors over an independent test set of size
1000). 

\cref{fig:random-forests} shows the results for $n=2000$ and $p=50$. As we can 
see in the left panel, the random-X prediction error initially decreases, then
it increases again as the number of leaves approaches the interpolation
threshold. After this point, it decreases as we increase the total number of
leaves by including more trees. As expected, the fixed-X degrees of freedom
increases before the interpolation threshold, while remaining constant beyond
the point, as shown in the middle panel. On the other hand, both the emergent
and intrinsic random-X degrees of freedom decrease after this threshold, and
generally remain much smaller than the trivial saturation value (of $n$ degrees
of freedom). The degrees of freedom ``due to bias'', emergent minus intrinsic,
is also consistently large throughout. 

\begin{figure}[p]
\includegraphics[width=\textwidth]{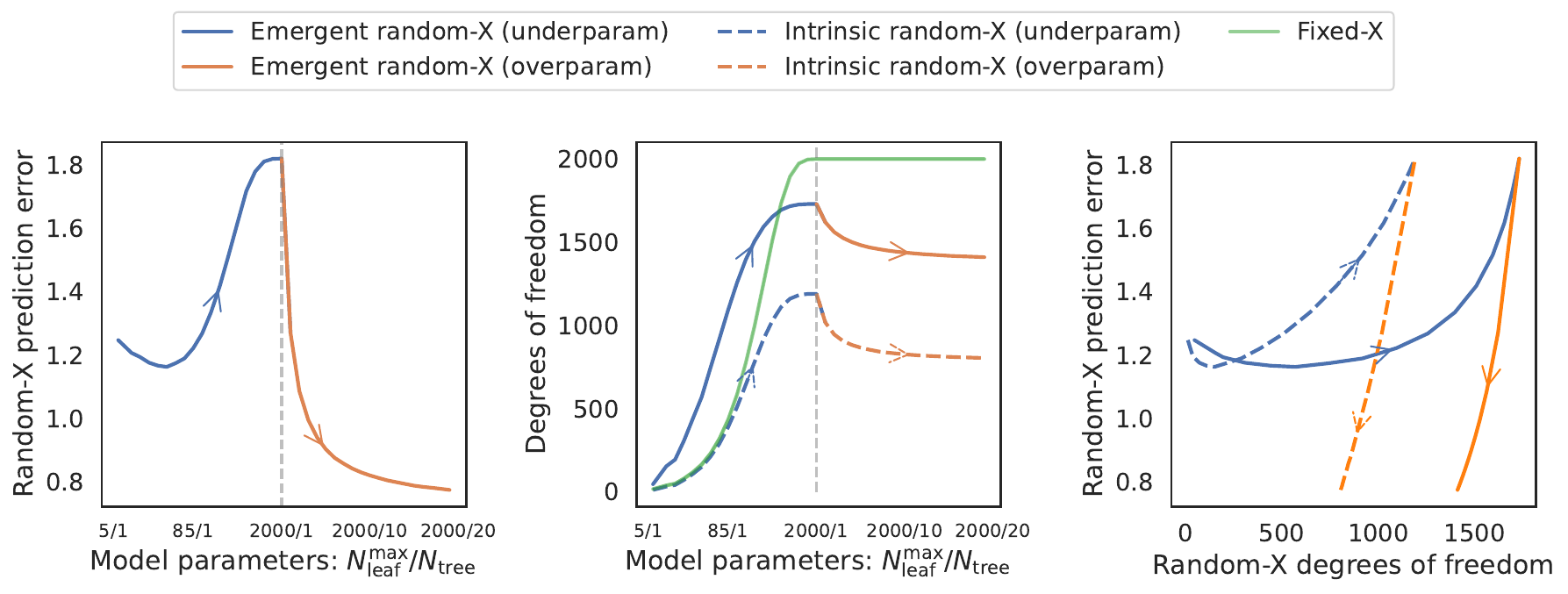}
\caption{Prediction error and degrees of freedom of random forest predictors, as
  we vary the number of trees \smash{$N_{\text{tree}}$} and the maximum number
  of leaves for each tree \smash{$N_{\text{leaf}}^{\max}$}, in a problem with
  $n=2000$, $p=50$.}  
\label{fig:random-forests}

\bigskip
\begin{center}
\includegraphics[width=0.9\textwidth]{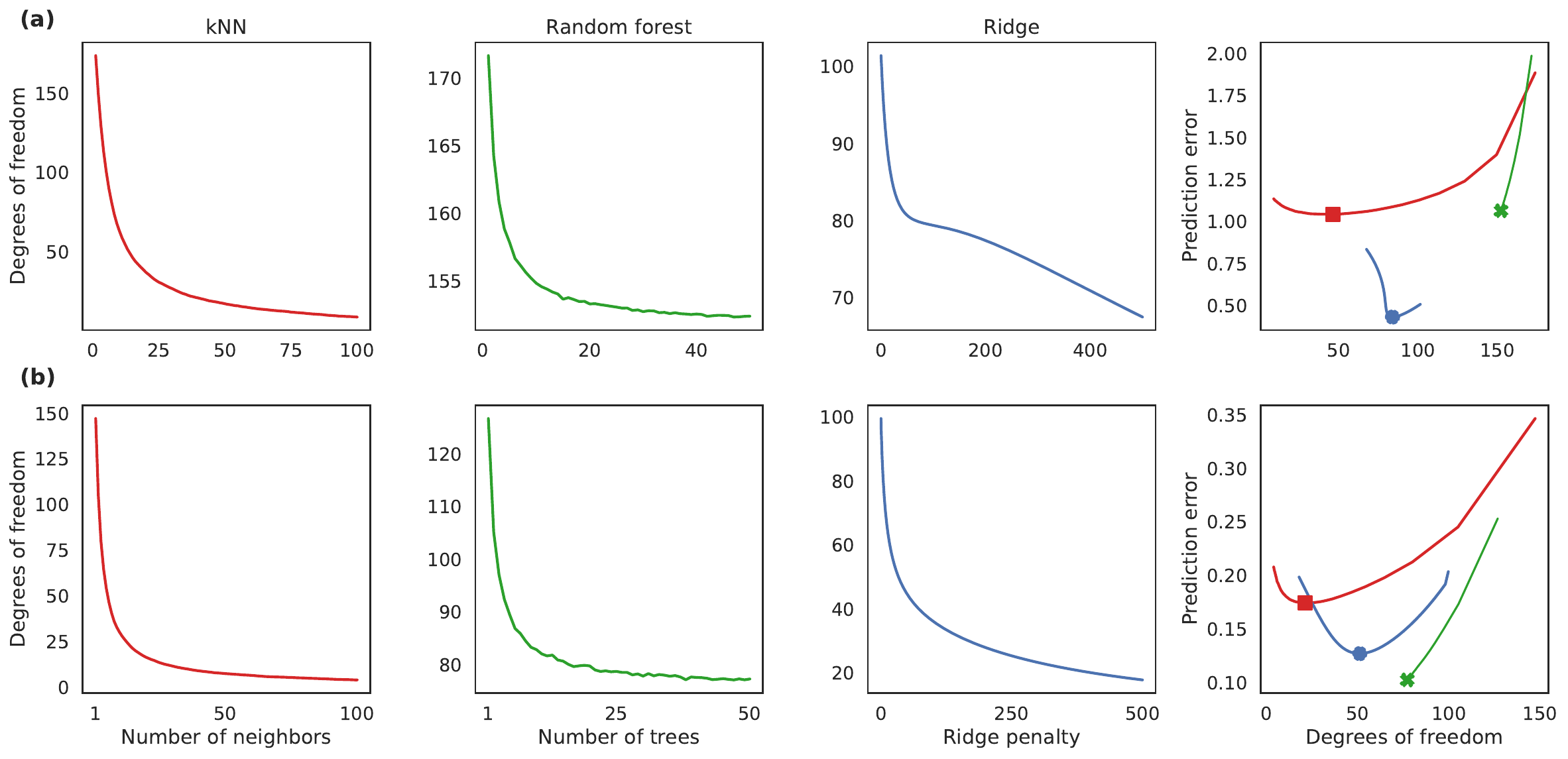}
\caption{Prediction error and degrees of freedom for ridge regression, \knn, and
  random forests. In both rows, $n=200$ and $p=100$. The top row displays data 
  drawn from a linear model, which favors ridge. The bottom displays data drawn 
  from a model that favors random forests.}   
\label{fig:comparisons-single-n}
\end{center}

\bigskip\hspace{-25pt}
\begin{minipage}[c]{0.65\textwidth}
\includegraphics[width=\textwidth]{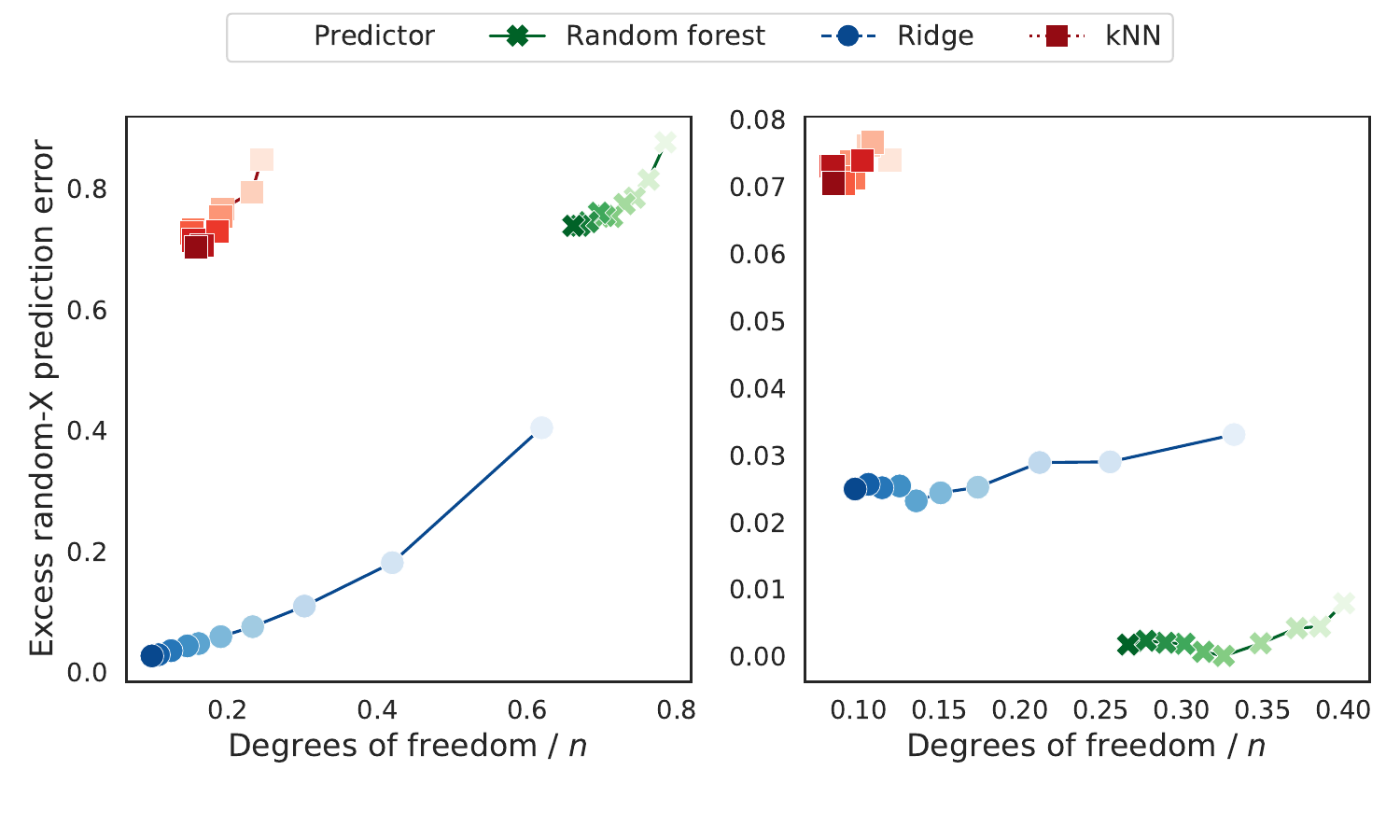}
\end{minipage}
\begin{minipage}[c]{0.4\textwidth}
\caption{Comparisons of excess prediction error and degrees of freedom for
  optimally-tuned ridge regression, \knn, and random forests, as the sample size
  $n$ varies. Lighter, more transparent colors indicate smaller $n$, whereas
  darker, more opaque colors indicate larger $n$. The left panel uses data
  simulated as in the top row of \cref{fig:comparisons-single-n}, and the right
  panel uses data simulated as in the bottom row of
  \cref{fig:comparisons-single-n}.}     
\label{fig:comparisons-varying-n}
\end{minipage}
\end{figure}

\subsection{Degrees of freedom comparisons}

So far we have mostly examined the behavior of degrees of freedom for
individual classes of models, and have drawn comparisons between members of one
class. Now, we shift our focus to comparing \emph{across} model classes.
\cref{fig:comparisons-single-n} studies such comparisons across ridge 
regression, \knn, and random forest predictors. The top row shows results for
data simulated from a linear model with $n=200$ and $p=100$, a setting that
favors ridge regression; the bottom row shows the results for data simulated
with the same $n,p$, but in a way that favors random forests. (This uses the
\texttt{make\_classification} function from \texttt{scikit-learn} v1.2.2; see
\citet{scikit-learn}). In the top row, we can see that optimally-tuned ridge
regression achieves the best random-X prediction error (rightmost panel), but
interestingly, does so at a much larger emergent random-X degrees of freedom
than optimally-tuned \knn. In the bottom row, the optimally-tuned random forest
achieves the best prediction error, and does so at a much larger degrees of
freedom than either ridge regression or \knn.

As a follow-up on the comparisons just discussed, one may naturally ask: can a 
model achieve both the best prediction error and emergent degrees of freedom,
simultaneously? In a sense, this would put the model on the ``Pareto frontier''
traced out by predictive accuracy versus complexity. Both ridge regression and
random forests achieve the best predictive accuracy (in top and bottom rows,
respectively) but fail to do so at the lowest complexity, in
\cref{fig:comparisons-single-n}. Yet, this is only a snapshot of their
performance at a given sample size. In \cref{fig:comparisons-varying-n}, we
examine the optimally-tuned ridge, \knn, and random forest predictors as we vary
the sample size $n$ from 100 to 1000. A each $n$, we measure the excess random-X
prediction error (the random-X prediction error minus the Bayes error) and the
normalized emergent random-X degrees of freedom (scaled by $n$) of each
optimally-tuned predictor. For the linear model simulation (corresponding to the 
left panel of \cref{fig:comparisons-varying-n}), ridge quickly becomes ``Pareto
optimal'' as $n$ increases, eventually demonstrating a lower emergent degrees of
freedom than \knn. For the simulation designed to favor random forests (right 
panel), random forests fail to be ``Pareto optimal'' at any $n$, as ridge and
\knn each offer a nontrivial tradeoff in balancing predictive accuracy versus 
complexity. (As a side note, it is interesting to note that the dynamic range of
the emergent degrees of  freedom of optimally-tuned \knn is very small, in both 
settings.)    

\section{Degrees of freedom decomposition}
\label{sec:decomposition}

In previous sections, we spoke frequently of the degrees of freedom ``due to
bias'', which refers to the difference in emergent and intrinsic random-X
degrees of freedom. Here we describe how this general idea---decomposing degrees
of freedom by attributing complexity to different sources of errors---can be
extended to problems involving distribution shift.  

We demonstrate the idea in the context of covariate shift. Given a predictor
\smash{$\hf$}, we consider four scenarios (\cref{tab:attribution} gives a
summary). In all cases, the reference model remains the least squares predictor
on well-specified data, as in \cref{def:df-R} or \cref{def:df-R-intrinsic}, but
the left-hand side of the matching equations \eqref{eq:opt-R-match} or 
\eqref{eq:opt-R-intrinsic-match} changes. 

\begin{enumerate}
\item The \emph{total emergent model} (both signal and covariate shift): for the
  left-hand side in \eqref{eq:opt-R-match}, the optimism is computed using
  random-X prediction error under covariate shift, and the result is denoted
  \smash{$\df^{11}(\hf)$}, which we simply call emergent degrees of freedom
  \smash{$\df\r(\hf)$}.    

\item The \emph{partial emergent model} (with signal but no covariate shift):
  for the left-hand side in \eqref{eq:opt-R-match}, the optimism is computed
  using random-X prediction error without covariate shift, and the result is
  denoted \smash{$\df^{10}(\hf)$}.    

\item The \emph{partial intrinsic model} (with no signal but with covariate
  shift): for the left-hand side in \eqref{eq:opt-R-intrinsic-match}, the
  optimism is computed using random-X prediction error without signal yet still  
  with covariate shift, and the result is denoted \smash{$\df^{01}(\hf)$}.    

\item The \emph{intrinsic model} (with no signal and no covariate shift): for
  the left-hand side in \eqref{eq:opt-R-intrinsic-match}, the optimism is
  computed using random-X prediction error without signal or covariate shift,
  and the result is denoted \smash{$\df^{00}(\hf)$}, which we simply call the
  intrinsic degrees of freedom \smash{$\df\r\i(\hf)$}.    
\end{enumerate}

\begin{table}[t]
\centering
\begin{tabularx}{0.5\textwidth}{YYY}
\toprule
\multirow{2}{*}{Signal presence} & \multicolumn{2}{c}{Covariate shift}
\\\cmidrule(lr){2-3}  
& \xmark & \cmark \\    
\midrule
\xmark & $\df^{00}(\hf)$ & $\df^{01}(\hf)$ \\ 
\cmark & $\df^{10}(\hf)$ & $\df^{11}(\hf)$ \\ 
\bottomrule
\end{tabularx}
\caption{Scenarios for decomposing degrees of freedom due to bias and covariate 
  shift.} 
\label{tab:attribution}
\end{table}

In order to attribute an amount of degrees of freedom to each source of
error---bias and covariate shift---we use a definition akin to Shapley values 
\citep{shapley1953value}:
\begin{align*}
\phi^\bias(\hf) &= \frac{1}{2} (\df\r(\hf) - \df^{01}(\hf)) + \frac{1}{2}
(\df^{10}(\hf) - \df\r\i(\hf)), \\
\phi^\cov(\hf) &= \frac{1}{2} (\df\r(\hf) - \df^{10}(\hf)) + \frac{1}{2}
(\df^{01}(\hf) - \df\r\i(\hf)). 
\end{align*}
Note that by construction (which is also a Shapley axiom called ``efficiency''),
we have:  
\[
\df\r(\hf) = \df\r\i(\hf) + \phi^\bias(\hf) + \phi^\cov(\hf).
\]
In other words, we have created a bona fide decomposition of the total emergent
degrees of freedom into constituent parts---attributed to variance, bias, and
covariate shift (first three terms above, respectively). The same idea can be
extended to an arbitrary number of sources of error.

As a simple example, we revisit the setup used in the first row of
\cref{fig:comparisons-single-n}, but introduce covariate shift by drawing test
features from a scaled and shifted version of the training feature distribution.
\cref{fig:attribution} displays the degrees of freedom of ridge, \knn, and
random forest predictors broken down into components due to variance, bias, and
covariate shift. \cref{fig:attribution-best} shows the same quantities, but
restricted to the optimally-tuned model within each class (which minimizes
out-of-distribution prediction error). We can see that the \knn predictor
exhibits the smallest intrinsic complexity; the ridge predictor exhibits the
smallest complexity due to bias (recall, the true model here is linear); and
quite interestingly, random forests display by the smallest complexity due to
covariate shift.

\begin{figure}[t]
\includegraphics[width=\textwidth]{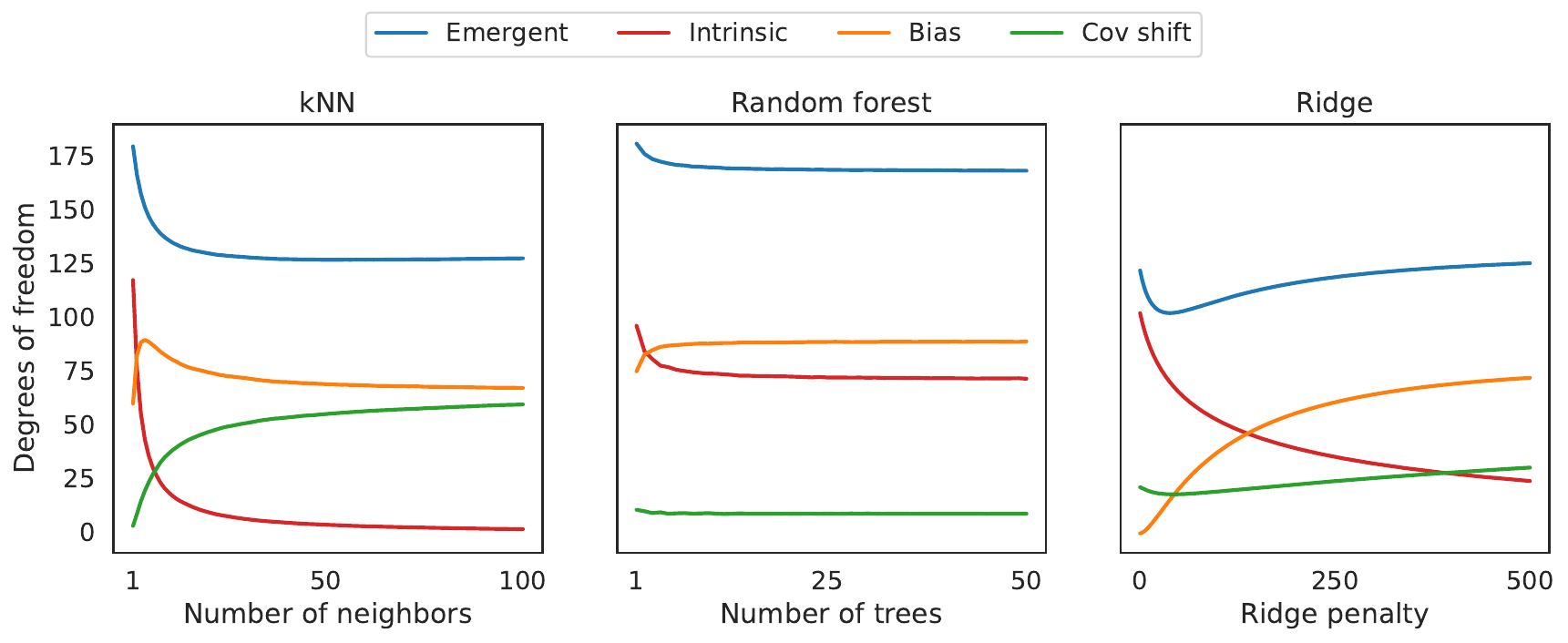}    
\caption{Decomposing degrees of freedom for ridge, \knn, and random forest
  predictors. The setup is as in the top row of \cref{fig:comparisons-single-n}
  but with covariate shift.} 
\label{fig:attribution}

\bigskip
\includegraphics[width=\textwidth]{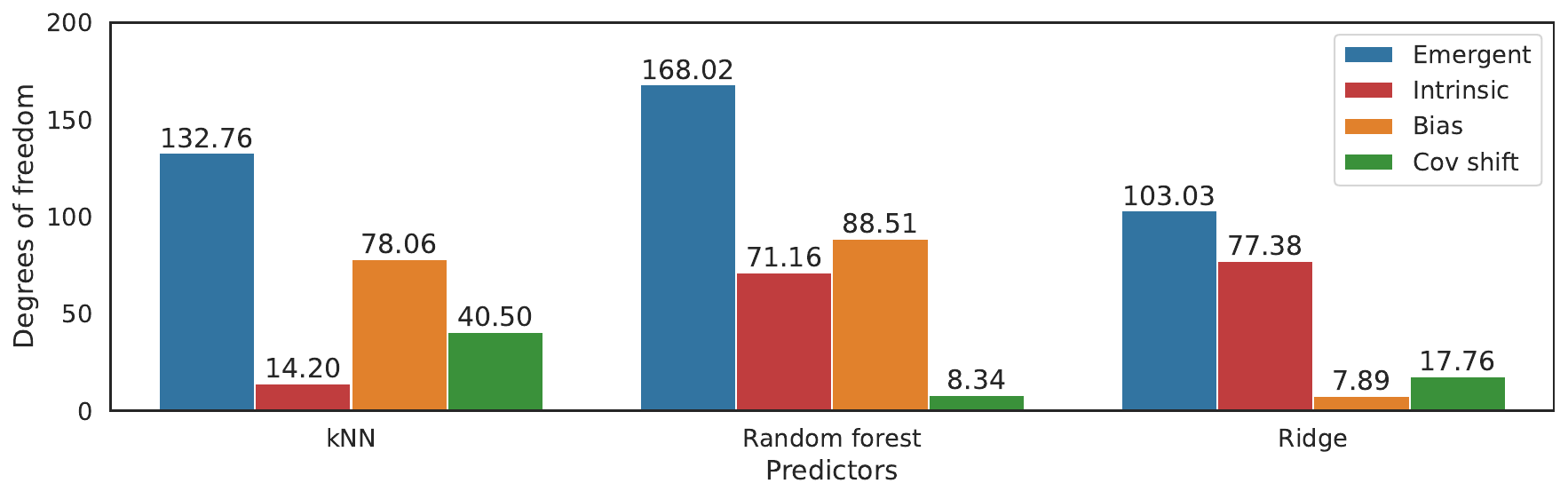}
\caption{Decomposition for the optimally-tuned model within each class in
  \cref{fig:attribution}.}  
\label{fig:attribution-best}
\end{figure}

\section{Discussion}
\label{sec:discussion}

A high-level summary of our proposal is as follows. In order to define the
complexity of an arbitrary prediction model, we consider two critical
components: a metric and a reference model. The metric quantifies complexity,
while the reference model provides context, which is analogous adding units to a
measurement. Precisely, we define the complexity of a given model as the number
of parameters in the reference model we require in order to obtain an identical
value of the metric.

Each choice of a metric and reference model gives rise to a different notion of
model complexity. In fact, if we take the metric to be fixed-X optimism (the
difference between random-X prediction error and training error) as the metric,
and least squares as the reference model, then this formulation reproduces the
classical notion of (effective) degrees of freedom. With this motivation in
mind, we focused on examining random-X optimism (the difference between random-X 
prediction error and training error) as the metric, and least squares as the
reference model, which allowed us to define a new random-X notion of degrees of 
freedom.

By changing the metric---to measure the random-X optimism when the given model
is run on pure noise, we can isolate the degrees of freedom due to variance,
which we call the intrinsic random-X degrees of freedom. Then, taking the
difference between the original random-X degrees of freedom and this intrinsic
version allows us to isolate the degrees of freedom due to bias. A similar idea
can be used to isolate the degrees of freedom due to other sources of error in
settings with distribution shift, such as covariate shift.

Of course, our choice to focus on regression, and metrics based on squared
error, does not reflect a fundamental restriction. An interesting direction for
follow-up work would be to use the framework we proposed in this paper in order
to study model complexity in classification.

\subsection*{Acknowledgements} 

Preliminary versions of our work were presented at a workshop on the Theory of 
Overparameterized Machine Learning (TOPML) in 2022, and the Joint Statistical 
Meetings (JSM) in 2023. We thank the participants for their comments and
feedback. 

PP and JHD thank Pierre Bellec and Takuya Koriyama for useful discussions about  
the degrees of freedom of convex regularized M-estimators. RJT thanks Trevor
Hastie, Saharon Rosset, and Rob Tibshirani for many insightful conversations
about optimism and model complexity over the years.

{
\RaggedRight
\spacingset{1}\small
\bibliographystyle{plainnat}
\bibliography{pratik}      
}

\clearpage
\appendix

\section{Proofs for \cref{sec:proposal,sec:properties}}

\subsection[Derivation of random-X optimism of least squares]{Derivation of
  right-hand side in \eqref{eq:opt-R-ls}}  
\label{app:opt-R-ls}

Recall the least squares regression estimator on \smash{$(\tX_d, \ty)$} is given
by  
\[
\hbeta\ls = (\tX_d^\top \tX_d)^{-1} \tX_d^\top \ty.
\]
The predicted values at the design points are
\[
\hf\ls(\tX_d) = \tX_d \hspace{1pt} \hbeta\ls = L_d \hspace{1pt} \ty,  
\]
where \smash{$L_d = \tX_d (\tX_d^\top \tX_d)^{-1} \tX_d^\top \in \RR^{n \times
    n}$} is the smoothing matrix for the least squares estimator. The training
error is thus    
\begin{align*}
\EE \bigg[ \frac{1}{n} \| \ty - L_d \hspace{1pt} \ty \|_2^2 \bigg] 
&= \frac{\sigma^2}{n} \EE[\tr(I - L_d)^2] \\
&= \frac{\sigma^2}{n} \EE[\tr(I - L_d)] \\
&= \sigma^2 (1-d/n).
\end{align*}
In the second equality above, we used the fact that the matrix $I - L_d$ is
idempotent. Now let \smash{$(\tx_0, \ty_0)$} denote a sample which is i.i.d.\ to
the training data \smash{$(\tx_i, \ty_i)$}, $i \in [n]$. Conditional on
\smash{$\tX_d,\tx_0$}, we can decompose the random-X prediction error of the
least squares predictor \smash{$\hf\ls$} into irreducible error squared bias
plus variance, as usual:   
\begin{align*}
\EE \big[\big( \ty_0 - \hf\ls(\tx_0) \big)^2 \,|\, \tX_d,\tx_0 \big] 
&= \sigma^2 + \big( \EE[ \tx_0^\top \hbeta\ls \,|\, \tX_d,\tx_0 ] - \tx_0^\top
\beta \big)^2 + \Var[ \tx_0^\top \hbeta\ls \,|\, \tX_d,\tx_0 ] \\ 
&= \sigma^2 + \big( \tx_0^\top (\tX_d^\top \tX_d)^{-1} \tX_d^\top \tX_d 
\hspace{1pt} \beta - \tx_0^\top \beta \big)^2 + \sigma^2 \tx_0^\top (\tX_d^\top 
\tX_d)^{-1} \tx_0 \\
&= \sigma^2 \big( 1 + \tx_0^\top (\tX_d^\top \tX_d)^{-1} \tx_0 \big).
\end{align*}
where in the second line we used \smash{$\ty = \tX_d \hspace{1pt} \beta + v$},
where \smash{$\EE[v | \tX_d] = 0$} and \smash{$\Cov[v | \tX_d] = \sigma^2
  I$}. Taking an expectation over \smash{$\tx_0$}, which is independent of
\smash{$\tX_d$}, gives  
\[
\EE \big[\big( \ty_0 - \hf\ls(\tx_0) \big)^2 \,|\, \tX_d \big] 
= \sigma^2 \big(1 + \tr[ \Sigma (\tX_d^\top \tX_d)^{-1} ] \big).
\]
Finally, taking an expectation over \smash{$\tX_d$}, and using the fact that
\smash{$(\tX_d^\top \tX_d)^{-1} \sim W^{-1}(\Sigma^{-1}, n)$} (inverse Wishart
distributed), 
\begin{align*}
\EE \big[\big( \ty_0 - \hf\ls(\tx_0) \big)^2 \big] 
&= \sigma^2 \bigg( 1 + \tr \bigg[ \Sigma \frac{\Sigma^{-1}}{n-d-1} \bigg] 
\bigg) \\ 
& = \sigma^2 \bigg( 1 + \frac{d}{n-d-1} \bigg).
\end{align*}
The random-X optimism is therefore given by
\begin{align*}
\opt\r(\hf\ls) 
&= \sigma^2 \bigg( 1 + \frac{d}{n-d-1} \bigg) - \sigma^2 \bigg( 1 - \frac{d}{n}  
\bigg) \\ 
&= \sigma^2 \bigg( \frac{d}{n} + \frac{d}{n-d-1} \bigg),
\end{align*}
which completes the derivation.

\subsection[Proof of approximation result for omega function]{Proof of
  approximation result in \eqref{eq:omega_approx}}    
\label{app:omega_approx}

Let $z = d/n$, and rewrite 
\[
x = \frac{d}{n} + \frac{d}{n-d-1}
\]
as 
\[
x = z + \frac{z}{1-z-1/n} \iff z^2 - (x+1-1/n) z + (1-1/n) x = 0. 
\]
Note that we can write $\omega_n(x) / n$ as a solution of the above quadratic
equation in $z$,   
\[
\omega_n(x) / n = \frac{b_n - \sqrt{b_n^2 - 4c_n}}{2},
\]
where we define
\[
b_n = x+1-1/n \quad \text{and} \quad c_n =  1-1/n. 
\]
Meanwhile, one can check that $\omega(x)$ solves  
\[
x = z + \frac{z}{1-z} \iff z^2 - (x+1) z + x = 0,
\]
and indeed we can write 
\[
\omega(x) = \frac{b - \sqrt{b^2 - 4c}}{2},
\]
where $b = x+1$ and $c = 1$. The desired fact \eqref{eq:omega_approx} therefore 
follows using $b_n \to b$, $c_n \to c$, and using continuity. 

\section{Proofs for \cref{sec:theory}}

\subsection{Proof of \cref{prop:ridge-monotonicity}}
\label{app:ridge-monotonicity}

Due to the monotonicity of $\omega_n$ in \eqref{eq:omega_n}, it suffices to show
that the intrinsic random-X optimism \smash{$\opt\r\i(\hf_\lambda^\ridge)$} is
decreasing in  $\lambda$. From \eqref{eq:opt-R-smoother-intrinsic}, recall, this  
is
\[
\opt\r\i(\hf_\lambda^\ridge) = \sigma^2 \EE \bigg[ \frac{2}{n} \tr[L_X(X)] + \EE 
[L_X(x_0)^\top L_X(x_0) \, | \, X] - \frac{1}{n} \tr[L_X(X)^\top
L_X(X)] \bigg],
\]
where recall for ridge, we have \smash{$L_X(x) = X (\hSigma + \lambda I)^{-1}
  x$}, with \smash{$\hSigma = X^\top X /  n$}. From Proposition 2 of
\citet{rosset_tibshirani_2020}, we know that the middle term (which is $V + V^+$
in their notation) is decreasing in $\lambda$. For the first and last term,
writing $s_i \geq 0$, $i \in [p]$ for the eigenvalues of \smash{$\hSigma$},
observe
\begin{align*}
\frac{2}{n} \tr[L_X(X)] - \frac{1}{n} \tr[L_X(X)^\top L_X(X)] 
&= \sum_{i=1}^p \bigg( \frac{2s_i}{s_i+\lambda} - \frac{s_i^2}{(s_i+\lambda)^2}
  \bigg) \\
&= \sum_{i=1}^p \frac{2s_i^2 + 2\lambda s_i - s_i^2}{(s_i+\lambda)^2} \\
&= \sum_{i=1}^p \bigg( 1 - \frac{\lambda^2}{(s_i+\lambda)^2} \bigg). 
\end{align*}
Each summand here is decreasing in $\lambda$, which means that their sum is, and 
hence this remains true after taking an expectation with respect to $X$. This
completes the proof. 

\subsection{Proof of \cref{thm:ridge-asymptotics}}
\label{app:ridge-asymptotics}

Throughout the proof, we will use the language of asymptotic equivalents. For
sequences $\{ A_p \}_{p \ge 1}$ and $\{ B_p \}_{p \ge 1}$ of (random or
deterministic) matrices of growing dimension, we say that $A_p$ and $B_p$ are
asymptotically equivalent, and write this as $A_p \asympequi B_p$, provided
$\lim_{p \to \infty} | \tr[C_p (A_p - B_p)] | = 0$ almost surely for any
sequence $\{ C_p \}_{p \ge 1}$ of matrices with bounded trace norm, $\limsup \|
C_p \|_{\mathrm{tr}} < \infty$ as $p \to \infty$. The notion of asymptotic
equivalence satisfies various calculus rules that we will use in our proofs. We
refer readers to Lemma E.3 of \citet{patil2023generalized} for a list of these
rules.  

We collect below three equivalences that we will use in the proofs. These are
standard and we refer readers to Section S.6.5 of \citet{patil2022mitigating}
for more details. 

\begin{lemma}
  \label{lem:ridge-resolvents-asympequi}
    Under \cref{asm:ridge}.\ref{asm:ridge-features}, as $n,p \to \infty$ with $0
    < \liminf_{n \to \infty} \gamma_n \le \limsup_{n \to \infty} \gamma_n <
    \infty$, the following asymptotic equivalences hold for any $\lambda > 0$:    
    \begin{enumerate}
        \item 
        First-order basic equivalence:
        \begin{equation}
            \lambda (\hSigma + \lambda I)^{-1} 
            \asympequi 
            (v(\lambda; \gamma_n) \Sigma + I)^{-1}, 
            \label{eq:basic-asympequi}
        \end{equation}
        where $v_n = v(\lambda; \gamma_n) > 0$ is the unique solution to the
        fixed 
        point equation:   
        \begin{equation}
        \label{eq:ridge-fixed-point-v}
        v_n^{-1} 
        = \lambda + \gamma_n \otr[\Sigma (v_n \Sigma + I)^{-1}], 
        \end{equation}
        \item 
        Second-order variance-type equivalence:
        \begin{equation}
        (\hSigma + \lambda I)^{-1} \hSigma (\hSigma + \lambda I)^{-1}
        \asympequi
        \tv_v(\lambda; \gamma_n)
        (v(\lambda; \gamma_n) \Sigma + I)^{-1} \Sigma (v(\lambda; \gamma_n)
        \Sigma + I)^{-1}, 
        \label{eq:var-asympequi}
        \end{equation}
        where $\tv_v(\lambda; \gamma_n)$ is defined through $v(\lambda;
        \gamma_n)$ via the following equation:   
        \[
            \tv_v(\lambda; \gamma_n)
            = \frac{1}{v(\lambda; \gamma_n)^{-2} - \gamma_n \otr[\Sigma^2
              (v(\lambda; \gamma_n) \Sigma + I)^{-2}]}.
        \]
        \item
        Second-order bias-type equivalence:
        \begin{equation}
            \lambda^2
            (\hSigma + \lambda I)^{-1} A (\hSigma + \lambda I)^{-1}
            \asympequi 
            (v(\lambda; \gamma_n) \Sigma + I)^{-1}
            (\tv_b(\lambda; \gamma_n,A) \Sigma + A) 
            (v(\lambda; \gamma_n) \Sigma + I)^{-1}, 
        \label{eq:bias-asympequi}
        \end{equation}
        for any matrix $A\in\RR^{p\times p}$ with bounded operator norm which is
        independent of $\hSigma$, and where $\tv_b(\lambda; \gamma_n, A)$ is
        defined through $v(\lambda; \gamma_n)$ by the following equation:    
        \[
            \tv_b(\lambda; \gamma_n,A)
            =
            \ddfrac{\gamma_n \otr[A \Sigma(v(\lambda; \gamma_n)\Sigma +
              I)^{-2}]}{v(\lambda; \gamma_n)^{-2} - \gamma_n \otr[\Sigma^2
              (v(\lambda; \gamma_n) \Sigma + I)^{-2}]}. 
          \]
    \end{enumerate}
\end{lemma}
 
With this background, we are now ready to derive the asymptotic equivalents for
the fixed-X and random-X degrees of freedom of the ridge predictor below. Note
that $v_n$ in \eqref{eq:ridge-fixed-point-v} is the reciprocal of $\mu_n$ in
\eqref{eq:ridge-fixed-point-mu}.    

\paragraph{Fixed-X degrees of freedom.}

Recall that the fixed-X degrees of freedom of ridge regression is:
\begin{align*}
    \df\f(\hf_\lambda^\ridge) / n
    &= \tr[L_X(X)] / n \\
    &= \gamma_n \otr[\hSigma (\hSigma + \lambda I)^{-1}] \\
    &= \gamma_n \otr[I - \lambda (\hSigma + \lambda I)^{-1}] \\
    &= \gamma_n - \gamma_n \lambda \otr[(\hSigma + \lambda I)^{-1}].
\end{align*}
Thus, using \eqref{eq:basic-asympequi}, we have the following asymptotic
equivalence: 
\begin{equation}
    \label{eq:df-F-asympequi-deriv}
    \df\f(\hf_\lambda^\ridge) / n
    \asympequi \gamma_n - \gamma_n \otr[(v_n \Sigma + I)^{-1}]
    = \gamma_n \otr[v_n \Sigma (v_n \Sigma + I)^{-1}].
\end{equation}
Now, multiplying the fixed point equation \eqref{eq:ridge-fixed-point-v} by
$v_n$, note that that the final expression in \eqref{eq:df-F-asympequi-deriv}
is simply $1 - \lambda v_n$. In addition, substituting $\mu_n = v_n^{-1}$
yields the final expression in \eqref{eq:ridge-df-F-asympequi}, as desired.

\paragraph{Intrinsic random-X degrees of freedom.}

Recall from \eqref{eq:opt-R-smoother-intrinsic} that the intrinsic random-X
optimism of ridge regression is: 
\begin{equation}
    \opt\r\i(\hf_\lambda^\ridge \,|\, X) / \sigma^2 = 2 \tr[L_{X}(X)]/n +
    \EE_{x_0}[L_{X}(x_0)^\top L_{X}(x_0)] - \tr[\bL_{X}(X)^\top \bL_{X}(X)]  /
    n. \label{eq:proof:prop:ridge-eq-1} 
\end{equation}
We now rewrite the three terms in \eqref{eq:proof:prop:ridge-eq-1} to make them 
amenable for applications of asymptotic equivalents described in the background
above. 

On one hand, note that:
\begin{align}
    &2 \tr[L_{X}(X)] / n - \tr[L_{X}(X)^\top L_{X}(X)]  / n \notag \\
    &= - \tr[(I - L_X(X))^2] / n + 1 \notag \\
    &= - \tr[(I - L_X(X))] / n + \tr[L_X(X) (I - L_X(X))] / n + 1 \notag \\
    &= -1 + \tr[L_X(X)] / n + \tr[L_X(X) (I - L_X(X))] / n + 1  \notag \\
    &= -1 + \gamma_n - \lambda \tr[(\hSigma + I)^{-1}] / n + \lambda \tr[\hSigma
      (\hSigma + \lambda I)^{-2}] / n + 1 \notag \\ 
    &= \gamma_n - \lambda \tr[(\hSigma + I)^{-1}] / n + \lambda \tr[\hSigma
      (\hSigma + \lambda I)^{-2}] / n. \label{eq:intrinsic-opt-train-error} 
\end{align}
On the other hand, note that:
\begin{align}
    \EE[\tr[L_{X}(x_0)^\top L_{X}(x_0)]]
    &= \tr[\Sigma \hSigma (\hSigma + \lambda I)^{-2}] / n \notag \\
    &= \tr[\Sigma (\hSigma + \lambda I)^{-1}] / n - \lambda \tr[\Sigma (\hSigma
      + \lambda)^{-2}] / n. \label{eq:intrinsic-opt-test-error} 
\end{align}
Substituting \eqref{eq:intrinsic-opt-train-error},
\eqref{eq:intrinsic-opt-test-error} into \eqref{eq:proof:prop:ridge-eq-1}, our 
goal is reduced to obtaining an asymptotic equivalent for: 
\begin{align}
    \opt\r\i (\hf_\lambda^\ridge)/ \sigma^2
    &= \tr[\Sigma (\hSigma + \lambda I)^{-1}] / n - \lambda \tr[\Sigma (\hSigma
      + \lambda)^{-2}] / n \notag \\ 
    &\quad + \gamma_n - \lambda \tr[(\hSigma + I)^{-1}] / n + \lambda
      \tr[\hSigma (\hSigma + \lambda I)^{-2}] / n \notag \\ 
    &= (\tr[\Sigma (\hSigma + \lambda I)^{-1}] / n + 1) - (1 - \gamma_n +
      \lambda \tr[(\hSigma + \lambda I)^{-1}] / n) \notag \\ 
    &\quad - (\lambda \tr[\Sigma (\hSigma + \lambda)^{-2}] / n - \lambda
      \tr[\hSigma (\hSigma + \lambda I)^{-2}] / n). \label{eq:OptR-reduced} 
\end{align}

Now observe that for the first line in \eqref{eq:OptR-reduced}:
\begin{align}
    1 - \gamma_n + \lambda \tr[(\hSigma + \lambda I)^{-1}] / n
    &\asympequi 1 - \gamma_n + \gamma_n \otr[(v_n \Sigma + I)^{-1}] \notag \\ 
    &= 1 - \gamma_n + \lambda v_n + \gamma_n - 1 \notag \\
    &= \lambda v_n (1 - \lambda v_n) + \lambda^2 v_n^2 \notag \\
    &= \lambda v_n^2 \gamma_n \otr[\Sigma (v_n \Sigma + I)^{-1}]
    + \lambda^2 v_n^2 \notag \\
    &\asympequi \lambda^2 v_n^2 (\tr[\Sigma (\hSigma + \lambda I)^{-1}] / n +
      1). 
    \label{eq:OptR-reduced-term1-matching}
\end{align}

Similarly, observe that for the second line in \eqref{eq:OptR-reduced}:
\begin{align}
    \tr[\hSigma (\hSigma + \lambda I)^{-2}] / n
    &= \gamma_n \otr[\hSigma (\hSigma + \lambda I)^{-2}] \notag \\
    &\asympequi
    \frac{1}{v_n^{-2} - \gamma_n \otr[\Sigma^2 (v_n \Sigma + I)^{-2}]} \cdot
      \gamma_n \otr[\Sigma (v_n \Sigma + I)^{-2}] \notag \\ 
    &= 
    v_n^2 \bigg( \frac{\gamma_n \otr[\Sigma^2 (v_n \Sigma + I)^{-2}]}{v_n^{-2} -
      \gamma_n \otr[\Sigma^2 (v_n \Sigma + I)^{-2}]} + 1 \bigg) \cdot \gamma_n
      \otr[\Sigma (v_n \Sigma + I)^{-2}] \notag \\ 
    &\asympequi
    \lambda^2 v_n^2 \tr[\Sigma (\hSigma + \lambda I)^{-2}] / n.
    \label{eq:OptR-reduced-term2-matching}
\end{align}

Hence, substituting \eqref{eq:OptR-reduced-term1-matching} and
\eqref{eq:OptR-reduced-term2-matching} into \eqref{eq:OptR-reduced}, we have 
\begin{align*}
    \opt\r\i(\hf_\lambda^\ridge \,|\, X)/ \sigma^2
    &\asympequi
    (1 - \lambda^2 v_n^2) (\tr[\Sigma (\hSigma + \lambda I)^{-1}] / n + 1 -
      \lambda \tr[\Sigma (\hSigma + \lambda)^{-2}] / n) \\ 
    &= (1 - \lambda^2 v_n^2) (\tr[\Sigma \hSigma (\hSigma + \lambda I)^{-2}] / n
      + 1) \\ 
    &= (1 - \lambda^2 v_n^2) \bigg( \frac{1}{v_n^{-2} - \gamma_n \otr[\Sigma^2
      (v_n \Sigma + I)^{-2}]} \gamma_n \otr[\Sigma^2 (v_n \Sigma + I)^{-2}] + 1
      \bigg) \\ 
    &\asympequi
    (1 - \lambda^2 v_n^2) \bigg(\frac{\gamma_n \otr[\Sigma^2 (\Sigma + \mu_n
      I)^{-2}]}{1 - \gamma_n \otr[\Sigma^2 (\Sigma + \mu_n I)^{-2}]} + 1 \bigg)
  \\ 
    &= (1 - \lambda^2 v_n^2) \Big( \frac{V_n}{D_n} + 1 \Big), 
\end{align*}
where in the second-to-last step above, we used $\mu_n = v_n^{-1}$ to simplify
the expressions. 
Now applying the mapping $\omega$ to bring on the degrees of
freedom scale, we have that
\begin{equation}
    \label{eq:omega-of-condopt-ridge}
    \omega(\opt\r\i (\hf_\lambda^\ridge \,|\, X)/ \sigma^2)
    \asympequi
    \omega((1 - \lambda^2 / \mu^2) (V_n / D_n + 1)).
\end{equation}
Note that the right-hand side of \eqref{eq:omega-of-condopt-ridge} is always
bounded by $1$ (by construction), thus we can apply the dominated convergence
theorem to conclude that same asymptotic equivalence
\eqref{eq:omega-of-condopt-ridge} holds after we take an expectation with
respect to $X$. This yields the result in \eqref{eq:ridge-df-R-i-asympequi}.   

\paragraph{Emergent random-X degrees of freedom.}

Using \eqref{eq:opt-R-smoother-emergent-intrinsic}, the only additional quantity
we need to deal with is the excess bias, whose asymptotic equivalent we will
derive next. 

Recalling \eqref{eq:f-decomposition}, let us abbreviate $f\li(x) = x^\top 
\beta$ and thus $f(x) = f\li(x) + f\nl(x)$. We can decompose the excess bias
\smash{$B^+ = B^+(\hf^\ridge_\lambda)$} into linear, nonlinear, and cross
components as follows:   
\begin{align}
    B^+
    &= \EE_{x_0}[ (f(x_0) - L_{X}(x_0)^\top f(X))^2 ]  
    - \| (I - L_{X}(X)) f(X) \|_2^2 / n \notag \\
    &= \EE_{x_0}[ (f\li(x_0) + f\nl(x_0) - L_{X}(x_0)^\top (f\li(X) + 
      f\nl(X)))^2 ] \notag \\
    &\quad - \| (I - L_{X}(X)) (f\li(X) + f\nl(X)) \|_2^2 / n \notag \\
    &= B^+\li + B^+\nl + C^+, \label{eq:Bplus-decomposition-ridge}
\end{align}
where $B^+\li$, $B^+\nl$, and $C^+$ are defined as:
\begin{align*}
    B^+\li
    &= \EE_{x_0}[(f\li(x_0) - L_X(x_0)^\top f\li(X))^2] - \|(I - L_X(X)) f\li(X)
      \|_2^2/n, \\ 
    B^+\nl
    &= \EE_{x_0}[(f\nl(x_0) - L_X(x_0)^\top f\nl(X))^2] - \|(I - L_X(X)) f\nl(X)
      \|_2^2/n, \\ 
    C^+
    &= 2 \EE_{x_0}[(f\li(x_0) - L_X(x_0)^\top f\li(X)) (f\nl(x_0) -
      L_X(x_0)^\top f\nl(X))] \notag \\ 
    &\quad - 2 f\li(X)^\top (I - L_X(X))^2 f\nl(X). 
\end{align*}
We will obtain the asymptotic equivalents for $B^+\li$, $B^+\nl$, and $C^+$
separately below. 

\emph{Asymptotic equivalent for $C^+$.} For the first term in $C^+$, we have: 
\begin{align}
    &\EE_{x_0}[(f\li(x_0) - L_X(x_0)^\top f\li(X)) (f\nl(x_0) - L_X(x_0)^\top
      f\nl(X))] \notag \\ 
    &= - \EE_{x_0}[f\li(x_0) L_X(x_0)^\top f\nl(X)] + \EE_{x_0}[f\li(X)^\top
      L_X(x_0) L_X(x_0)^\top f\nl(X)]. \notag 
\end{align}
Here we used the fact that $\EE_{x_0}[f\nl(x_0) f\li(x_0)] = 0$ and
$\EE_{x_0}[f\nl(x_0) L_X(x_0)^\top f\li(X)] = 0$ because $\EE_{x_0}[f\nl(x_0)
x_0] = 0$. For the remaining two terms in $C^+$, observe that: 
\begin{align}
    \EE_{x_0}[f\li(x_0) L_X(x_0)^\top f\nl(X)]
    = \beta^\top \Sigma (\hSigma + \lambda I)^{-1} X^\top f\nl(X) /
  n, \label{eq:cross-term-1} \\ 
    \EE_{x_0}[f\li(X)^\top L_X(x_0) L_X(x_0)^\top f\nl(X)] 
    = \beta^\top X^\top / n (\hSigma + \lambda I)^{-1} \Sigma (\hSigma + \lambda
  I)^{-1} X^\top f\nl(X) / n. \label{eq:cross-term-2} 
\end{align}
Invoking Lemma A.3 of \citet{patil2023generalized}, we conclude that the
right-hand sides of both \eqref{eq:cross-term-1} and \eqref{eq:cross-term-2}
almost surely vanish. In a similar way, we can show that the second term of
$C^+$ vanishes almost surely. Thus, we have $C^+ \asympequi 0$.

\emph{Asymptotic equivalent for $B^+\li$.} For the linear component of excess
bias, we have 
\begin{align}
    B^+\li 
    &= \EE_{x_0}[ (f\li(x_0) - L_X(x_0)^\top f\li(X))^2] - \|(I - L_X(X))
      f\li(X) \|_2^2/n \notag \\  
    &= \EE_{x_0}[(\beta^{\top} x_0 - \beta^\top X^\top L_X(x_0))^2] - \|(I -
      L_X(X)) X\beta\|_2^2/n \notag \\ 
    &= \EE_{x_0}[(\beta^{\top} (I - \hSigma (\hSigma + \lambda I)^{-1}) x_0)^2]
      - \|(I - X (\hSigma + \lambda I)^{-1}X^{\top}/ n) X\beta\|_2^2/n \notag\\ 
    &= \lambda^2 \beta^{\top}  (\hSigma + \lambda I)^{-1} \Sigma (\hSigma +
      \lambda I)^{-1} \beta - \|X(I - (\hSigma + \lambda I)^{-1}\hSigma)
      \beta\|_2^2/n \notag\\ 
    &= \lambda^2 \beta^{\top}  (\hSigma + \lambda I)^{-1} \Sigma (\hSigma +
      \lambda I)^{-1} \beta  - \lambda^2\beta^{\top} (\hSigma + \lambda I)^{-1}
      \hSigma (\hSigma + \lambda I)^{-1}\beta. \label{eq:B+-ridge} 
\end{align}
From \eqref{eq:var-asympequi} and \eqref{eq:bias-asympequi}, we have the
following equivalence: 
\begin{align*}
    B^+\li / \sigma^2
    &\asympequi
    (1+\tv_b(\lambda; \gamma_n,\Sigma)) 
    \beta^\top
    (v_n \Sigma + I)^{-1}
     \Sigma
    (v_n \Sigma + I)^{-1}
    \beta \\
    &\quad - 
    \lambda^2 \tv_v(\lambda; \gamma_n)
    \beta^\top
    (v(\lambda; \gamma_n) \Sigma + I)^{-1} \Sigma 
      (v(\lambda; \gamma_n) \Sigma + I)^{-1} 
    \beta \\
    &\asympequi 
   (1+\tv_b(\lambda; \gamma_n,\Sigma)) 
   \beta^\top 
   (v_n \Sigma + I)^{-1}
     \Sigma
    (v_n \Sigma + I)^{-1}
    \beta \\
   &\quad -
   \lambda^2 v_n^2 (1+\tv_b(\lambda; \gamma_n,\Sigma)) 
   \beta^\top
   (v_n \Sigma + I)^{-1}
     \Sigma
    (v_n \Sigma + I)^{-1}
    \beta
   \\
    &\asympequi (1 - \lambda^2v_n^2)(1+\tv_b(\lambda; \gamma_n,\Sigma))
      \beta^{\top}(v_n \Sigma + I)^{-1} \Sigma
    (v_n \Sigma + I)^{-1}\beta / \sigma^2 \\
    &= \frac{(1 - \lambda^2v_n^2) \beta^{\top}(v_n \Sigma + I)^{-1}
     \Sigma
    (v_n \Sigma + I)^{-1}\beta / \sigma^2}{v_n^{-2}- \gamma_n \otr[\Sigma^2 (v_n
      \Sigma + I)^{-2}]} \\ 
    &= \frac{(1 - \lambda^2 v_n^2) \mu^2 \beta^{\top} (\Sigma + \mu_n I)^{-1}
      \Sigma (\Sigma + \mu_n I)^{-1} \beta / \sigma^2}{1 - \gamma_n
      \otr[\Sigma^2 (\Sigma + \mu_n I)^{-2}]} \\ 
    &= (1 - \lambda^2 v_n^2) \frac{B_n}{D_n},
\end{align*}
where we again used the parameterization $\mu_n = v_n^{-1}$ to simplify the
expression. 

\emph{Asymptotic equivalent for $B^+\nl$.} For the nonlinear component of excess
bias, we have  
\begin{align}
    B^+\nl
    &=
    \EE_{x_0}[(f\nl(x_0) - L_X(x_0)^\top f\nl(X))^2] - \|(I - L_X(X)) f\nl(X)
      \|_2^2/n \notag \\ 
    &=
    \EE_{x_0}[f\nl(X)^\top L_X(x_0) L_X(x_0)^\top f\nl(X)] + \sigma^2\nl - \|(I
      - L_X(X)) f\nl(X) \|_2^2/n \notag \\ 
    &=
    f\nl(X)^\top 
    (X (\hSigma + \lambda I)^{-1} \Sigma (\hSigma + \lambda I)^{-1} X^\top / n)  
    f\nl(X) / n
    + \sigma^2\nl \notag \\
    &\quad - 
    f\nl(X)^\top
    (I - X (\hSigma + \lambda I)^{-1} X^\top / n)^2
    f\nl / n \notag \\
    &= 
    f\nl(X)^\top 
    ((\hSigma + \lambda I)^{-1} X^\top / n)^\top 
    \Sigma 
    ((\hSigma + \lambda I)^{-1} X^\top / n) 
    f\nl(X) 
    + \sigma^2\nl
    \label{eq:fnl-finite-term-1} \\
    &\quad - 
    f\nl(X)^\top (X (\hSigma + \lambda I)^{-1} X^\top / n - I)^\top (X (\hSigma
      + \lambda I)^{-1} X^\top / n - I) f\nl(X) /
      n. \label{eq:fnl-finite-term-2} 
\end{align}
In the second equality above, we used the facts that $\EE_{x_0}[f\nl(x_0) x_0] =
0$, and $\EE_{x_0}[f\nl(x_0)^2] = \sigma^2\nl$. We will now use Part (2) of Lemma
A.2 of \citet{patil2023generalized} get asymptotic equivalents for the first
term in \eqref{eq:fnl-finite-term-1} and \eqref{eq:fnl-finite-term-2}. We have 
\[
    f\nl(X)^\top 
    ((\hSigma + \lambda I)^{-1} X^\top / n)^\top 
    \Sigma 
    ((\hSigma + \lambda I)^{-1} X^\top / n)
    f\nl(X)
    \asympequi \frac{V_n}{D_n} \sigma^2\nl
\]
and
\[
    f\nl(X)^\top (X (\hSigma + \lambda I)^{-1} X^\top / n - I)^\top (X (\hSigma
    + \lambda I)^{-1} X^\top / n - I) f\nl(X) / n \asympequi 
    \lambda^2 v_n^2
    \Big(\frac{V_n}{D_n} \sigma^2\nl + \sigma^2\nl\Big).
\]
Thus, we obtain the following asymptotic equivalent for $B^+\nl$:
\[
    B^+\nl / \sigma^2
    \asympequi 
    (1 - \lambda^2 v_n^2) \Big(\frac{V_n}{D_n} + 1\Big)
    \frac{\sigma^2\nl}{\sigma^2}. 
\]
And hence, we have the overall asymptotic equivalent for $B^+$: 
\[
    B^+ / \sigma^2
    \asympequi (1 - \lambda^2 v_n^2) \bigg( \frac{B_n}{D_n} +
    \Big(\frac{V_n}{D_n} + 1 \Big) \frac{\sigma^2\nl}{\sigma^2} \bigg). 
\]
Combining this with the calculation above for the intrinsic random-X optimism
then applying the mapping $\omega$, followed by the dominated convergence
theorem to convert this to an expectation over $X$, yields the desired
equivalent in \eqref{eq:ridge-df-R-asympequi} and finishes the proof.  

\subsection{Proof of \cref{thm:ridgeless-asymptotics}}
\label{app:ridgeless-asymptotics}

The proof is similar to that in \cref{app:ridge-asymptotics}. We will make use
of various asymptotic equivalences for ridgeless regression from Section S.6.5
of \citet{patil2022mitigating}, collected in the lemma below.%

\begin{lemma}
    \label{lem:deter-approx-generalized-ridgeless}
     Under \cref{asm:ridge}.\ref{asm:ridge-features}, as $n,p \to \infty$ with
     $0 < \liminf_{n \to \infty} \gamma_n \le \limsup_{n \to \infty} \gamma_n < 
    \infty$, the following asymptotic equivalences hold:
    \begin{enumerate}
        \item
        First-order basic equivalence:
        \begin{equation}
            \label{eq:basic-asympequi-ridgeless}
            I - \hSigma \hSigma^\pinv
            \asympequi
            \begin{cases}
                0 & \gamma_n \le 1 \\
                (v(0; \gamma_n) \Sigma + I)^{-1} & \gamma_n > 1, 
            \end{cases}
        \end{equation}
        where $v_n =  v(0; \gamma_n) > 0$ is the unique solution to the
        fixed-point equation: 
        \begin{equation}
            \label{eq:ridgeless-fixed-point-v}
            \frac{1}{v_n}
            = \gamma_n \otr[\Sigma (v_n \Sigma + I)^{-1}]. 
        \end{equation}
        \item 
        Second-order variance-type equivalence:
        \begin{equation}
            \label{eq:detequi-mn2ls-genvar}
            \hSigma^{\pinv} \hSigma \hSigma^{\pinv}
            \asympequi
            \begin{dcases}
                \frac{\Sigma^{-1}}{1 - \gamma_n} & \gamma_n \le 1 \\
                \tv_v(0; \gamma_n)
                (v(0; \gamma_n) \Sigma + I_p)^{-1} 
                \Sigma
                (v(0; \gamma_n) \Sigma + I_p)^{-1} & \gamma_n > 1,
            \end{dcases}
        \end{equation}
        where $\tv_v(0; \gamma)$ is defined through $v(0; \gamma)$  via  
        \[
            \tv_v(0; \gamma)
            = 
            \frac{1}{v(0; \gamma)^{-2} - \gamma \otr[\Sigma^2 (v(0; \gamma)
              \Sigma + I_p)^{-2}]}. 
        \]
        \item 
        Second-order bias-type equivalence:
        \begin{align}
            &(I_p - \hSigma^{+} \hSigma)
            \Sigma
            (I_p - \hSigma^{+} \hSigma) \notag \\
            &\asympequi
            \begin{cases}
                0 & \gamma_n \le 1 \\
                (1 + \tv_b(0; \gamma_n))
                (v(0; \gamma_n) \Sigma + I_p)^{-1}
                \Sigma
                (v(0; \gamma_n) \Sigma + I_p)^{-1} & \gamma_n > 1,
            \end{cases}
            \label{eq:detequi-mn2ls-genbias}
        \end{align}
        where $v(0; \gamma)$ is as defined in
        \eqref{eq:ridgeless-fixed-point-v}, and $\tv_b(0; \gamma)$ is defined
        via $v(0; \gamma)$ by 
        \[
          \tv_b(0; \gamma)
          =  \frac{\gamma \otr[\Sigma^2 (v(0; \gamma) \Sigma + I_p)^{-2}]}{v(0;  
            \gamma)^{-2} - \gamma \otr[\Sigma^2 (v(0; \gamma) \Sigma +
            I_p)^{-2}]}.  
        \]
      \end{enumerate}
    \end{lemma}

Note: in \eqref{eq:basic-asympequi-ridgeless} and
\eqref{eq:detequi-mn2ls-genbias} above we use $0$ to denote the all-zero 
matrix in $\RR^{p \times p}$. Also, $v_n$ in \eqref{eq:ridgeless-fixed-point-v}
is the reciprocal of $\mu_n$ in \eqref{eq:ridgeless-fixed-point-mu}.       

We are now ready to obtain the asymptotic equivalents for the fixed-X and
random-X degrees of freedom of the ridgeless predictor below. 

\paragraph{Fixed-X degrees of freedom.}

Note that the smoother matrix for ridgeless regression can be written as $L_X(X)
= X \hSigma^{\pinv} X^\top / n$. The fixed-X degrees of freedom is thus: 
\[
    \df\f(\hf_0^\ridge) / n
    = \tr[L_X(X)] / n
    = \tr[\hSigma \hSigma^{\pinv}] / n.
\]
Now using \eqref{eq:basic-asympequi-ridgeless}, we have 
\[
    \df\f(\hf_0^\ridge) / n
    \asympequi 
    \begin{cases}
        \gamma_n \otr[I] = \gamma_n & \gamma_n < 1 \\
        \gamma_n (1 - \otr[(v_ \Sigma + I)^{-1}])
        = \gamma_n \gamma_n \otr[v_n \Sigma (v_n \Sigma + I)^{-1}] 
        = 1 & \gamma_n > 1,
    \end{cases}
\]
as desired.

\paragraph{Intrinsic random-X degrees of freedom.}

Since $L_{X}(x_0) = x_0^\top \hSigma^{\pinv} X^\top / n$ for ridgeless
regression, the intrinsic random-X optimism of ridgeless regression is: 
\begin{align}
    \opt\r\i(\hf_0^\ridge \,|\, X) / \sigma^2
    &= 2 \tr[\hSigma \hSigma^{\pinv}]/n -  \tr[(\hSigma \hSigma ^{\pinv})^2]/n +
      \tr[\Sigma \hSigma ^{\pinv} \hSigma \hSigma^{\pinv}] / n \notag \\ 
    &= \tr[\hSigma \hSigma^{\pinv}] / n + \tr[\Sigma \hSigma ^{\pinv} \hSigma
      \hSigma^{\pinv}] / n, \label{eq:opt-intrinsic-ridgeless} 
\end{align}
where we used the fact that $\hSigma^{\pinv} \hSigma \hSigma^{\pinv} =
\hSigma^{\pinv}$ (a property of the Moore-Penrose pseudoinverse). We now
use \eqref{eq:basic-asympequi-ridgeless} and \eqref{eq:detequi-mn2ls-genvar} to 
obtain the asymptotic equivalent for \eqref{eq:opt-intrinsic-ridgeless}. We will
do the underparameterized and overparameterized cases separately below.  

\emph{Underparameterized regime.} We have
\[
    \opt\r\i(\hf_0^\ridge \,|\, X) / \sigma^2
    \asympequi
    2 \gamma_n - \gamma_n + \frac{\gamma_n \otr[\Sigma \Sigma^{-1}]}{1 -
      \gamma_n} =
    \gamma_n + \frac{\gamma_n}{1 - \gamma_n}.
\]
Applying the mapping $\omega$, followed by the dominated convergence theorem,
yields the result.

\emph{Overparameterized regime.} We have
\begin{align*}
    \opt\r\i(\hf_0^\ridge \,|\, X) / \sigma^2
    &\asympequi
    \gamma_n (1 - \otr[(v_n \Sigma + I)^{-1}]) 
    + \gamma_n \frac{1}{v_n^{-2} - \gamma_n \otr[\Sigma^2 (v_n \Sigma +
      I_p)^{-2}]} \otr[\Sigma^2 (v_n \Sigma + I)^{-2}] \\ 
    &= \gamma_n - \gamma_n \mu_n \otr[(\Sigma + \mu_n I)^{-1}]
    + \frac{\gamma_n \otr[\Sigma^2 (\Sigma + \mu_n I)^{-2}]}{1 - \gamma_n
      \otr[\Sigma^2 (\Sigma + \mu_n I)^{-2}]} \\ 
    &= \gamma_n \otr[\Sigma (\Sigma + \mu_n I)^{-1}] + \frac{\gamma_n
      \otr[\Sigma^2 (\Sigma + \mu_n I)^{-2}]}{1 - \gamma_n \otr[\Sigma^2 (\Sigma 
      + \mu_n I)^{-2}]} \\ 
    &= \frac{\gamma_n \otr[\Sigma^2 (\Sigma + \mu_n I)^{-2}]}{1 - \gamma_n 
      \otr[\Sigma^2 (\Sigma + \mu_n I)^{-2}]} + 1 \\ 
    &= \frac{V_n}{D_n} + 1,
\end{align*}
where we used the parameterization $\mu_n = v_n^{-1}$ to simplify the
expressions on the second line and \eqref{eq:ridgeless-fixed-point-mu} on the
last line. Applying the mapping $\omega$ gives the desired result.

\paragraph{Emergent random-X degrees of freedom.}

As with ridge, we will derive an asymptotic equivalent for the excess bias
\smash{$B^+ = B^+(\hf^\ridge_0)$}, and then use the decomposition
\eqref{eq:opt-R-smoother-emergent-intrinsic} to obtain the final equivalent. 
Let us write $B^+ = B^+\li + B^+\nl + C^+$, as in
\eqref{eq:Bplus-decomposition-ridge} in the ridge proof. By similar arguments,
we have $C^+ \asympequi 0$. It thus suffices to obtain asymptotic equivalents
for $B^+\li$ and $B^+\nl$.     

\emph{Asymptotic equivalent for $B^+\li$.} For the linear component of excess
bias, we have 
\begin{align}
    B^+\li
    &= \EE_{x_0}[ (f(x_0) - L_{X}(x_0)^\top f(X))^2] - \| (I - L_{X}(X)) f(X)
      \|_2^2 / n \notag \\ 
    &= \EE_{x_0}[(\beta^{\top} (I - \hSigma \hSigma^{\pinv}) x_0)^2] - \|(I - X
      \hSigma^{\pinv} X^{\top}/ n) X\beta\|_2^2/n \notag\\ 
    &= \beta^{\top} (I - \hSigma \hSigma^{\pinv}) \Sigma (I - \hSigma
      \hSigma^{\pinv}) \beta - \|X(I - \hSigma \hSigma^{\pinv}) \beta\|_2^2/n
      \notag\\ 
    &= \beta^{\top} (I - \hSigma \hSigma^{\pinv}) \Sigma (I - \hSigma
      \hSigma^{\pinv}) \beta - \beta^{\top} (I - \hSigma \hSigma^{\pinv})
      \hSigma (I - \hSigma \hSigma^{\pinv}) \beta \notag \\ 
    &= \beta^{\top} (I - \hSigma \hSigma^{\pinv}) \Sigma (I - \hSigma
      \hSigma^{\pinv}) \beta, \label{eq:B+-ridgeless} 
\end{align}
where we use the fact that $\hSigma \hSigma^{\pinv} \hSigma = I$. Now using
\eqref{eq:detequi-mn2ls-genbias}, we can obtain the asymptotic equivalent for
\eqref{eq:B+-ridgeless} as follows:  
\[
    B^+\li
    \asympequi
    \begin{dcases}
        0 
        & \gamma_n \le 1 \\
        (1 + \tv_b) (v_n \Sigma + I)^{-1} \Sigma (v_n \Sigma + I)^{-1}
        = \frac{\mu_n^2 (\Sigma + \mu_n I)^{-1} \Sigma (\Sigma + \mu_n
          I)^{-1}}{1 - \gamma_n \otr[\Sigma^2 (\Sigma + \mu_n I)^{-2}]} 
        = \frac{B_n}{D_n} 
        & \gamma_n > 1 ,
    \end{dcases}
\]
where in the last line we use the parameterization $\mu_n = v_n^{-1}$.

\emph{Asymptotic equivalent for $B^+\nl$.} For the nonlinear component of excess
bias, we have 
\begin{align}
    B^+\nl
    &=
    \EE_{x_0}[(f\nl(x_0) - L_X(x_0)^\top f\nl(X))^2] - \|(I - L_X(X)) f\nl(X)
      \|_2^2/n \notag \\ 
    &=
    \EE_{x_0}[f\nl(X)^\top L_X(x_0) L_X(x_0)^\top f\nl(X)] + \sigma^2\nl - \|(I
      - L_X(X)) f\nl(X) \|_2^2/n \notag \\ 
    &=
    f\nl(X)^\top 
    (X \hSigma^{\dagger} \Sigma \hSigma^{\dagger} X^\top / n) 
    f\nl(X) / n
    + \sigma^2\nl
    - 
    f\nl(X)^\top
    (I - X \hSigma^{\dagger} X^\top / n)^2
    f\nl / n \notag \\
    &= 
    f\nl(X)^\top 
    (\hSigma^{\dagger} X^\top / n)^\top 
    \Sigma 
    (\hSigma^{\dagger} X^\top / n) 
    f\nl(X) 
    + \sigma^2\nl
    \label{eq:fnl-finite-term-1-ridgeless} \\
    &\quad - 
    f\nl(X)^\top (X \hSigma^{\dagger} X^\top / n - I)^\top (X \hSigma^{\dagger}
      X^\top / n - I) f\nl(X) / n. \label{eq:fnl-finite-term-2-ridgeless} 
\end{align}
As with ridge regression, in the second equality above, we used the fact that
$\EE_{x_0}[f\nl(x_0) x_0] = 0$ and $\EE_{x_0}[f\nl(x_0)^2] = \sigma^2\nl$. As
shown in the proof of Theorem 1 of \citet{patil2023generalized}, the two
quadratic forms in \eqref{eq:fnl-finite-term-1-ridgeless} and
\eqref{eq:fnl-finite-term-2-ridgeless} concentrate around the traces. Thus, for
\eqref{eq:fnl-finite-term-1-ridgeless}, we have 
\begin{align*}
    f\nl(X)^\top 
    (\hSigma^{\dagger} X^\top / n)^\top 
    \Sigma 
    (\hSigma^{\dagger} X^\top / n)
    f\nl(X) + \sigma^2\nl
    &\asympequi 
    \tr[(\hSigma^{\dagger} X^\top / n)^\top \Sigma (\hSigma^{\dagger} X^\top /
      n)] \sigma^2\nl + \sigma^2\nl \\ 
    &= (\tr[\hSigma^{\dagger} \Sigma \hSigma^{\dagger} \hSigma] / n + 1)
      \sigma^2\nl \\ 
    &= (\tr[\hSigma^{\dagger} \Sigma] / n + 1) \sigma^2\nl, 
\end{align*}
where we used the fact that $\hSigma^{\dagger} \hSigma \hSigma^{\dagger} =
\hSigma^{\dagger}$ in the third line. Similarly, for
\eqref{eq:fnl-finite-term-2-ridgeless}, we have 
\begin{align*}
    f\nl(X)^\top (X \hSigma^{\dagger} X^\top / n - I)^\top (X \hSigma^{\dagger}
  X^\top / n - I) f\nl(X) / n 
    &\asympequi 
    \tr[(X \hSigma^{\dagger} X^\top / n - I)^2] / n \cdot \sigma^2\nl \\
    &= \tr[X \hSigma^{\dagger} X^\top / n - I] / n \cdot \sigma^2\nl \\
    &= (\tr[\hSigma^{\dagger} \hSigma] / n - 1) \sigma^2\nl,
\end{align*}
where we used the fact that $X \hSigma^{\dagger} X^\top / n - I$ is an
idempotent matrix in the second line. Combining the two asymptotic equivalents,
we thus have 
\[
    B^+\nl / \sigma^2
    \asympequi 
    \tr[\hSigma^{\dagger} \Sigma] / n \cdot \frac{\sigma^2\nl}{\sigma^2}
    + \tr[\hSigma^{\dagger} \hSigma] / n \cdot \frac{\sigma^2\nl}{\sigma^2}.
\]
Similar to the intrinsic analysis, we obtain the following asymptotic equivalent
for $B^+\nl$:  
\begin{align*}
    B^+\nl / \sigma^2
    \asympequi
    \begin{dcases}
        \Big( \gamma_n + \frac{\gamma_n}{1 - \gamma_n} \Big)
        \frac{\sigma^2\nl}{\sigma^2} & \gamma_n \le 1 \\ 
        \Big( \frac{V_n}{D_n} + 1 \Big) \frac{\sigma^2\nl}{\sigma^2} & \gamma_n
        > 1. 
    \end{dcases}
\end{align*}
Combining this with the results for the intrinsic random-X optimism, and passing
through $\omega$ and subsequent application of the dominated convergence 
theorem, completes the proof.   

\subsection{Proof of \cref{prop:ridgeless-monotonicity}}
\label{app:ridgeless-monotonicity}

Because $\omega$ is strictly increasing, in order to analyze the monotonicity of 
the asymptotic equivalents for normalized degrees of freedom in $\gamma_n$, it
suffices to analyze the monotonicity of the equivalents for random-X optimism in
$\gamma_n$, respectively. We do this for the intrinsic and emergent cases below.      

\paragraph{Intrinsic random-X optimism.}

There are two regimes to examine.

\emph{Underparameterized regime.}
When $\gamma_n<1$, from the proof of \cref{thm:ridgeless-asymptotics}, we 
have that 
\begin{align*}
    \opt\r\i(\hf_0^\ridge \,|\, X) / \sigma^2 
    \asympequi 
    \gamma_n + \frac{\gamma_n}{1 - \gamma_n},
\end{align*}
which is a strictly increasing function in $\gamma_n\in (0,1)$, with the
following boundary limits: 
\[
    \lim_{\gamma_n \to 0^+} \opt\r\i(\hf_0^\ridge \,|\, X)/\sigma^2 = 0, 
    \quad \text{and} \quad 
    \lim_{\gamma_n \to 1^-} \opt\r\i(\hf_0^\ridge \,|\, X)/\sigma^2 = \infty,  
\]
Consequently, $\df\r\i(\hf_0^\ridge) / n$ is increasing from 0 to 1 in
$\gamma_n\in (0,1)$. 

\emph{Overparameterized regime.} When $\gamma_n>1$, by Lemma F.11 in
\citet{du2023subsample}, the solution $v(0;\gamma_n)$ to the fixed point
equation \eqref{eq:ridgeless-fixed-point-v} is finite. Then, it follows from the
proof of \cref{thm:ridgeless-asymptotics} that 
\[
    \opt\r\i(\hf_0^\ridge \,|\, X) / \sigma^2 \asympequi \tv_v(0;\gamma_n). 
\]
Next, we study the monotonicity of $\tv$. Taking the derivative with respect to
$\gamma_n$ yields 
\begin{align*}
    &\frac{\partial \tv_v(0;\gamma_n)}{\partial\gamma_n} \\ 
    &= \frac{\otr[\Sigma^2 (v(0; \gamma_n) \Sigma + I)^{-2}]}{\big(v(0;
      \gamma_n)^{-2} - \gamma_n\otr[\Sigma^2 (v(0; \gamma_n) \Sigma +
      I)^{-2}]\big)^3}\\ 
    &\qquad \cdot \big[\big(v(0; \gamma_n)^{-2} - \gamma_n\otr[\Sigma^2 (v(0;
      \gamma_n) \Sigma + I)^{-2}]\big)^3 - 2 \gamma_n v(0;\gamma_n)^{-3}
      \otr[\Sigma (v(0; \gamma_n) \Sigma + I)^{-1}]\big]\\ 
    &=\frac{\otr[\Sigma^2 (v(0; \gamma_n) \Sigma + I)^{-2}]}{\big(v(0;
      \gamma_n)^{-2} - \gamma_n\otr[\Sigma^2 (v(0; \gamma_n) \Sigma +
      I)^{-2}]\big)^3}\\ 
    &\qquad \cdot \big[\big(v(0; \gamma_n)^{-2} - \gamma_n\otr[\Sigma^2 (v(0;
      \gamma_n) \Sigma + I)^{-2}]\big)^3 - 2 v(0; \gamma_n)^{-4}\big]\\ 
    &=\frac{\otr[\Sigma^2 (v(0; \gamma_n) \Sigma + I)^{-2}]}{\big(v(0;
      \gamma_n)^{-2} - \gamma_n\otr[\Sigma^2 (v(0; \gamma_n) \Sigma +
      I)^{-2}]\big)^3}\\ 
    &\qquad \cdot \big[- v(0; \gamma_n)^{-4} - v(0; \gamma_n)^{-2} - \big(v(0;
      \gamma_n)^{-2} - \gamma_n \otr[\Sigma^2 (v(0; \gamma_n) \Sigma +
      I)^{-2}]\big) \\
    &\qquad\qquad \cdot \gamma_n \otr[\Sigma^2 (v(0; \gamma_n) \Sigma + 
      I)^{-2}]\big] \\ 
    &\leq 0.
\end{align*}
Here, we use the fact from Lemma F.11 (3) in \citet{du2023subsample} that 
\[
\frac{1}{v(0;\gamma)^2} - \gamma \otr[\Sigma^2 (v(0; \gamma) \Sigma + I)^{-2}]
\geq 0,
\] 
with equality obtained only when $\gamma = \infty$. This indicates that
$\opt\r\i(\hf_0^\ridge \,|\, X)$ is strictly increasing in $\gamma_n$ for
$\gamma_n\in (1,\infty)$, with  
\[
    \lim_{\gamma_n \to 1^+} \opt\r\i(\hf_0^\ridge \,|\, X)/\sigma^2 = \infty,  
    \quad \text{and} \quad 
    \lim_{\gamma_n \to \infty} \opt\r\i(\hf_0^\ridge \,|\, X)/\sigma^2 = 0.
\]
Consequently, $\df\r\i(\hf_0^\ridge) / n$ is decreasing from 1 to 0 in
$\gamma_n\in (1,\infty)$. 

\paragraph{Emergent random-X optimism.}

From the proof of \cref{thm:ridgeless-asymptotics}, when $\gamma_n < 1$, we have
\[
    \opt\r(\hf_0^\ridge \,|\, X) / \sigma^2
    \asympequi
    (\gamma_n + \gamma_n / (1 - \gamma_n)) (1 + \sigma^2\nl/\sigma^2), 
\]
which is strictly increasing in $\gamma_n \in (0, 1)$ with the following
boundary limit: 
\[
    \lim_{\gamma_n \to 1^-}
    \opt\r(\hf_0^\ridge \,|\, X) / \sigma^2
    = \infty.
\]
Consequently, $\df\r(\hf_0^\ridge) / n$ is increasing from 0 to 1 on
$\gamma_n\in (0,1)$ and maximized at $\gamma_n = 1$. This finishes the proof. 

\subsection{Proof of \cref{prop:lasso-monotonicity-nonnegativity}}
\label{app:lasso-monotonicity-nonnegativity}

We will first parameterize the nonlinear system in
\eqref{eq:lasso-fixed-point-tau} and \eqref{eq:lasso-fixed-point-mu} slightly 
differently by introducing a new variable $a = \mu / \tau$. Namely, we let
$(\tau, a)$ solve:
\begin{empheq}{align}
  \tau^2 &= \sigma^2 + \gamma \EE [(\soft(\tau H + B ; {a\tau}) - B)^2],
  \label{eq:interpolators-tau-reformulation-supplement} \\ 
  \lambda &= a \tau (1 - \gamma \EE [\soft' (\tau H + B ;  {a\tau}) ]).
  \label{eq:interpolators-a-reformulation-supplement}
\end{empheq}
The nonlinear system in \eqref{eq:interpolators-tau-reformulation-supplement}
and \eqref{eq:interpolators-a-reformulation-supplement} is similar to the one in
\citet{bayati2011lasso}. When $B=0$ (almost surely), we denote its solution by
$(\tau_0, a_0)$. Before we start the proof, we will collect the following two
properties of soft-thresholding (the proximal operator for the $\ell_1$ norm)
for $a > 0$:   
\begin{align}
  \soft(x ; \kappa) &= \tfrac{1}{a} \soft(ax ; a \kappa), \label{eq:scale-prox} \\
  \soft'(x ; \kappa) &= \soft'(ax ; a \kappa). \label{eq:scale-prox-prim}
\end{align}
These are straightforward to check (see, e.g., Lemma B.2 in
\citet{wang2020bridge}). We will split the proof below into two parts, following 
the two statements in the proposition. As before, since $\omega$ is strictly 
increasing, it suffices to show the desired properties on the optimism scale. 

\paragraph{Monotonicity of intrinsic random-X optimism.}

Combining \eqref{eq:lasso-df-F-asympequi} and \eqref{eq:lasso-df-R-i-asympequi},
we can write 
\[
     \opt\r\i(\hf_\lambda^\lasso \,|\, X, y)
     \asympequi (1 - (1 - \df\f\i(\hf_\lambda^\lasso) / n))^2 \tau_0^2. 
\]
Below, we will argue that each of $\tau_0^2$ and $\df\f\i(\hf_\lambda^\lasso) /
n$ are monotonic in $\lambda$, with limits $\tau_0^2 \to \sigma^2$ and 
$\df\f\i(\hf_\lambda^\lasso) / n\to 0$ as $\lambda \to \infty$. 

\emph{Monotonicity of $\tau_0^2$.} We first argue below that $\tau_0^2$ is
monotonically nonincreasing in $\lambda$. We have 
\[
    \tau_0^2
    = \sigma^2 
    + \gamma \EE [(\soft(\tau H ;  {a_0\tau_0}))^2]
    = \sigma^2 
    + \gamma \tau_0^2 \EE [(\soft(H ;  {a_0}))^2],
\]
where we used \eqref{eq:scale-prox} in the second equality above. Rearranging,
we get that 
\[
    \tau_0^2 = \frac{\sigma^2}{1 - \gamma \EE[(\soft(H;a_0))^2] } .
\]
Now, observe that the right-hand side is monotonically nonincreasing in $a_0$,
which follows because $x \mapsto |\soft(u;x)|$ is noncreasing in $x$ for fixed
$u$, and $a_0$ is nondecreasing in $\lambda$ from Corollary 1.7 of
\citet{bayati2011lasso}. This implies that $\tau_0^2$ is monotonically
nonincreasing in $\lambda$. Lastly, by Corollary 1.7 of
\citet{bayati2011lasso} once again, we have $a_0 \to \infty$ as $\lambda \to 
\infty$, and hence $\EE[\soft(H;a_0)^2] \to 0$, and $\tau_0^2 \to \sigma^2$
as $\lambda \to \infty$.    

\emph{Monotonicity of $\df\f\i(\hf_\lambda^\lasso) / n$.} To show that
$\df\f\i(\hf_\lambda^\lasso) / n$ is decreasing in $\lambda$, observe that  
\[
    \frac{\df\f\i(\hf_\lambda^\lasso)}{n} 
    \asympequi \gamma \EE[\soft'(\tau_0 H; a_0 \tau_0)].
\]
To see this, note from \eqref{eq:lasso-fixed-point-mu}, after replacing $\mu_0$
with $a_0 \tau_0$, that 
\[
    1 - \lambda / \mu_0
    = \gamma \EE [\soft' (\tau_0 H; a_0 \tau_0)]
\]
Using \eqref{eq:scale-prox-prim}, we have $\gamma \EE[\soft'(\tau_0 H; a_0
\tau_0)] = \gamma \EE[\soft'(H; a_0)]$. Also, $\EE[\soft'(H; a_0)] = \PP(|H| >
a_0)$, which is nonincreasing in $a_0$. Using the monotonically nondecreasing    
behavior of $a_0$ in $\lambda$ from Corollary 1.7 of \citet{bayati2011lasso}, we 
then have the desired monotonicity. Lastly, that $\df\f\i(\hf_\lambda^\lasso) /
n \to 0$ as $\lambda \to \infty$ follows from $\lambda \to \mu_0 \to 1$, which
can be checked from \eqref{eq:lasso-fixed-point-mu}.

\paragraph{Nonnegativity of emergent minus intrinsic optimism.} 

From \eqref{eq:lasso-df-F-asympequi}, we can write:
\[
     \opt\r\i(\hf_\lambda^\lasso \,|\, X, y)
     \asympequi \tau_0^2 (1 - (1 - \df\f\i(\hf_\lambda^\lasso) / n)))^2.
\]
Similarly, we can write the emergent optimism as:
\[
     \opt\r(\hf_\lambda^\lasso \,|\, X, y)
     \asympequi \tau^2 (1 - (1 - \df\f(\hf_\lambda^\lasso) / n))^2. 
\]
To show the asymptotic equivalent for $\opt\r(\hf_\lambda^\lasso \,|\, X, y)$ is
no less than that for $\opt\r\i(\hf_\lambda^\lasso \,|\, X, y)$, we will argue
that $\tau^2 \ge \tau_0^2$ and $\df\f(\hf_\lambda^\lasso) \ge
\df\f\i(\hf_\lambda^\lasso)$, below.  

\emph{Nonnegativity of $\df\f(\hf_\lambda^\lasso) -
  \df\f\i(\hf_\lambda^\lasso)$.} The two quantities we need to compare are: 
\[
    {\df\f}(\hf_\lambda^\lasso) / {n}
    \asympequi
    \gamma \EE[\soft'(\tau H + B; a \tau)]
    \quad
    \text{and}
    \quad
    {\df\f\i}(\hf_\lambda^\lasso) / {n}
    \asympequi
    \gamma \EE[\soft'(\tau_0 H; a_0 \tau_0)].
\]
Using \eqref{eq:scale-prox-prim}, we first rewrite the asymptotic equivalents in
the display above as: 
\[
    {\df\f}(\hf_\lambda^\lasso) / {n}
    \asympequi
    \gamma \EE[\soft'(H + B / \tau; a)]
    \quad
    \text{and}
    \quad
    {\df\f\i}(\hf_\lambda^\lasso) / {n}
    \asympequi
    \gamma \EE[\soft'(H; a_0)].
\]
Observe now that for $\tau \ge 0$, assuming $a \le a_0$, we have
\begin{align*}
    \EE[\soft'(H + B / \tau; a)]
    \ge \EE[\soft'(H; a)]
    \ge \EE[\soft'(H; a_0)],
\end{align*}
The first inequality can be explained as follows: $\EE[\soft'(H + b; a) \ge
\PP(|H| > a) =  \EE[\soft'(H; a)]$ for any fixed $b$ and hence $\EE[\soft'(H + B
/ \tau; a)] \ge \EE[\soft'(H; a)]$ by conditioning on the random variable $B$
which is independent of $H$. Thus we get the desired claim that
$\df\f(\hf_\lambda^\lasso) \geq \df\f\i(\hf_\lambda^\lasso)$ assuming $a \le
a_0$, which we will show in the next part, along with $\tau \ge \tau_0$. 

\emph{Nonnegativity of $\tau^2 - \tau_0^2$.} We consider solving the system
for emergent parameters \eqref{eq:interpolators-tau-reformulation-supplement},
\eqref{eq:interpolators-a-reformulation-supplement}. We will solve these using
the fixed point iteration algorithm initialized at the solution $\tau_0, a_0$ of
the system with intrinsic parameters. 

Namely, we will start with $a^{(0)} = a_0$ and $\tau^{(0)} = \tau_0$. If $a_0$
and $\tau_0$ solve the emergent system, then we are done. Suppose they do not. 
Then, we first solve \eqref{eq:interpolators-tau-reformulation-supplement} with
fixing $a = a^{(0)}$ and solving for $\tau$. Call this solution $\tau^{(1)}$. We
claim that $\tau^{(1)} \ge \tau^{(0)} = \tau_0$. Suppose in order to achieve a
contradiction that $\tau^{(1)} < \tau^{(0)}$. Rewrite
\eqref{eq:interpolators-tau-reformulation-supplement} after normalizing with 
respect to $\tau^2$:  
\[
    1 = \frac{\sigma^2}{\tau^2} 
    + \gamma \EE\bigg[\bigg(\soft\Big(H + \frac{B}{\tau}; a\Big) -
    \frac{B}{\tau}\bigg)^2\bigg]. 
\]
From Lemma 12 in \citet{weng2018overcoming}, we know that the function that
multiplies $\gamma$ in the display above is a decreasing function of $\tau$. 
Note that the function $h:x\mapsto \EE[(\soft(H + x B; a) - xB)^2] $ is an even 
function, as $h(x) = \EE[(\soft(-H + x B; a) - xB)^2] = \EE[(\soft(H + (-x) B;
a) - (-x)B)^2] = h(-x)$. From Lemma 6 in \citet{weng2018overcoming}, we have that
$h(x)$ is increasing in $x$. Thus, the same function in the above display has a
larger value when $B \neq 0$. Thus, if $\tau^{(1)} < \tau^{(0)}$, then both of
the terms on the right-hand side of the display above increase. But we already
know that $a = a^{(0)}$ satisfies the equation with $\tau^{(0)}$. This supplies
the desired contradiction.  

Now, fix this $\tau^{(1)}$, and solve
\eqref{eq:interpolators-a-reformulation-supplement} for $a$. Call this solution
$a^{(1)}$. As before, we claim $a^{(1)} \le a^{(0)} = a_0$. This follows again
from a contradiction-based argument because if $a^{(1)} > a^{(0)}$, then both
the terms on the right-hand side of
\eqref{eq:interpolators-a-reformulation-supplement} go up because the term
multiplying $\gamma$ is decreasing in $a$ (since we can eliminate $\tau$) and 
has a larger value when $B \neq 0$. 

Iterating the above argument, we obtain two monotonic nonnegative sequences
$a^{(m)}$, $\tau^{(m)}$. When $\tau^{(m)}=\infty$ one has $a^{(m)} =0$,
and when $a^{(m)} =0$, one has $\tau^{(m+1)}=\infty$. Thus, we have $a \ge 
0$ and $\tau \le \infty$, which indicates that the process terminates as
$m\to\infty$, and $\tau \geq \tau_0$, $a \leq a_0$. 

\subsection{Proof of \cref{thm:lassoless-asymptotics}}
\label{app:lassoless-asymptotics}

In the underparameterized regime (when $\gamma \leq 1$), the statements follow 
from \cref{thm:ridgeless-asymptotics} since both predictors are simply least
squares in this regime. In the overparameterized regime (when $\gamma > 1$), the
results follow by sending $\lambda \to 0^+$ in the results of
\cref{thm:lasso-asymptotics}. The validity of this limit, along with the
existence and uniqueness of the solution to the nonlinear system
\eqref{eq:lassoless-fixed-point-tau} and \eqref{eq:lassoless-fixed-point-mu} is
shown by \citet{li_wei_2021}. 

\subsection{Proof of \cref{prop:lassoless-monotonicity}}
\label{app:lassoless-monotonicity}

For $\gamma \in (1, \infty)$, the parameters $(\tau_0, a_0)$ solve the system: 
\begin{empheq}{align*}
    \tau_0^2 &= \sigma^2 + \gamma \EE \big[\big(\soft\big(\tau_0 H ; {a_0\tau_0}
    \big) \big)^2\big] \\
    1 &= \gamma \EE\big[\soft' \big(\tau_0 H ; {a_0\tau_0}\big) \big].
\end{empheq}
Here recall that $\soft'(x; y)$ is the derivative of $\soft(x; y)$ in $x$. This
can be simplified to:
\begin{empheq}{align*}
    1 &= \sigma^2 / \tau_0^2 + \gamma \EE \big[\big(\soft\big(H ; {a_0} \big) 
    \big)^2\big] \\
    1 &= \gamma \EE\big[\soft' \big(H ; {a_0}\big) \big].
\end{empheq}
This leads to:
\[
    \tau_0^2 = \frac{\sigma^2}{1 - \gamma \EE[(\soft(H; a_0))^2]},
\]
where $a_0$ solves:
\[
    1 = \gamma \EE[\soft'(H; a_0)].
\]
We first conclude that $a_0$ is monotonically increasing in $\gamma \in (1,
\infty)$ and ranges from $0$ to $\infty$. This follows because the function $x
\mapsto |\soft(u;x)|$ is decreasing in $x$, for fixed $u$. In particular,   
\[
    \EE[\soft'(H; a_0)]
    = \PP(|H| > a_0)
    = 2 (1 - \Phi(a_0))
    = \frac{1}{\gamma}.
\]
This leads to
\[
    a_0 = \Phi^{-1}\Big(\frac{2 \gamma - 1}{2 \gamma}\Big). 
\]
Since both the functions $\Phi^{-1}$ and $\tfrac{2\gamma - 1}{2\gamma}$ are 
monotonically increasing $\gamma$, we have that the composition is monotonically 
increasing in $\gamma$. When $\gamma = 1$, we have $a_0 = 0$ and when $\gamma =
\infty$, we have $a_0 = \infty$. 

Next we will argue that $\gamma \mapsto \gamma \EE[(\soft(H; a))^2]$ decreases
in $\gamma \in (1, \infty)$ and ranges from $1$ to $0$. We do by first
substituting for $\gamma$ as $\frac{1}{\EE[\soft'(H; a)]}$. The goal then
reduces to arguing that the function 
\[
    \gamma \mapsto \frac{\EE[(\soft(H; a))^2]}{\EE[\soft'(H; a)]} 
\]
is decreasing in $\gamma$. Since $a$ is increasing in $\gamma$, it suffices to
argue that the function 
\[
    y \mapsto\frac{\EE[(\soft(H; y))^2]}{\EE[\soft'(H; y)]} 
\]
is decreasing in $y$. This follows from \cref{lem:monotonicity-ratio-prox1}
below, and finishes the proof.

\begin{lemma}
    \label{lem:monotonicity-ratio-prox1}
    For $H \sim \cN(0, 1)$, the function
    \[
        y \mapsto \frac{\EE[\soft(H; y)^2]}{\EE[\soft'(H; y)]}
    \]
    is monotonically decreasing in $y$. Here, recall, the derivative of $\soft$
    is understood to be with respect to its first argument. 
\end{lemma}

\begin{proof}
Denote the numerator and the denominator by
\begin{align*}
    f(y) &= \EE[\soft(H; y)^2] = 2\EE[(H-y)^2\ind\{H>y\}] \\ 
    g(y) &= \EE[\soft'(H; y)] = \EE[H\soft(H; y)] = f(y) +
           2y\EE[(H-y)\ind\{H>y\}]. 
\end{align*}
Here, in the second equality of the second row, we use Stein's lemma.

Recall for $X\sim\cN(0,1)$, the truncated normal distribution admits 
\begin{align*}
    \EE[X\mid X>a] &= \varphi (a )/(1-\Phi(a))   \\
    \Var (X\mid X>a) &= 1+a \varphi (a )/(1-\Phi(a))-
         (\varphi (a)/(1-\Phi(a)))^{2} \\
    \EE[X^2\mid X>a] &= \Var(X\mid X>a) +\EE[X\mid X>a]^2\\ 
    &= 1+a \varphi (a )/(1-\Phi(a))-(\varphi (a )/(1-\Phi(a)))^{2} + (
      \varphi(a)/(1-\Phi(a)))^2\\
    &= 1+a \varphi (a )/(1-\Phi(a)).
\end{align*}
Then we have
\begin{align*}
    \EE[(H-y)\ind\{H>y\}] &= \varphi(y)  - y(1-\Phi(y))\\
    f(y) &= 2 (\EE[H^2\mid H>y]\PP(H>y) - 2 y\EE[H\mid H>y]\PP(H>y) +
           y^2(1-\Phi(y)) \\ 
    &= 2[ (1-\Phi(y)) +y \varphi (y ) - 2y \varphi(y)  +y^2(1-\Phi(y)]\\   
    &= 2[- y \varphi (y )  + (1+y^2)(1-\Phi(y))]\\
    g(y) &= 2\EE[(H-y)^2\ind\{H>y\}]+ 2y\EE[(H-y)\ind\{H>y\}] \\
    &=f(y) + 2y\EE[(H-y)\ind\{H>y\}]\\
    &= f(y) + \underbrace{2y( \varphi(y)  - y(1-\Phi(y)))}_{h(y)}.
\end{align*}
Because $\Phi(y) = \varphi(y)$ and $\varphi'(y) = -y\varphi(y)$, we further have 
\begin{align*}
    f'(y)g(y) - f(y)g'(y) &= f'(y)[f(y)+h(y)] - f(y) [f'(y) + h'(y)] \\
    &=f'(y)h(y) - f(y)h'(y) \\
    &= 2[-\varphi(y)+y^2\varphi(y) +2y(1-\Phi(y)) - (1+y^2)\varphi(y)]h(y)\\
    &\qquad -f(y) 2[\varphi(y)-y^2\varphi(y) - 2y(1-\Phi(y)) + y^2\varphi(y)]\\ 
    &= 4[y(1-\Phi(y))-\varphi(y)] h(y) + 4f(y)[y(1-\Phi(y))-\varphi(y)]\\
    &=4[\underbrace{y(1-\Phi(y))-\varphi(y)}_{c(y)}] [h(y) + f(y)].
\end{align*}
Now $c'(y) = (1-\Phi(y)-y\varphi(y) +y\varphi(y) = 1-\Phi(y) \geq 0$ and
$\lim_{y\to\infty} c(y) = 0$, thus we have $c(y) \leq 0$ and
hence 
\[
\frac{\partial}{\partial y} \frac{f(y)}{g(y)} = \frac{f'(y)g(y) -
  f(y)g'(y)}{g(y)^2} \leq 0,
\] 
which finishes the proof.
\end{proof}

\subsection{Proof of \cref{thm:convex-asymptotics}}
\label{app:convex-asymptotics}

We will use results from \citet{thrampoulidis_abbasi_hassibi_2018}, which use a 
slightly different scaling for the feature matrix. In particular, they use a
variance scaling of $1/p$ for the entries of the feature vector $x_i$, whereas
recall (from \cref{asm:cgmt}), we consider a variance scaling of $1/n$. We can
thus rewrite the estimator of interest from \eqref{eq:convex-opt} (after
dividing by $\gamma$) as: 
\begin{align}
  \label{eq:convex-regularized-supplement}
  \hbeta_\lambda^{\convex} \in \argmin_{b \in \RR^p} \frac{1}{2}
  \sum_{i=1}^{n} (\ty_i - \tx_i^\top b)^2 + \tlambda \sum_{i=1}^p \reg(b_i),  
\end{align}
where the transformed variables are:
\begin{equation}
\label{eq:data-transformation-cgmt}
\tx_i = \gamma^{-1/2} x_i, \quad \ty_i = \gamma^{-1/2} y_i, \quad \tlambda =
\lambda / \gamma. 
\end{equation}
This transformation follows since the minimizers do not change up to positive
scaling (which in our case is by $\gamma$) of the objective function. 
Since $x_i$ has i.i.d.\ entries with variance $1/n$ and $y_i = x_i^\top \beta +
\eps_i$ in \cref{asm:cgmt}, in the transformed formulation
\eqref{eq:convex-regularized-supplement}, the feature vectors $\tx_i$ have 
i.i.d.\ entries with variance $1/p$, and the response variables follow the
linear model $\ty_i = \tx_i \beta + \teps_i$, with the transformed noise defined 
as $\teps_i = \gamma^{-1/2} \eps_i$.

With the transformation in \eqref{eq:convex-regularized-supplement}, we now
apply the master theorem of \citet{thrampoulidis_abbasi_hassibi_2018}. 
Define the following nonlinear system of equations in four scalar variables
$(\alpha, \zeta, \kappa, \nu)$: 
\begin{empheq}{align}
    \alpha^2 &= \EE \big[
    \big( \tfrac{\gamma \lambda}{\nu} \cdot \env_{\reg}' ( \tfrac{\zeta}{\nu}
    H + B; \tfrac{\gamma \lambda}{\nu} ) - \tfrac{\zeta}{\nu} H \big)^2 \big]  
    = \EE \big[ \big( \prox_{\reg}( \tfrac{\zeta}{\nu} H + B; \tfrac{\gamma
      \lambda}{\nu}) - B \big)^2 \big] 
    \label{eq:CGMT-1a}\\
    \gamma \zeta^2 &= \frac{\alpha^2 + \sigma^2/\gamma}{(1 +
      \kappa)^2} \label{eq:CGMT-1b}\\ 
    \kappa\zeta &= \EE\big[  
    \big(\tfrac{\gamma \lambda}{\nu} \cdot \env_{\reg}' (\tfrac{\zeta}{\nu} H +
    B ; \tfrac{\gamma \lambda}{\nu} ) - \tfrac{\zeta}{\nu} H \big) \cdot (-H) 
    \big] = \EE \big[ \big(\prox_\reg( \tfrac{\zeta}{\nu} H + B;
    \tfrac{\lambda}{\nu}) - B \big) \cdot H \big] \label{eq:CGMT-1c} \\ 
    \gamma \nu &= \frac{1}{1 + \kappa} \label{eq:CGMT-1d} 
\end{empheq}
where $H\sim \cN(0,1)$, and$B \sim F$ independently of $H$. As usual, when $B=0$
(almost surely), we denote the solution by $(\alpha_0,\zeta_0,\kappa_0,\nu_0)$.  

The parameters from \eqref{eq:CGMT-1a}--\eqref{eq:CGMT-1d} encode information
regarding the asymptotics of various stochastic quantities that we will need in
our derivation. Before we do that, we will reformulate the system above to
better align with the results for the ridge and lasso predictors. 

\paragraph{Reformulation of \eqref{eq:CGMT-1a}--\eqref{eq:CGMT-1d}.} 

Consider the following change of variables:
\begin{equation}
  \label{eq:def-a-tau}
    a = \frac{\gamma \lambda}{\zeta}
    \quad
    \text{and}
    \quad
    \tau = \frac{\zeta}{\nu}.
\end{equation}
We will first reformulate \eqref{eq:CGMT-1a}--\eqref{eq:CGMT-1d} using $(\tau,
a)$, yielding the following equivalent system:  
\begin{empheq}{align}
  \tau^2 &= \sigma^2 + \gamma \EE [(\prox_{\reg}(\tau H +  B ; {a\tau}) -
  B)^2] \label{eq:CGMT-tau-reformulation}, \\  
  \lambda &= a \tau (1 - \gamma \EE[\prox_{\reg}' (\tau H +  B ;
  {a\tau})]), \label{eq:CGMT-a-reformulation} 
\end{empheq}
where $H \sim \cN(0, 1)$ and $\bTheta \sim F$ independent of $H$.
The validity of this reformulation is proved later on. Letting $\mu = a
\tau$, the system in \eqref{eq:CGMT-tau-reformulation},
\eqref{eq:CGMT-a-reformulation} is exactly the same (after rearranging) as 
\eqref{eq:cgmt-fixed-point-tau}, \eqref{eq:cgmt-fixed-point-mu}.  

We are finally ready to obtain the asymptotics of the various notions of degrees
of freedom, which we present in separate parts in what follows.

\paragraph{Fixed-X degrees of freedom.}

We first note that for the estimator $\hbeta_\lambda^\convex$ as defined in
\eqref{eq:convex-opt}, the map $y \mapsto X \hbeta_\lambda^{\convex}$ is
1-Lipschitz on $\RR^n$ (see, e.g, Proposition 3 of \citet{bellec2017bounds}) and  
has symmetric positive semidefinite Jacobian. Thus it is weakly differentiable
and Stein's formula can be applied, which shows that its fixed-X degrees of
freedom are then given by:
\[
    \df\f(\hf_\lambda^\convex)
    = \EE\big[ \tr[(\partial/\partial y) X \hbeta_\lambda^\convex] \,|\, X
    \big].  
\]
Now, observe that
\[
    (\partial / \partial \ty) \tX \hbeta_\lambda^{\convex}
    = (\partial / \partial \ty) \gamma^{-1/2} X \hbeta_\lambda^{\convex}
    = (\partial / \partial y) (\partial y / \partial \ty) \gamma^{-1/2} X
    \hbeta_\lambda^{\convex} = (\partial / \partial y) X
    \hbeta_\lambda^{\convex}. 
\]
Thus, fixed-X degrees of freedom is unchanged under the transformation of the 
data in \eqref{eq:data-transformation-cgmt}:  
\begin{equation}
    \label{eq:DofF-formula-penalized-estimators}
    \df\f(\hf_\lambda^\convex)
    = \EE\big[ \tr[(\partial/\partial \ty) \tX \hbeta_\lambda^\convex]
    \,|\, X \big]. 
\end{equation}
In what follows, we first obtain limit in probability of the trace functional
$\tr[(\partial/\partial \ty) \tX \hbeta_\lambda^\convex] / p$, and then convert
this convergence to obtain the desired limit of $\df\f(\hf_\lambda^\convex) /
n$ in \eqref{eq:DofF-formula-penalized-estimators}.  

Define the matrix $V_\lambda = I -  (\partial / \partial \ty) \tX
\hbeta_\lambda^\convex \in \RR^{n \times n}$. By Corollary 3.2 in
\citet{bellec2023out}, as $n,p \to \infty$ with $p/n \to \gamma \in (0,
\infty)$, we have   
\begin{equation}
    \label{eq:limit-trace-V}
    \tr[V_\lambda] / p \pto \nu.
\end{equation}
We mention in the passing here that the trace convergence result
\eqref{eq:limit-trace-V} in the special case of lasso follows from Theorem
8 of \citet{celentano2023lasso} and in the more general case of convex
regularized M-estimators follows from Appendix A.4 of
\citet{koriyama2024precise}. Now rearranging \eqref{eq:limit-trace-V}, we get  
\begin{equation}
    \label{eq:bDofF-penalized-estimators-form1}
    \tr[(\partial/\partial \ty) \tX \hbeta_\lambda^\convex] / n
    \pto 1 - \gamma \nu = 1 - \gamma \EE\big[\prox_reg'( B + \tau H;  a \tau)\big]
    = 1 - \lambda / \mu, 
\end{equation}
where the second-to-last equality follows from
\eqref{eq:ellq-interpolator-eq2-0-without-atau}, and the last equality follows 
from \eqref{eq:CGMT-a-reformulation} (after the change of variables $\mu = a
\tau$). Finally, noting that $\tr[V_\lambda] / n$ ranges between $[0,1]$ for almost
every $y$ (see, e.g., Proposition 2.2 of \citet{bellec2023out}), invoking the 
dominated convergence theorem (to be clear, a variant that handles convergence
in probability by passing to a subsequence; see, e.g., Exercise 2.3.7 of
\citet{durrett2010probability}) to convert
\eqref{eq:bDofF-penalized-estimators-form1} to a statement about convergence in
expectation, we have that $\df\f(\hf_\lambda^\convex) / n$ converges to the same
limit. This finishes the proof of \eqref{eq:convex-df-F-asympequi}.

\paragraph{Emergent random-X degrees of freedom.}

Next we consider emergent random-X optimism. Under the scaled (by $n$) isotropic  
features and linear model in \cref{asm:cgmt}, observe that  
\begin{align}
    \err\r(\hf_\lambda^\convex \,|\, X, y)
    &= \sigma^2 + \| \hbeta_\lambda^\convex - \beta \|_2^2 / n,
    \label{eq:ErrR-penalized-estimators}
    \\
    \err\t(\hf_\lambda^\convex \,|\, X, y)
    &= \| y - X \hbeta_\lambda^\convex \|_2^2 / n
    = \gamma \| \ty - \tX \hbeta^\convex \|_2^2 / n.
    \label{eq:ErrT-penalized-estimators}
\end{align}
From Theorem 4.1 in \citet{thrampoulidis_abbasi_hassibi_2018}, for the problem  
\eqref{eq:convex-regularized-supplement}, we note that 
\begin{equation}
    \label{eq:ErrR-ErrT-limits-penalized-estimators}
    \| \hbeta_\lambda^\convex - \beta \|_2^2 / p
    \pto \alpha^2
    \quad
    \text{and}
    \quad
    \| \ty - \tX \hbeta_\lambda^\convex \|_2^2 / p
    \pto \zeta^2,
\end{equation}
as $n,p \to \infty$ with $p/n \to \gamma \in(0, \infty)$.
Combining \eqref{eq:ErrR-penalized-estimators},
\eqref{eq:ErrT-penalized-estimators} and
\eqref{eq:ErrR-ErrT-limits-penalized-estimators}, we have   
\begin{equation}
    \label{eq:OptR-limit-penalized-estimators}
    \opt\r(\hf_\lambda^\convex \,|\, X, y)
    \pto \gamma \alpha^2 + \sigma^2 - \gamma^2 \zeta^2
    = \tau^2 - \gamma^2 \zeta^2,
\end{equation}
where the last line follows from combining \eqref{eq:CGMT-1a} and
\eqref{eq:CGMT-tau-reformulation}. Also, note from \eqref{eq:CGMT-1b},
\eqref{eq:CGMT-1d}, and \eqref{eq:bDofF-penalized-estimators-form1}, we have  
\begin{equation}
    \label{eq:gcv-relation-proof}
    \frac{\gamma^2 \zeta^2}{\gamma \alpha^2 + \sigma^2}
    = \frac{1}{(1 + \kappa)^2}
    = \gamma^2 \nu^2
    \asympequi \lambda^2 / \mu^2,
\end{equation}
where the last equivalence follows from 
\eqref{eq:bDofF-penalized-estimators-form1}. Again, using \eqref{eq:CGMT-1a} and
\eqref{eq:CGMT-tau-reformulation}, note that we can rewrite
\eqref{eq:gcv-relation-proof} 
as 
\begin{equation}
    \label{eq:gcv-relation-proof-tau}
    \gamma^2 \zeta^2 \asympequi \lambda^2 / \mu^2 \cdot \tau^2.
\end{equation}
Substituting
\eqref{eq:gcv-relation-proof-tau} into \eqref{eq:OptR-limit-penalized-estimators}
and applying $\omega$ and dominated convergence finishes the proof of
\eqref{eq:convex-df-R-asympequi}.    

\paragraph{Intrinsic random-X degrees of freedom.}

The proof for the intrinsic case follows similarly. When the signal is absent,
we have   
\begin{equation}
    \label{eq:OptRi-limit-penalized-estimators}
    \opt\r\i(\hf_\lambda^\convex \,|\, X, y) 
    \to \sigma^2 + \gamma \alpha_0^2 - \gamma^2 \zeta_0^2
    = \gamma \tau_0^2 - \lambda^2 / \mu_0^2 \cdot \tau_0^2, 
\end{equation}
where we replaced $\alpha$ with $\alpha_0$, $\zeta$ with $\zeta_0$, and
$\mu$ with $\mu_0$ in \eqref{eq:OptR-limit-penalized-estimators} and
\eqref{eq:gcv-relation-proof-tau}. Applying $\omega$ to
\eqref{eq:OptRi-limit-penalized-estimators} and invoking the dominated
convergence theorem finishes the proof of \eqref{eq:convex-df-R-i-asympequi}.   

\paragraph{Derivation of the reformulation
  \eqref{eq:CGMT-1a}--\eqref{eq:CGMT-1d}.} 

Using \eqref{eq:def-a-tau}, along with \eqref{eq:cgmt-fixed-point-tau} and
\eqref{eq:CGMT-1d}, note that \eqref{eq:CGMT-1b} becomes 
\[
    \tau^2 
    = \sigma^2 + \gamma \EE [ ( \prox_{\reg}( \tau H +  B;  a \tau ) -  B )^2 ].
\]
This supplies us with \eqref{eq:CGMT-tau-reformulation}.
Now, define the Moreau envelope by
\[
\env_\reg (x; t) = \min_{z \in \RR} \, \frac{1}{2t} (x - z)^2 + \reg(z).
\]
We recall a key relationship between the proximal operator and Moreau envelope. 
\begin{equation}
    \label{eq:prox_moreau_relation}
    \env_\reg'(x; \tau) 
    = \frac{1}{\tau} (x - \prox_\reg(x; \tau)).
\end{equation} 
Towards obtaining \eqref{eq:CGMT-a-reformulation}, from Stein's lemma, observe
that 
\begin{equation}
    \label{eq:stein-lemma}
    \EE[\env'_{\reg}(B + \tau H; \kappa) \cdot H]
    = \tau \EE[\env''_{\reg}(B + \tau H; \kappa)].
\end{equation}
Taking the derivative of the relation \eqref{eq:prox_moreau_relation}, we also
have 
\begin{equation}
    \label{eq:deriv-prox-moreau-relation}
    \kappa \env''_{\reg}(B + \tau H; \kappa) = 1 - \prox_{\reg}'(B + \tau H;
    \kappa).  
\end{equation}
Combining \eqref{eq:stein-lemma} and \eqref{eq:deriv-prox-moreau-relation}, we
obtain 
\begin{equation}
    \label{eq:CGMT-1c-rewrite-0}
    \EE\big[\env'_{q}(\tfrac{\zeta}{\nu} H + B; \tfrac{\lambda}{\nu}) \cdot
    H\big] 
    = \tfrac{\zeta}{\nu} \EE\big[ \env''_{q}(\tfrac{\zeta}{\nu} H + B;
    \tfrac{\lambda}{\nu}) \big] 
    = \tfrac{\zeta}{\nu} \tfrac{\nu}{\lambda} \EE\big[1 -
    \prox_{\reg}'(\tfrac{\zeta}{\nu} H + B; \tfrac{\lambda}{\nu})\big]. 
\end{equation}
Using \eqref{eq:CGMT-1c-rewrite-0}, we can rewrite
\eqref{eq:cgmt-fixed-point-mu} as: 
\begin{equation}
    \label{eq:CGMT-1c-rewrite-1}
    \kappa \zeta 
    = \tfrac{\zeta}{\nu} - \tfrac{\lambda}{\nu} \tfrac{\zeta}{\lambda} \EE\big[1
    - \prox_{\reg}'(\tfrac{\zeta}{\nu} H + B; \tfrac{\lambda}{\nu})\big] 
    = \tfrac{\zeta}{\nu} \big(1 - \EE\big[1 - \prox_{\reg}'(\tfrac{\zeta}{\nu} H
    + B; \tfrac{\lambda}{\nu})\big]\big). 
\end{equation}
Now, using \eqref{eq:def-a-tau}, we can express \eqref{eq:CGMT-1c-rewrite-1} as:   
\begin{equation}
    \kappa \nu
     =
     1 - \EE\big[1 - \prox_{\reg}'( B + \tau H;  a \tau)\big]. 
\end{equation}
Rearranging and using \eqref{eq:CGMT-1d} yields
\begin{equation}
    \label{eq:ellq-interpolator-eq2-0-without-atau}
    \gamma \nu
    = 1 - \gamma \EE\big[\prox_{\reg}'( B + \tau H;  a \tau)\big].
\end{equation}
Multiplying both sides of \eqref{eq:ellq-interpolator-eq2-0-without-atau} by $a 
\tau$ and using \eqref{eq:def-a-tau}, we then arrive at: 
\begin{equation}
    \label{eq:ellq-interpolator-eq2-0}
     \lambda = a \tau (1 - \gamma \EE[\prox_{\reg}'( B + \tau H;  a \tau)]). 
\end{equation}
This supplies us with \eqref{eq:CGMT-a-reformulation}, completing the
reformulation.

\section{Numerical experiments for \cref{sec:theory}}
\label{app:numerical-illustrations}

\subsection{Data models}
\label{app:data-models}

For the simulations in \cref{app:ridge-illustration,app:ridgeless-illustration}, 
as well as that behind \cref{fig:ridgeless-intro}, we generate data according to
a nonlinear model     
\[
    y_i = x_i^\top \beta + (\|x_i\|_2^2/d - 1 ) + \eps_i, \quad i \in [n],   
\]
where each \smash{$x_i \sim \cN(0, \Sigma_{\textsc{ar1}, \rho=0.25})$}, $\eps_i 
\sim \cN(0, 0.4^2)$, and $\beta$ is drawn uniformly from the unit sphere in
$\RR^p$. Here, we use \smash{$\Sigma_{\textsc{ar1}, \rho}$} to denote a
covariance matrix with $\rho^{|i-j|}$. The ``linearized'' SNR in this setup is  
$\Var[x_i^\T \beta] / \sigma^2 = 6.25$.    

For the simulation behind \cref{fig:ridgeless-intro} only (i.e., not in
\cref{app:ridge-illustration,app:ridgeless-illustration}), we sample $P=300$
features total according to the above model, sort them in order of deceasing 
magnitude of $|\beta_j|$ (the linear part of the signal), and use the first $p$
for least squares (if $p \leq n$), or ridgeless regression (if $p > n$), as $p$
varies from 1 to 300. All quantities in this figure are empirical estimates
computed over 500 repetitions (500 times drawing the simulated data sets), and
in each repetition, the empirical prediction errors are computed based on a test
set of 1000 samples.     

For the simulations in \cref{app:lasso-illustration,app:lassoless-illustration},
we generate data according to a linear model
\[
y_i = x_i^\top \beta + \eps_i, \quad i \in [n],
\]
where each $x_i \sim \cN(0, I/n)$, $\eps_i \sim \cN(0, 1)$, and we set
\smash{$\beta_j = \sqrt{n/(\delta p)}$} with probability $\delta$ on, and
$\beta_j = 0$ with probability $1-\delta$, independently for $j \in [p]$. This
setup has an SNR of 1.   
    
In all figures that follow in this appendix section, 
\cref{fig:ridge-illustration,fig:ridgeless-illustration,%
fig:lasso-illustration,fig:lassoless-illustration},
the curves indicate theoretical quantities (asymptotic equivalents from the
theorems), while the dots denote empirical estimates from averaging over 100 
repetitions (100 times drawing the simulated data sets). In each repetition,
empirical prediction errors are computed based on a test set of 1000 samples. 
   
\subsection{Figure formatting}

For all figures in this section, we use the following formatting scheme. 

\begin{itemize}
\item Curves in the underparameterized regime are colored
  \textcolor{pythonblue}{blue}. 
\item Curves in the overparameterized regime are colored
  \textcolor{pythonorange}{orange}. 
\item Fixed-X quantities are colored \textcolor{pythongreen}{green}. 
\item Emergent random-X quantities are denoted by solid lines
  (\rule[0.5ex]{1em}{0.2pt}).
\item Intrinsic random-X quantities are denoted by dashed lines
  (\rule[0.5ex]{0.35em}{0.2pt}\hspace{0.25em}\rule[0.5ex]{0.35em}{0.2pt}\hspace{0.25em}\rule[0.5ex]{0.35em}{0.2pt}).  
\end{itemize}

\subsection{Ridge regression}
\label{app:ridge-illustration}

\cref{fig:ridge-illustration} provides empirical support for the behaviors
described in \cref{prop:ridge-monotonicity} and
\cref{thm:ridge-asymptotics}. The top row corresponds to the underparameterized 
regime, while the bottom row corresponds to the overparameterized regime. 
Throughout, we see that the empirical estimates (dots) closely track the
asymptotic equivalents (curves). 

Moreover, we observe the following behaviors which align with the theory. The
intrinsic random-X degrees of freedom decreases monotonically with $\lambda$ in
both the underparameterized and overparameterized regimes. Interestingly, the
emergent random-X degrees of freedom can have nonmonotonic behavior in
$\lambda$. Lastly, emergent random-X degrees of freedom is consistently higher  
than intrinsic random-X degrees of freedom, confirming that the presence of bias
inflates degrees of freedom.

\begin{figure}[p]
\includegraphics[width=\textwidth]{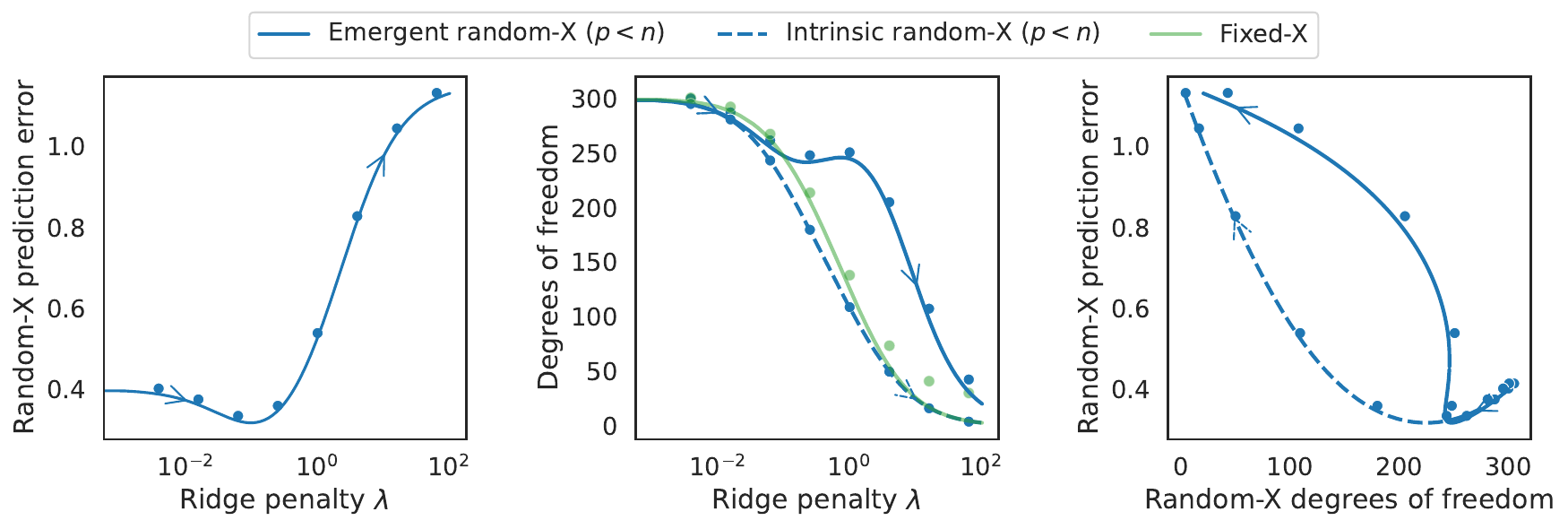}
\includegraphics[width=\textwidth]{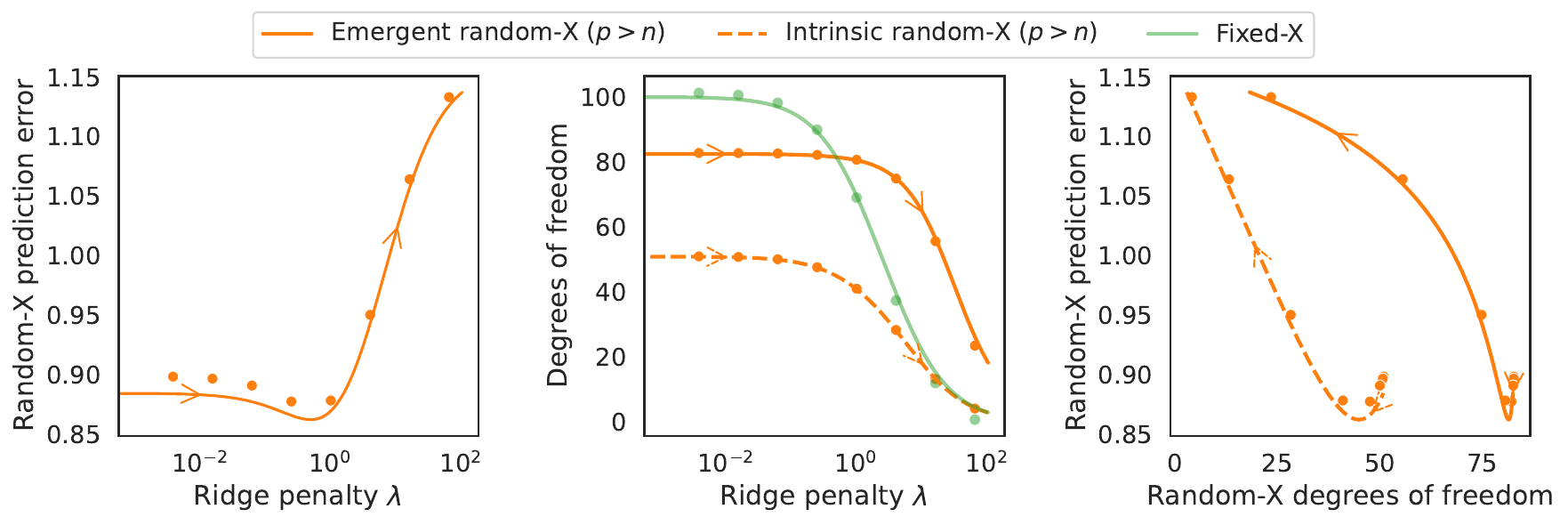}
\caption{Prediction error and degrees of freedom of ridge predictors, over
  varying $\lambda$, in a problem setting with $p=300$ features. The first
  row corresponds to the underparameterized regime, $n=500$, and the second to
  the overparameterized regime, $n=200$. The precise setup is as described in
  \cref{app:data-models}.}   
\label{fig:ridge-illustration}

\bigskip\bigskip
\includegraphics[width=\textwidth]{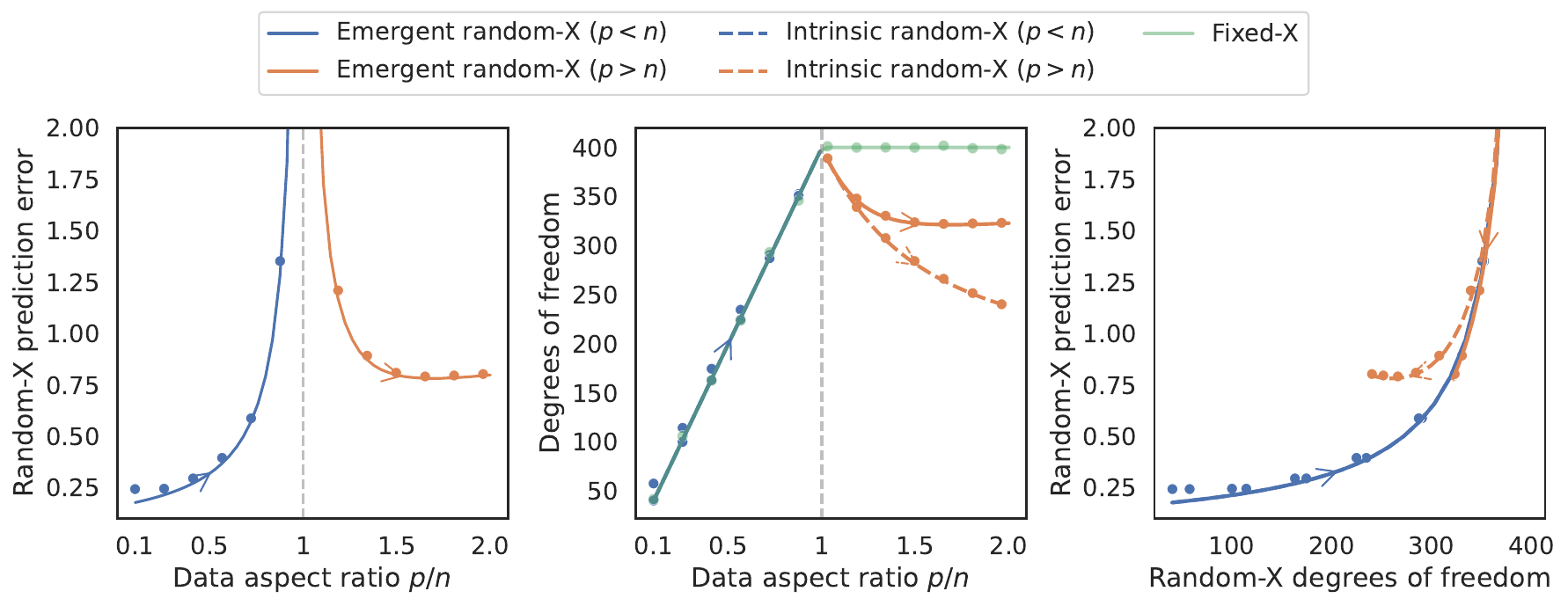}
\caption{Prediction error and degrees of freedom of ridgeless predictors with
  varying aspect ratio $\gamma_n = p/n$. The number of samples is $n=400$. The 
  precise setup is as described in \cref{app:data-models}.}   
\label{fig:ridgeless-illustration}
\end{figure}

\subsection{Ridgeless regression}
\label{app:ridgeless-illustration}

\cref{fig:ridgeless-illustration} provides empirical support for the behaviors 
described in \cref{thm:ridgeless-asymptotics} and
\cref{prop:ridgeless-monotonicity}. We see that the empirical estimates (dots)
closely track the asymptotic equivalents (curves).

Furthermore, we observe the following behaviors which align with the
theory. Both the intrinsic and emergent random-X degrees of freedom are
maximized at $\gamma_n = 1$. The intrinsic random-X degrees of freedom decreases
on both sides as $\gamma_n$ moves away from 1. Moreover, emergent random-X
degrees of freedom is always higher than intrinsic random-X degrees of freedom. 

\subsection{Lasso illustration}
\label{app:lasso-illustration}

\cref{fig:lasso-illustration} provides empirical support for the behaviors 
described in \cref{thm:lasso-asymptotics} and
\cref{prop:lasso-monotonicity-nonnegativity}. The top row corresponds to the
underparameterized regime, while the bottom row corresponds to the
overparameterized regime. Throughout, we see that the empirical estimates (dots)
closely track with the asymptotic equivalents (curves). 

We also see the following behaviors which align with the theory. The
intrinsic random-X degrees of freedom decreases monotonically with $\lambda$ in
either the underparameterized and overparameterized setting. Also, the emergent 
random-X degrees of freedom is always higher than intrinsic random-X degrees of
freedom, confirming that the presence of bias inflates degrees of freedom.   

\begin{figure}[p]
\includegraphics[width=\textwidth]{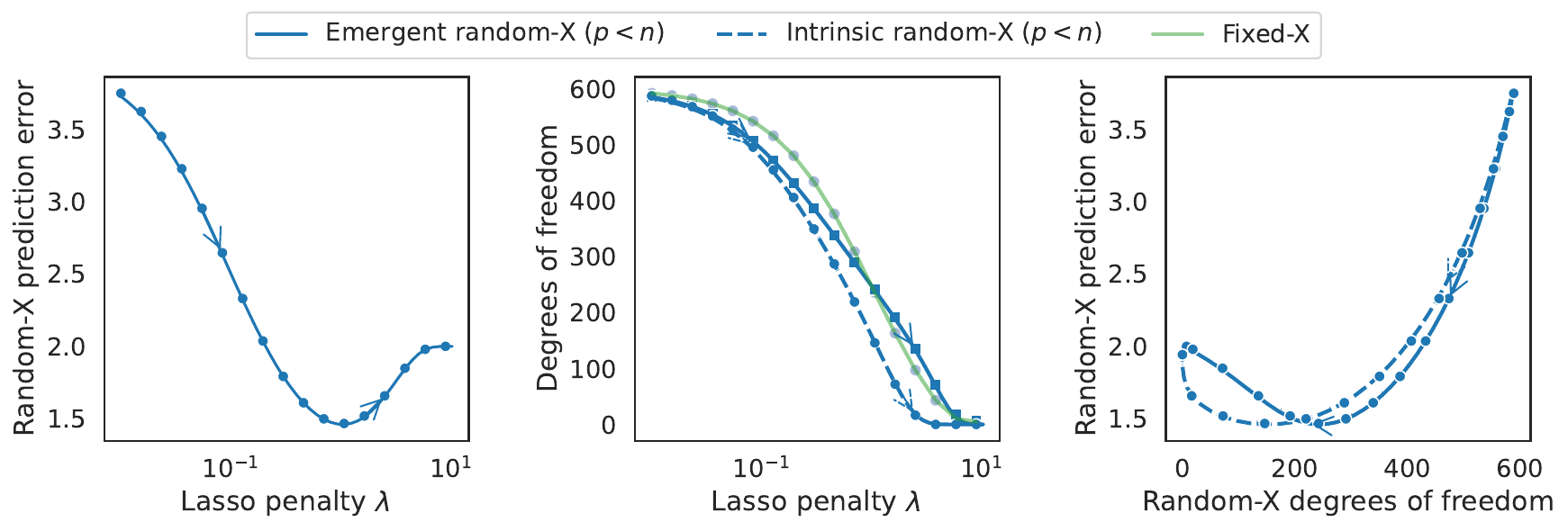}
\includegraphics[width=\textwidth]{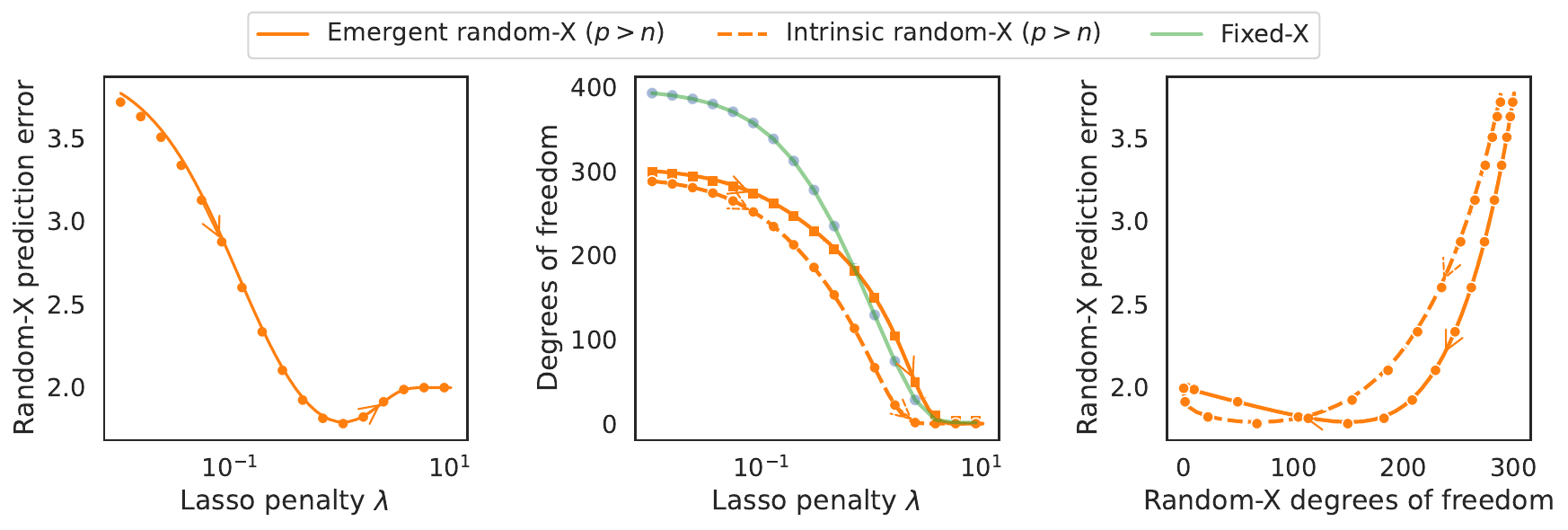}
\caption{Prediction error and degrees of freedom of lasso predictors, over
  varying $\lambda$, in a problem setting with $p=600$ features. The first
  row corresponds to the underparameterized regime, $n=800$, and the second to
  the overparameterized regime, $n=400$. The precise setup is as described in
  \cref{app:data-models} with $\delta=1/6$.}  
\label{fig:lasso-illustration}

\bigskip\bigskip
\includegraphics[width=\textwidth]{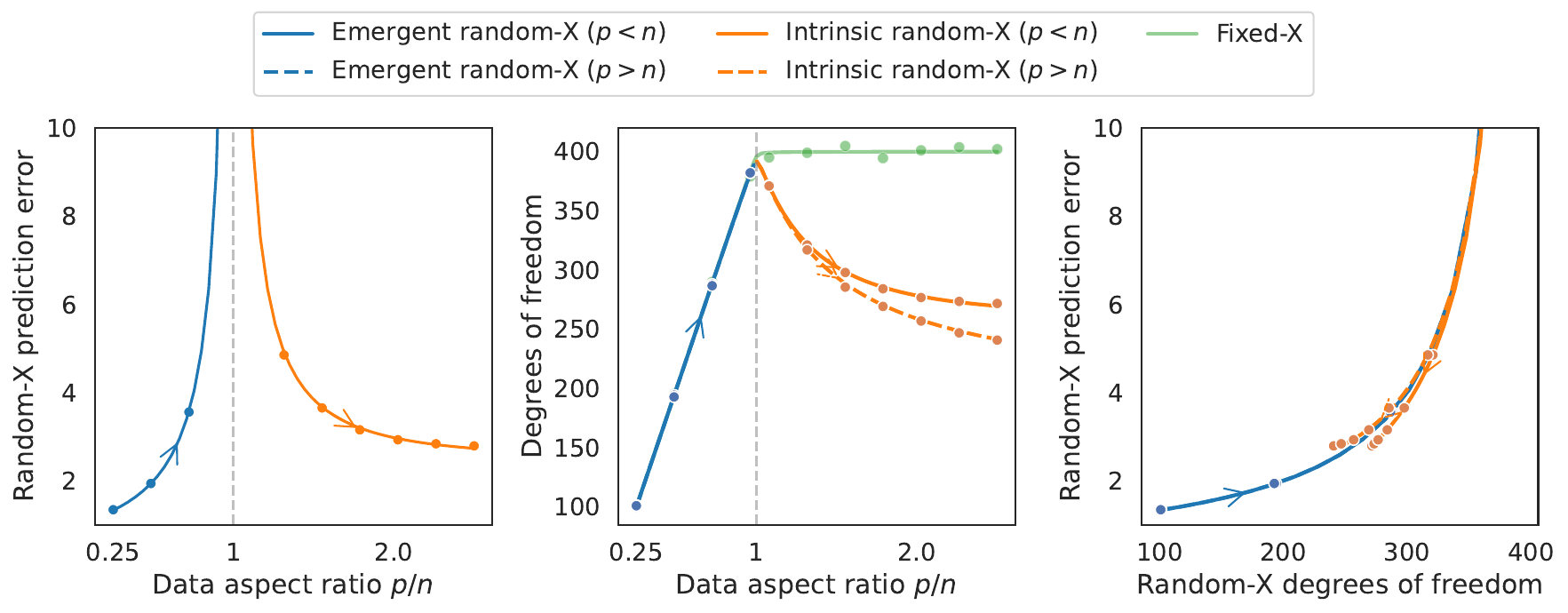}
\caption{Prediction error and degrees of freedom of lassoless predictors with
  varying aspect ratio $\gamma_n = p/n$. The number of samples is $n=400$. The
  precise is as described in \cref{app:data-models} with $\delta=1/10$.} 
\label{fig:lassoless-illustration}
\end{figure}

\subsection{Lassoless illustration}
\label{app:lassoless-illustration}

\cref{fig:lassoless-illustration} provides empirical support for the behaviors 
described in \cref{thm:lassoless-asymptotics} and
\cref{prop:lassoless-monotonicity}. We see that the empirical estimates (dots)
closely track the asymptotic equivalents (curves).

Furthermore, we observe the following behaviors which align with the
theory. Both the intrinsic and emergent random-X degrees of freedom are
maximized at $\gamma_n = 1$. The intrinsic random-X degrees of freedom decreases
on both sides as $\gamma_n$ moves away from 1. Moreover, emergent random-X
degrees of freedom is always higher than intrinsic random-X degrees of freedom. 

\section{Additional experiments for \cref{sec:experiments}}
\label{app:additional-experiments}

\begin{figure}[p]
\includegraphics[width=\textwidth]{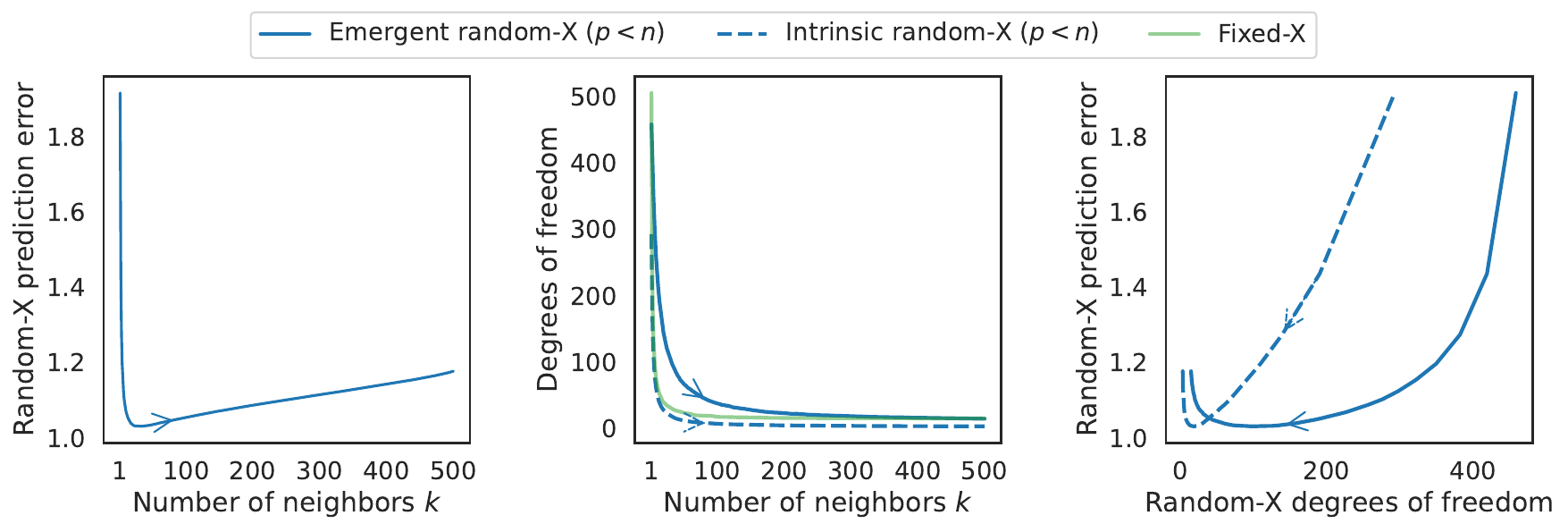}
\caption{Prediction error and degrees of freedom of \knn predictors, in a
  problem with $n=500$, $p=300$.}
\label{fig:knn-underparam}

\bigskip
\includegraphics[width=\textwidth]{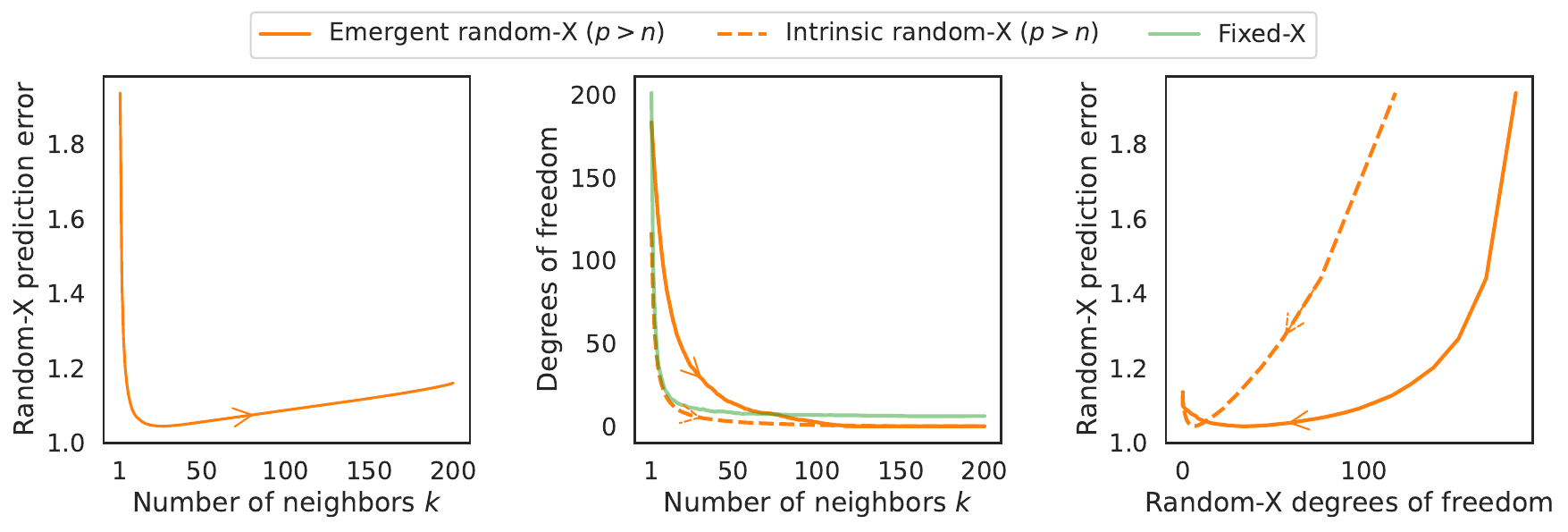}
\caption{Prediction error and degrees of freedom of \knn predictors, in a
  problem with $n=200$, $p=300$.}
\label{fig:knn-overparam}

\bigskip
\includegraphics[width=\linewidth]{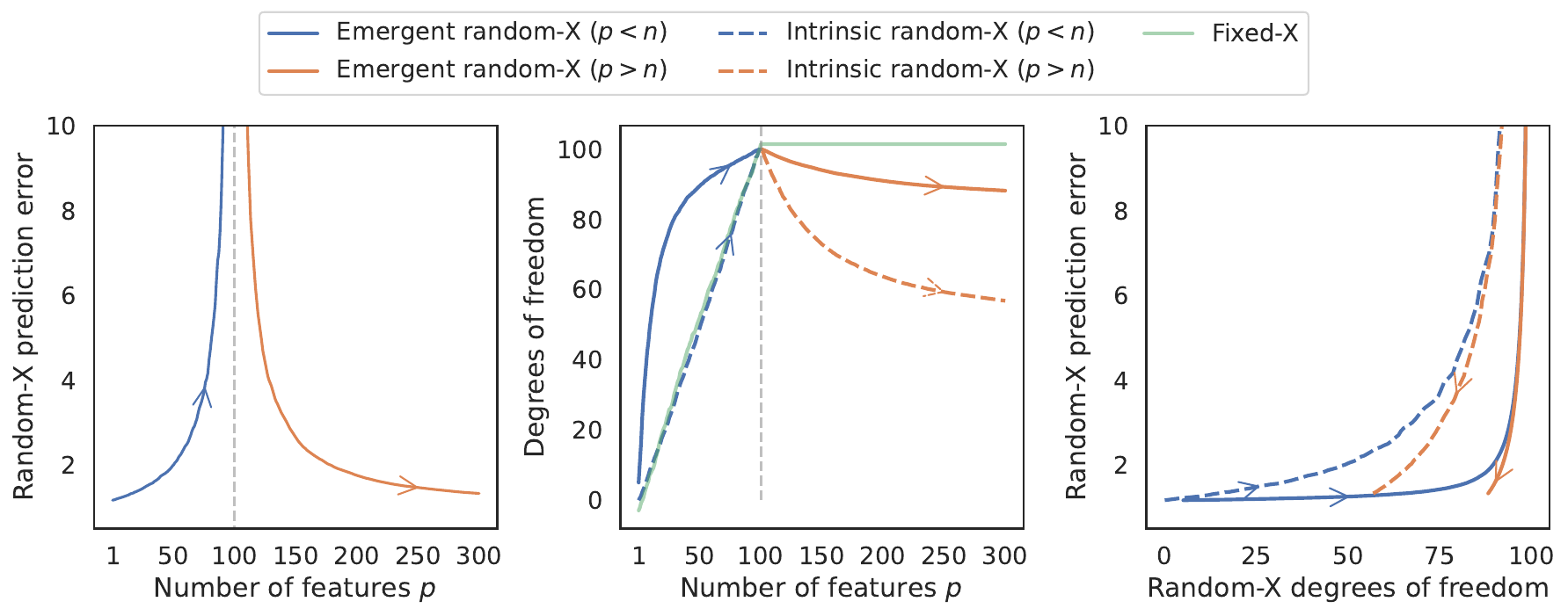}
\caption{Prediction error and degrees of freedom of ridgeless regression on
  random features, in a problem where $n=100$ and $p$ ranges from 1 to 300.} 
\label{fig:random-features}
\end{figure}

\subsection[k-nearest neighbors regression]{$k$-nearest neighbors regression} 
\label{app:knn-additional}

Here we study $k$-nearest neighbors (\knn) regression. Note that this is a
linear smoother, hence its random-X degrees of freedom is characterized by 
\cref{prop:smoother}, but it is not defined by a penalized least squares
problem, therefore it eludes the analysis in \cref{prop:B+} which characterizes  
emergent minus intrinsic degrees of freedom. 

We simulate data according to the nonlinear model described in
\cref{app:data-models}. \cref{fig:knn-underparam} displays the results for an 
underparameterized problem with $n = 500$, $p=300$, and \cref{fig:knn-overparam}
displays the results for an overparameterized problem with $n = 200$,
$p=300$. In both cases, we can see (middle panel) that the intrinsic random-X
degrees of freedom is slightly smaller than the fixed-X degrees of freedom
throughout, for all $k$; whereas the emergent random-X degrees of freedom is
somewhat larger than fixed-X degrees of freedom for small $k$, then it drops
down for larger $k$. A commonality we see here, as with all other experiments,
is that the degrees of freedom ``due to bias'' is positive. However, an
interesting difference is as follows: emergent degrees of freedom is
\emph{larger} than fixed-X degrees of freedom on the less-regularized side of
the model class (smaller $k$); with other predictors, we observe emergent
degrees of freedom being smaller than fixed-X degrees of freedom on this side of
the path (cf.\ ridge and lasso predictors for small $\lambda$ in
\cref{app:ridge-illustration,app:lasso-illustration}).

\subsection{Random features}

We examine ridgeless regression on random features. We simulate data
according to the nonlinear model in \cref{app:data-models}, with $n=100$ samples
and $P=300$ features total, then we use features \smash{$\tilde{x}_i =
  \mathrm{tanh}(Fx_i)$} for least squares (if $p \leq n$), or ridgeless
regression (if $p > n$), where $F \in \RR^{p \times P}$ has entries drawn from
\smash{$\cN(0, 1/\sqrt{P})$}, and $p$ varies from 1 to 300.  

\cref{fig:random-features} displays the results. These results are overall
similar to \cref{fig:ridgeless-intro}, except the emergent random-X degrees of
freedom is inflated before the interpolation threshold at $p = n$. 

\end{document}